\documentclass[12pt]{article}

\usepackage{etoolbox}

\newbool{report}
\setbool{report}{true}
\usepackage{graphics}
\usepackage{graphicx}
\usepackage{amssymb}
\usepackage{verbatim}     
\usepackage{amsxtra}
\usepackage{epsfig}
\usepackage{subfigure}
\usepackage{amsmath}
\usepackage{latexsym}
\usepackage{ifthen}
\usepackage{psfrag}
\usepackage{subfigure}
\usepackage{color}
\usepackage{float}
\usepackage{tikz}
\usepackage{mathtools}
\usepackage{etoolbox}
\usepackage{algorithm}
\usepackage{algpseudocode}
\let\labelindent\relax
\usepackage{enumitem}
\usepackage{import}
\usepackage{pdfpages}
\usepackage{amsthm}

\usepackage{moreverb,url}
\usepackage{natbib}
\usepackage[colorlinks,bookmarksopen,bookmarksnumbered,citecolor=red,urlcolor=red]{hyperref}

\usepackage{./macros/rgsEnvironments}
\usepackage{./macros/rgsMacros}  

\ifbool{report}{
\usepackage{endnotes}}{}

\newcommand\BibTeX{{\rmfamily B\kern-.05em \textsc{i\kern-.025em b}\kern-.08em
		T\kern-.1667em\lower.7ex\hbox{E}\kern-.125emX}}

\DeclareMathOperator{\interior}{int}
\DeclareMathOperator*{\argmin}{arg\,min}
\newcommand{\nnumbers}[1][]{\mathbb{N}}
\newcommand{\myreals}[1][]{\mathbb{R}^{#1}}
\newcommand{\mypreals}[1][]{\mathbb{R}_{\geq 0}^{#1}}
\newcommand{\myinputt}{\upsilon}
\newcommand{\mystatet}{\phi}
\newcommand{\hs}{\mathcal{H}}
\newcommand{\myball}{\mathbb{B}}

\newcommand{\pn}[1]{#1}

\newcounter{theorem}
\newtheorem{problem}[theorem]{Problem}

\newlist{steps}{enumerate}{1}
\setlist[steps, 1]{leftmargin=45pt, label = \textbf{Step \arabic*}:}

\newcommand{\EndRemark}{$\hfill \blacktriangle$}
\ifbool{report}{
	\usepackage{hslTR}
}{
	\def\volumeyear{2024}
	\setcounter{secnumdepth}{3}
}
	
\begin{document}
\ifbool{report}{
\ititle{Motion Planning for Hybrid Dynamical Systems: Framework, Algorithm Template, and a Sampling-based Approach}
\iauthor{
	Nan Wang\\
	{\normalsize nanwang@ucsc.edu} \\
	Ricardo G. Sanfelice\\
	{\normalsize ricardo@ucsc.edu}}
\idate{\today{}} 
\iyear{2023}
\irefnr{05}
\makeititle

\title{Motion Planning for Hybrid Dynamical Systems: Framework, Algorithm Template, and a Sampling-based Approach}

\author{Nan Wang and Ricardo G. Sanfelice}
\maketitle
}{
	\runninghead{Wang and Sanfelice}
	\title{Motion Planning for Hybrid Dynamical Systems: Framework, Algorithm Template, and a Sampling-based Approach}
	\author{Nan Wang and Ricardo G. Sanfelice\affilnum{1}}
	
	\affiliation{\affilnum{1}Hybrid Systems, Laboratory, Department of Electrical and Computer Engineering, University of California, Santa Cruz, USA}
	
	\corrauth{\affilnum{1}Nan Wang, Department of Electrical and Computer Engineering, University of California, Santa Cruz,
		USA.}
	\email{nanwang@ucsc.edu}
	
	\begin{abstract}
This paper focuses on the motion planning problem for the systems exhibiting both continuous and discrete behaviors, which we refer to as hybrid dynamical systems. Firstly, the motion planning problem for hybrid systems is formulated using the hybrid equation framework, which is general to capture most hybrid systems. Secondly, a propagation algorithm template is proposed that describes a general framework to solve the motion planning problem for hybrid systems. Thirdly, a rapidly-exploring random trees (RRT) implementation of the proposed algorithm template is designed to solve the motion planning problem for hybrid systems. At each iteration, the proposed algorithm, called HyRRT, randomly picks a state sample and extends the search tree by flow or jump, which is also chosen randomly when both regimes are possible. Through a definition of concatenation of functions defined on hybrid time domains, we show that HyRRT is probabilistically complete, namely, the probability of failing to find a motion plan approaches zero as the number of iterations of the algorithm increases. This property is guaranteed under mild conditions on the data defining the motion plan, which include a relaxation of the usual positive clearance assumption imposed in the literature of classical systems. The motion plan is computed through the solution of two optimization problems, one associated with the flow and the other with the jumps of the system. The proposed algorithm is applied to an actuated bouncing ball system and a walking robot system so as to highlight its generality and computational features.
\end{abstract}
	
	\maketitle
}

\maketitle

\section{Introduction}
Motion planning consists of finding a state trajectory and associated inputs, connecting the initial and final state while satisfying the system dynamics and a given safety criterion. The motion planning technology has been widely used to help robotics applications, such as autonomous driving systems \cite{teng2023motion}, satellites systems \cite{rybus2020point}, biped robots \cite{huang2001planning}, and quadrotors \cite{liu2017search}, to complete complicated tasks. Particularly, the system dynamics considered in the motion planning not only depends on the mechanical design of the robot, but also depends on the internal logic/timer to resolve the task, and the interaction between the robot systems and the environment, et.al. The former usually leads to a continuous dynamics while the later usually leads to a discrete dynamics. For example, the \emph{continuous} evolution of the states of the wheeled vehicle can be modeled by either its kinematic model \cite{rajamani2011vehicle} or its dynamic model \cite{gillespie2021fundamentals}. On the other hand, if a motion planner is designed for a robot to move to the highest floor in a building, then a logic variable representing the floor number should be employed and updated in a \emph{discrete} manner \cite{branicky2003rrts}. In addition, a motion planner to plan the motion of consecutive steps of a biped robot should consider \emph{discrete} changes over the foot's speed when an impact between the foot and the ground occurs.

However, in certain motion planning tasks, the states can both evolve continuously and, from time to time, execute some discrete changes. For instance, in the context of a collision-resilient multicopter system in \cite{zha2021exploiting}, a RRT-type motion planner is devised. The multicopter's position and velocity states evolve continuously in open space, yet exhibit discrete changes upon collision with a wall. In such instances, neither a purely continuous-time nor discrete-time model suffices to accurately capture system behavior. Instead, a hybrid system model proves essential, encompassing not only purely continuous or discrete-time systems but also those manifesting both continuous and discrete behaviors.  Consequently, employing a general hybrid system model renders motion planning problems more comprehensive than those confined to purely continuous-time or discrete-time systems.

Research has been conducted to study motion planning for some specific classes of hybrid systems. Notably, in \cite{wu2020r3t} and \cite{branicky2003sampling}, two RRT-type motion planning algorithms are introduced to address motion planning problems within the realms of continuous-time hybrid systems. These systems feature continuous states evolving continuously alongside discrete states (modes) that transition among a finite set of feasible modes. However, they do not encompass hybrid systems where states evolve continuously and intermittently execute jumps. 
This paper focuses on motion planning problems for hybrid systems modeled as hybrid equations \cite{goebel2009hybrid}. In this modeling framework, differential and difference equations with constraints are used to describe the continuous and discrete behavior of the hybrid system, respectively. This general hybrid system framework can capture most hybrid systems emerging in robotic applications, not only the class of hybrid systems considered in \cite{wu2020r3t} and \cite{branicky2003sampling}, but also systems with memory states, timers, impulses, and constraints.  From the authors' best knowledge, it is the first time that the motion planning problem is formulated for this general hybrid system framework.

Motion planning problems for purely continuous-time systems (known as \emph{kinodynamic planning}) and purely discrete-time systems have been extensively explored in the literature; see, e.g., \cite{lavalle2006planning}. Most existing motion planning algorithms, such as RRT algorithm \cite{lavalle2001randomized}, incrementally construct a search tree in the state space and seek for a path in the search tree connecting the initial and final states. Operations like concatenation, well-defined for purely continuous-time and purely discrete-time systems, are crucial in these algorithms and are extensively documented in the literature; see, for instance, \cite[Chapter 14.3]{lavalle2006planning}. However, defining such operations for trajectories of hybrid systems poses a significant challenge because of their considerably more complicated domain structure. In fact, trajectories of hybrid systems may exhibit various behaviors: 1) evolves purely continuously, 2) exhibits jumps all the time, 3) evolves continuous and exhibits one or multiple jumps at times, or 4) exhibits Zeno behavior. To the authors' best knowledge, there is no existing work formulating such operations for hybrid systems, let alone their properties. Within this general framework, mathematical definitions of operations for hybrid systems are formulated, along with theoretical analysis of those operations for motion planning algorithms for hybrid systems. In addition, a high-level motion planning algorithm template that utilizes those operations is proposed to help design and validate the motion planning algorithms to solve the motion planning problems for hybrid systems.
 
 In recent years, various planning algorithms have been developed to solve motion planning problems, from  graph search algorithms \cite{likhachev2005anytime} to artificial potential field methods \cite{khatib1986real}. A main drawback of graph search algorithms is that the number of vertices grows exponentially as the dimension of states grows, which makes computing motion plans inefficient for high-dimensional systems. The artificial potential field method suffers from getting stuck at local minimum. Arguably, the most successful algorithm to solve  motion planning problems for purely continuous-time systems and purely discrete-time systems is the sampling-based RRT algorithm \cite{lavalle1998rapidly}. This algorithm incrementally constructs a tree of state trajectories toward random samples in the state space. Similar to graph search algorithms, RRT suffers from the curse of dimensionality, but, in practice, achieves rapid exploration in solving high-dimensional motion planning problems \cite{cheng2005sampling}. Compared with the artificial potential field method, RRT is probabilistically complete \cite{lavalle2001randomized}, which means that the probability of failing to find a motion plan converges to zero, as the number of samples approaches infinity.
While RRT algorithms have been used to solve motion planning problems for purely continuous-time systems \cite{lavalle2001randomized} and purely discrete-time systems \cite{branicky2003rrts}, fewer efforts have been devoted to applying RRT-type algorithms to solve motion planning problems for systems with combined continuous and discrete behavior except that in  \cite{branicky2003sampling}, a hybrid RRT algorithm is designed for a specific class of hybrid systems. Following the algorithm template proposed in this paper, for this broad class of hybrid systems, an RRT-type motion planning algorithm is designed. The proposed algorithm, called HyRRT, incrementally constructs search trees, rooted in the initial state set and toward the random samples. At first, HyRRT draws samples from the state space. Then, it selects the vertex such that the state associated with this vertex has minimal distance to the sample. Next, HyRRT propagates the state trajectory from the state associated with the selected vertex. It is established that, under mild assumptions, HyRRT is probabilistically complete. To the authors' best knowledge, HyRRT is the first RRT-type algorithm for systems with hybrid dynamics that is proved to be probabilistically complete. The proposed algorithm is applied to an actuated bouncing ball system and a walking robot system so as to assess its capabilities. \ifbool{report}{A conference version of this report was published in \cite{wang2022rapidly}.}{}

The remainder of the paper is structured as follows. Section~\ref{section:preliminary} presents notation and preliminaries. Section~\ref{section:problemstatement} presents the problem statement and introduction of  applications. Section \ref{section:bruteforcealgorithm} defines the operations heavily used in the algorithms and presents a general framework template for the motion planning algorithms for hybrid systems. Section \ref{section:hybridRRT} presents the HyRRT algorithm. Section \ref{section:pc} presents the analysis of the probabilistic completeness of HyRRT algorithm. Section \ref{section:illustration} presents the illustration of HyRRT in examples. 
\section{Notation and Preliminaries}\label{section:preliminary}
 
\subsection{Notation}
The real numbers are denoted as $\mathbb{R}$ and its nonnegative subset is denoted as $\mathbb{R}_{\geq 0}$. The set of nonnegative integers is denoted as $\mathbb{N}$. The notation $\interior I$ denotes the interior of the interval $I$. The notation $\overline{S}$ denotes the closure of the set $S$. The notation $\partial S$ denotes the boundary of the set $S$. Given sets $P\subset\mathbb{R}^n$ and $Q\subset\mathbb{R}^n$, the Minkowski sum of $P$ and $Q$, denoted as $P + Q$, is the set $\{p + q: p\in P, q\in Q\}$. The notation $|\cdot|$ denotes the Euclidean norm. 
The notation $\rge f$ denotes the range of the function $f$.
Given a point $x\in \mathbb{R}^{n}$ and a subset $S\subset \mathbb{R}^{n}$, the distance between $x$ and $S$ is denoted $|x|_{S} := \inf_{s\in S} |x - s|$. The notation $\mathbb{B}$ denotes the closed unit ball of appropriate dimension in the Euclidean norm. 
The probability of the probabilistic event\endnote{In this paper, we call the event in probability theory  \emph{probabilistic event} to distinguish it from the term \emph{event} in the hybrid system theory, which is a synonym for jump. In the probability theory, an event is a set of outcomes of an experiment (a subset of the sample space) to which a probability is assigned \cite{leon1994probability}.} $M$ is denoted as $\mbox{\rm Prob}(M)$. Given a set $S$, the notation $\mu(S)$ denotes its Lebesgue measure.
The Lebesgue measure of the $n$-th dimensional unit ball, denoted $\zeta_{n}$, is such that
\begin{equation}
	\label{equation:zetan}
		\zeta_{n} := \left\{\begin{aligned}
		&\frac{\pi^{k}}{k!} &\text{ if } n = 2k, k\in \mathbb{N}\\
		&\frac{2(k!)(4\pi)^{k}}{(2k + 1)!} &\text{ if } n = 2k + 1, k\in \mathbb{N};\\
		\end{aligned}\right.
\end{equation} see \cite{gipple2014volume}.

\subsection{Preliminaries}
Following \cite{goebel2009hybrid}, a hybrid system $\mathcal{H}$ with inputs is modeled as  
\begin{equation}
\mathcal{H}: \left\{              
\begin{aligned}               
\dot{x} & = f(x, u)     &(x, u)\in C\\                
x^{+} & =  g(x, u)      &(x, u)\in D\\                
\end{aligned}   \right. 
\label{model:generalhybridsystem}
\end{equation}
where $x\in \mathbb{R}^n$ is the state, $u\in \mathbb{R}^m$ is the input, $C\subset \mathbb{R}^{n}\times\mathbb{R}^{m}$ represents the flow set, $f: \mathbb{R}^{n}\times\mathbb{R}^{m} \to \mathbb{R}^{n}$ represents the flow map, $D\subset \mathbb{R}^{n}\times\mathbb{R}^{m}$ represents the jump set, and $g:\mathbb{R}^{n}\times\mathbb{R}^{m} \to \mathbb{R}^{n}$ represents the jump map. The continuous evolution of $x$ is captured by the flow map $f$. The discrete evolution of $x$ is captured by the jump map $g$. The flow set $C$ collects the points where the state can evolve continuously. The jump set $D$ collects the points where jumps can occur.

Given a flow set $C$, the set $U_{C} := \{u\in \mathbb{R}^{m}: \exists x\in \mathbb{R}^{n}\text{ such that } (x, u)\in C\}$ includes all possible input values that can be applied during flows. Similarly, given a jump set $D$, the set $U_{D} := \{u\in \mathbb{R}^ {m}: \exists x\in \mathbb{R}^{n}\text{ such that } (x, u)\in D\}$ includes all possible input values that can be applied at jumps. These sets satisfy $C\subset \mathbb{R}^{n}\times U_{C}$ and $D\subset \mathbb{R}^{n}\times U_{D}$. Given a set $K\subset \mathbb{R}^{n}\times U_{\star}$, where $\star$ is either $C$ or $D$, we define
	\begin{equation}
	\label{equation:pi}
			\Pi_{\star}(K) := \{x: \exists u\in U_{\star} \text{ s.t. } (x, u)\in K\}
	\end{equation}
	as the projection of $K$ onto $\mathbb{R}^{n}$, and define 
	\begin{equation}
	\label{equation:Cprime}
	C' := \Pi_{C}(C)
	\end{equation} and 
	\begin{equation}
	\label{equation:Dprime}
	D' := \Pi_{D}(D).
	\end{equation}

In addition to ordinary time $t\in \mathbb{R}_{\geq 0}$, we employ $j\in \mathbb{N}$ to denote the number of jumps of the evolution of $x$ and $u$ for $\mathcal{H}$ in (\ref{model:generalhybridsystem}), leading to hybrid time $(t, j)$ for the parameterization of its solutions and inputs. The domain of a solution to $\mathcal{H}$ is given by a hybrid time domain as is defined as follows. 
\begin{definition}[(Hybrid time domain)]\label{definition:hybridtimedomain}
	A set $E\subset \mypreals\times\nnumbers$ is a hybrid time domain if, for each $(T, J)\in E$, the set
	$$
		E\cap ([0, T]\times \{0, 1, ..., J\})
	$$ can be written in the form
	$$
		\bigcup_{j = 0}^{J} ([t_{j}, t_{j + 1}]\times \{j\})
	$$
	for some finite sequence of times $\{t_{j}\}_{j = 0}^{J + 1}$ satisfying $0=t_{0}\leq t_{1}\leq t_{2}\leq ... \leq t_{J+1} = T$.
\end{definition}
Given a compact hybrid time domain $E$, we denote the maximum of $t$ and $j$ coordinates of points in $E$ as 
	\begin{equation}\label{equation:maxt}
		\text{max}_{t} E := \max\{t\in \mypreals: \exists j\in \nnumbers \text{ such that } (t, j)\in E\}
	\end{equation} and
\begin{equation}\label{equation:maxj}
	\text{max}_{j} E := \max\{j\in \nnumbers: \exists t\in \mypreals \text{ such that } (t, j)\in E\}.
\end{equation}

The input to the hybrid systems is defined as a function on a hybrid time domain as follows.
\begin{definition}[(Hybrid input)]\label{definition:hybridinput}
	A function $\myinputt: \dom \myinputt \to \myreals[n]$ is a hybrid input if $\dom \myinputt$ is a hybrid time domain and if, for each $j\in \nnumbers$, $t\mapsto \myinputt(t,j)$ is Lebesgue measurable and locally essentially bounded on the interval $I^{j}_{\myinputt}:=\{t:(t, j)\in \dom \myinputt\}$. 
\end{definition}
A hybrid arc describes the state trajectory of the system as is defined as follows.
\begin{definition}[(Hybrid arc)]\label{definition:hybridarc}
	A function $\mystatet: \dom \mystatet \to \myreals[n]$ is a hybrid arc if $\dom \mystatet$ is a hybrid time domain and if, for each $j\in \nnumbers$, $t\mapsto \mystatet(t,j)$ is locally absolutely continuous on the interval $I^{j}_{\mystatet}:=\{t:(t, j)\in \dom \mystatet\}$. 
\end{definition}
\begin{definition}[(Types of hybrid arcs)] \label{definition:hybridarctypes}A hybrid arc $\mystatet$ is called
		\begin{itemize}
			\item compact if $\dom \mystatet$ is compact;
			\item nontrivial if $\dom \mystatet$ contains at least two points;
			\item purely discrete if nontrivial and $\dom \mystatet\subset \{0\}\times\nnumbers$;
			\item purely continuous if nontrivial and $\dom \mystatet\subset \mypreals\times \{0\}$.
		\end{itemize}
\end{definition}
The definition of solution pair to a hybrid system is given as follows.
\begin{definition}[(Solution pair to a hybrid system)]
	\label{definition:solution}
	A hybrid input $\myinputt$ and a hybrid arc $\mystatet$ define a solution pair $(\mystatet, \myinputt)$ to the hybrid system $\hs$ if
	\begin{enumerate}[label=\arabic*)]
		\item $(\mystatet(0,0), \myinputt(0,0))\in \overline{C} \cup D$ and $\dom \mystatet = \dom \myinputt (= \dom (\mystatet, \myinputt))$.
		\item For each $j\in \mathbb{N}$ such that $I^{j}_{\mystatet}$ has nonempty interior $\interior(I^{j}_{\mystatet})$, $(\mystatet, \myinputt)$ satisfies
		$$
		(\mystatet(t, j),\myinputt(t, j))\in C
		$$ for all $t\in \interior I^{j}_{\mystatet}$, and 
		$$
		\frac{\text{d}}{\text{d} t} {\mystatet}(t,j) = f(\mystatet(t,j), \myinputt(t,j))
		$$ for almost all $t\in I^{j}_{\mystatet}$.
		\item For all $(t,j)\in \dom (\mystatet, \myinputt)$ such that $(t,j + 1)\in \dom (\mystatet, \myinputt)$, 
		\begin{equation}
		\begin{aligned}
		(\mystatet(t, j), \myinputt(t, j))&\in D\\
		\mystatet(t,j+ 1) &= g(\mystatet(t,j), \myinputt(t, j)).
		\end{aligned}
		\end{equation}
	\end{enumerate}
\end{definition}

In the main result of this paper, the following definition of closeness between hybrid arcs is used.
\begin{definition}[($(\tau, \epsilon)$-closeness of hybrid arcs)]\label{definition:closeness}
	 Given $\tau, \epsilon>0$, two hybrid arcs $\mystatet_{1}$ and $\mystatet_{2}$ are $(\tau, \epsilon)$-close if
	\begin{enumerate}
		\item for all $(t, j)\in \dom \mystatet_{1}$ with $t + j \leq \tau$, there exists $s$ such that $(s, j)\in \dom \mystatet_{2}$, $|t - s|< \epsilon$, and $|\mystatet_{1}(t, j) - \mystatet_{2}(s, j)| < \epsilon$;
		\item for all $(t, j)\in \dom \mystatet_{2}$ with $t + j \leq \tau$, there exists $s$ such that $(s, j)\in \dom \mystatet_{1}$, $|t - s|< \epsilon$, and $|\mystatet_{2}(t, j) - \mystatet_{1}(s, j)| < \epsilon$.
	\end{enumerate}
\end{definition}

\section{Problem Statement}
The motion planning problem for hybrid systems studied in this paper is  formulated as follows.
\begin{problem}
	\label{problem:motionplanning}
	Given a hybrid system $\mathcal{H}$ as in (\ref{model:generalhybridsystem}) with input $u\in \myreals[m]$ and state $x\in \myreals[n]$, the initial state set $X_{0}\subset\myreals[n]$, the final state set $X_{f}\subset\myreals[n]$, and the unsafe set\endnote{The unsafe set $X_{u}$ can be used to constrain both the state and the input.} $X_{u}\subset\myreals[n]\times \myreals[m]$, find a pair $(\mystatet, \myinputt): \dom (\mystatet, \myinputt)\to \mathbb{R}^{n}\times \mathbb{R}^{m}$, namely, \emph{a motion plan}, such that for some $(T, J)\in \dom (\mystatet, \myinputt)$, the following hold:
	\begin{enumerate}[label=\arabic*)]
		\item $\mystatet(0,0) \in X_{0}$, namely, the initial state of the solution belongs to the given initial state set $X_{0}$;
		\item $(\mystatet, \myinputt)$ is a solution pair to $\mathcal{H}$ as defined in Definition \ref{definition:solution};
		\item $(T,J)$ is such that $\mystatet(T,J)\in X_{f}$, namely, the solution belongs to the final state set at hybrid time $(T, J)$;
		\item $(\mystatet(t,j), \myinputt(t, j))\notin X_{u}$ for each $(t,j)\in \dom (\mystatet, \myinputt)$ such that $t + j \leq T+ J$, namely, the solution pair does not intersect with the unsafe set before its state trajectory reaches the final state set.
	\end{enumerate}
Therefore, given sets $X_{0}$, $X_{f}$\pn{,} and $X_{u}$, and a hybrid system~$\mathcal{H}$ as in (\ref{model:generalhybridsystem}) with data $(C, f, D, g)$, a motion planning problem $\mathcal{P}$ is formulated as
$
	\mathcal{P} = (X_{0}, X_{f}, X_{u}, (C, f, D, g)).
$
\end{problem}
\begin{figure}[htbp]
	\centering
	\def\svgwidth{0.9\columnwidth}
\begingroup%
  \makeatletter%
  \providecommand\color[2][]{%
    \errmessage{(Inkscape) Color is used for the text in Inkscape, but the package 'color.sty' is not loaded}%
    \renewcommand\color[2][]{}%
  }%
  \providecommand\transparent[1]{%
    \errmessage{(Inkscape) Transparency is used (non-zero) for the text in Inkscape, but the package 'transparent.sty' is not loaded}%
    \renewcommand\transparent[1]{}%
  }%
  \providecommand\rotatebox[2]{#2}%
  \newcommand*\fsize{\dimexpr\f@size pt\relax}%
  \newcommand*\lineheight[1]{\fontsize{\fsize}{#1\fsize}\selectfont}%
  \ifx\svgwidth\undefined%
    \setlength{\unitlength}{460.76175145bp}%
    \ifx\svgscale\undefined%
      \relax%
    \else%
      \setlength{\unitlength}{\unitlength * \real{\svgscale}}%
    \fi%
  \else%
    \setlength{\unitlength}{\svgwidth}%
  \fi%
  \global\let\svgwidth\undefined%
  \global\let\svgscale\undefined%
  \makeatother%
  \begin{picture}(1,0.6529201)%
    \lineheight{1}%
    \setlength\tabcolsep{0pt}%
    \put(0,0){\includegraphics[width=\unitlength,page=1]{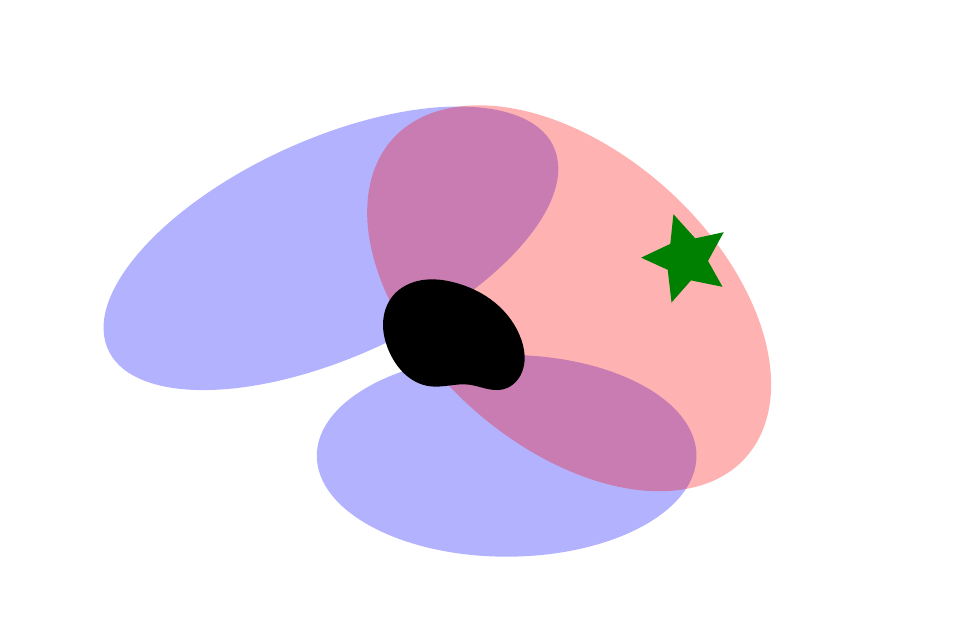}}%
    \put(0.05417655,0.20361801){\makebox(0,0)[lt]{\lineheight{1.25}\smash{\begin{tabular}[t]{l}initial state\end{tabular}}}}%
    \put(0.8141166,0.36989962){\makebox(0,0)[lt]{\lineheight{1.25}\smash{\begin{tabular}[t]{l}final state\end{tabular}}}}%
    \put(0.66638215,0.49017024){\makebox(0,0)[lt]{\lineheight{1.25}\smash{\begin{tabular}[t]{l}unsafe set\end{tabular}}}}%
    \put(0,0){\includegraphics[width=\unitlength,page=2]{hybridmotionplanning.pdf}}%
    \put(0.20653165,0.11709806){\makebox(0,0)[lt]{\lineheight{1.25}\smash{\begin{tabular}[t]{l}flow set\end{tabular}}}}%
    \put(0.82743469,0.26413127){\makebox(0,0)[lt]{\lineheight{1.25}\smash{\begin{tabular}[t]{l}jump set\end{tabular}}}}%
    \put(0,0){\includegraphics[width=\unitlength,page=3]{hybridmotionplanning.pdf}}%
  \end{picture}%
\endgroup%

	\caption{A motion plan to Problem \ref{problem:motionplanning}. The green square denote the initial state set. The green star denotes the final state set. The blue region denotes the flow set. The red region denotes the jump set. The solid blue lines denote flow and the dotted red lines denote jumps in the motion plan.}
\end{figure}
There are some interesting special  cases of Problem \ref{problem:motionplanning}. For example, when $D = \emptyset$ ($C = \emptyset$) and $C$ (respectively, $D$) is nonempty, $\mathcal{P}$ denotes the motion planning problem for purely continuous-time (respectively, discrete-time) systems with constraints. In this way, Problem \ref{problem:motionplanning} subsumes the motion planning problems for purely continuous-time and purely discrete-time system studied in \cite{lavalle2001randomized} and \cite{lavalle2006planning}.
\begin{example}[(Actuated bouncing ball system)]\label{example:bouncingball}
	Consider a ball bouncing on a fixed horizontal surface as is shown in Figure \ref{fig:bouncingball}. The surface is located at the origin and, through control actions, is capable of affecting the velocity of the ball after the impact.  The dynamics of the ball while in the air is given by
	\begin{equation}
	\label{model:bouncingballflow}
	\dot{x} = \left[ \begin{matrix}
	x_{2} \\
	-\gamma
	\end{matrix}\right] =: f(x, u)\qquad (x, u)\in C
	\end{equation}
	where $x :=(x_{1}, x_{2})\in \mathbb{R}^2$. The height of the ball is denoted by $x_{1}$. The velocity of the ball is denoted by $x_{2}$. The gravity constant is denoted by $\gamma$. 
	\begin{figure}[htbp] 
		\centering
		\parbox[h]{\columnwidth}{
			\centering
			\subfigure[The actuated bouncing ball system\label{fig:bouncingball} ]{%
	\def\svgwidth{0.7\columnwidth}
\begingroup%
  \makeatletter%
  \providecommand\color[2][]{%
    \errmessage{(Inkscape) Color is used for the text in Inkscape, but the package 'color.sty' is not loaded}%
    \renewcommand\color[2][]{}%
  }%
  \providecommand\transparent[1]{%
    \errmessage{(Inkscape) Transparency is used (non-zero) for the text in Inkscape, but the package 'transparent.sty' is not loaded}%
    \renewcommand\transparent[1]{}%
  }%
  \providecommand\rotatebox[2]{#2}%
  \newcommand*\fsize{\dimexpr\f@size pt\relax}%
  \newcommand*\lineheight[1]{\fontsize{\fsize}{#1\fsize}\selectfont}%
  \ifx\svgwidth\undefined%
    \setlength{\unitlength}{155.66442535bp}%
    \ifx\svgscale\undefined%
      \relax%
    \else%
      \setlength{\unitlength}{\unitlength * \real{\svgscale}}%
    \fi%
  \else%
    \setlength{\unitlength}{\svgwidth}%
  \fi%
  \global\let\svgwidth\undefined%
  \global\let\svgscale\undefined%
  \makeatother%
  \begin{picture}(1,0.86048663)%
    \lineheight{1}%
    \setlength\tabcolsep{0pt}%
    \put(0,0){\includegraphics[width=\unitlength,page=1]{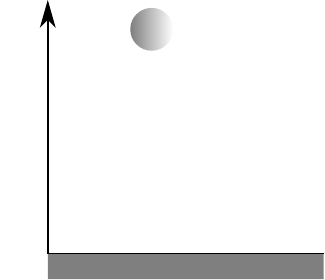}}%
    \put(0.03673765,0.65008529){\rotatebox{90}{\makebox(0,0)[rt]{\lineheight{1.25}\smash{\begin{tabular}[t]{r}height ($x_{1}$)\end{tabular}}}}}%
    \put(0,0){\includegraphics[width=\unitlength,page=2]{bouncingballfigure.pdf}}%
    \put(0.55137666,0.30449301){\makebox(0,0)[lt]{\lineheight{1.25}\smash{\begin{tabular}[t]{l}control\\input $(u)$\end{tabular}}}}%
  \end{picture}%
\endgroup%

}
		}
		\parbox[h]{\columnwidth}{
			\centering
			\subfigure[A motion plan to the sample motion planning problem for actuated  bouncing ball system.\label{fig:samplesolutionbb}]{\includegraphics[width=0.8\columnwidth]{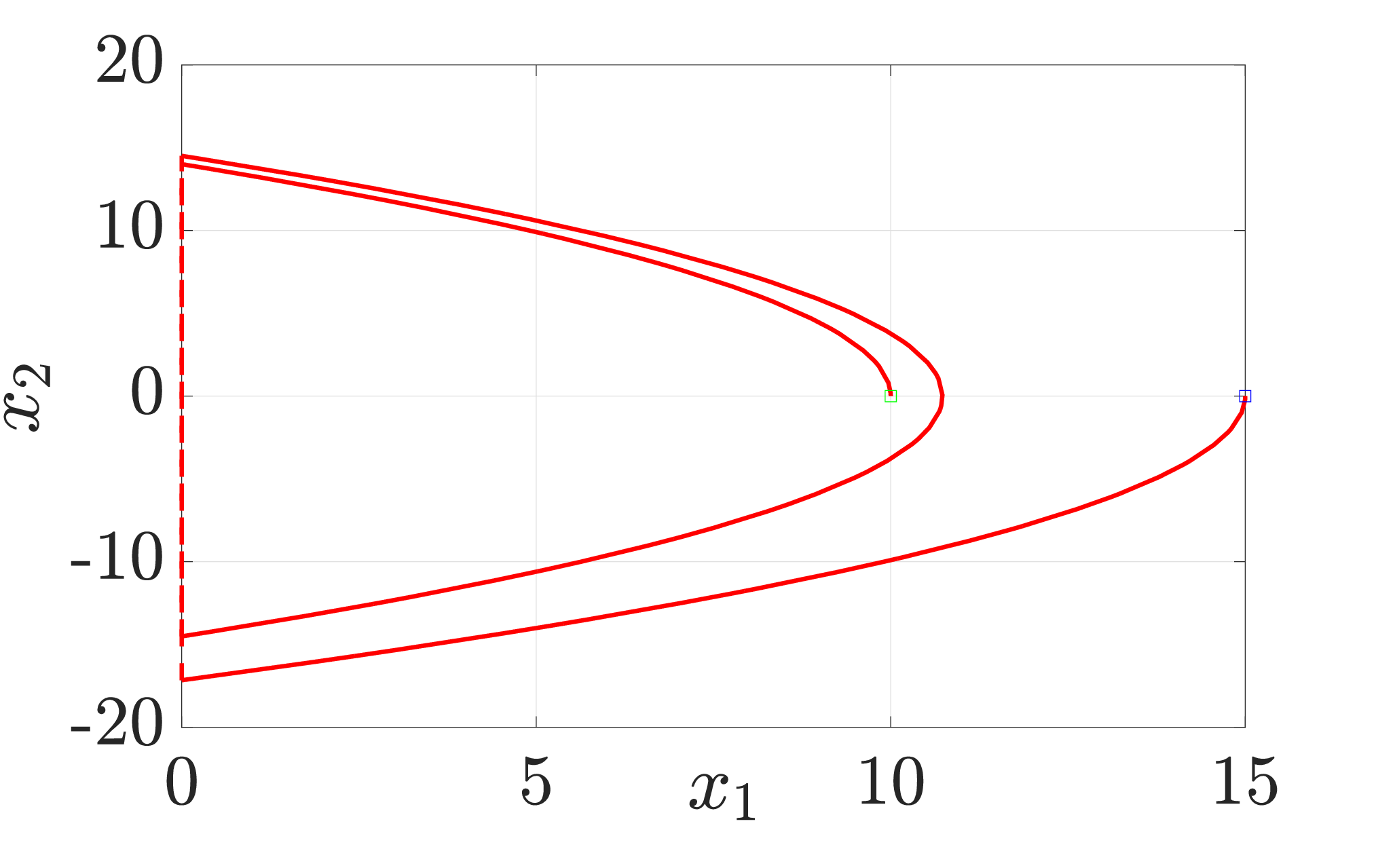}}
		}
		\caption{The actuated bouncing ball system in Example \ref{example:bouncingball}.}
	\end{figure} 
	Flow is allowed when the ball is above the surface. Hence, the flow set is
	\begin{equation}
	\label{model:bouncingballflowset}
	C := \{(x, u)\in \mathbb{R}^{2}\times \mathbb{R}: x_{1}\geq 0\}.
	\end{equation}
	At every impact, and with control input equal to zero, the velocity of the ball changes from negative to positive while the height remains the same. The dynamics at jumps of the actuated bouncing ball system is given as
	\begin{equation}
	x^{+} = \left[ \begin{matrix}
	x_{1} \\
	-\lambda x_{2}+u
	\end{matrix}\right] =: g(x, u)\qquad (x, u)\in D
	\label{conservationofmomentum}
	\end{equation}
	where $u\geq 0$ is the input and $\lambda \in (0,1)$ is the coefficient of restitution. 
	Jumps are allowed when the ball is on the surface with nonpositive velocity. Hence, the jump set is 
	\begin{equation}
	\label{model:bouncingballjumpset}
	D:= \{(x, u)\in \mathbb{R}^{2}\times \mathbb{R}: x_{1} = 0, x_{2} \leq 0, u\geq 0\}.
	\end{equation}
	
	The hybrid model of the actuated bouncing ball system is given by  (\ref{model:generalhybridsystem}), where the flow map $f$ is given in (\ref{model:bouncingballflow}), the flow set $C$ is given in (\ref{model:bouncingballflowset}), the jump map $g$ is given in (\ref{conservationofmomentum}), and the jump set $D$ is given in (\ref{model:bouncingballjumpset}).

	An example of a motion planning problem for the actuated bouncing ball system is as follows: using a bounded input signal, find a solution pair to (\ref{model:generalhybridsystem}) when the bouncing ball is released at a certain height with zero velocity and such that it reaches a given target height with zero velocity. To complete this task, not only the values of the input, but also the hybrid time domain of the input need to be planned properly such that the ball can reach the desired target. One such motion planning problem is given by defining the initial state set as $X_{0} = \{(15, 0)\}$, the final state set as $X_{f} = \{(10, 0)\}$, the unsafe set as $X_{u} =\{(x, u)\in \mathbb{R}^{2}\times \mathbb{R}: u\in (-\infty, 0]\cup[5, \infty)\}$.  The motion planning problem $\mathcal{P}$ is given as $\mathcal{P} = (X_{0}, X_{f}, X_{u}, (C, f, D, g))$. We solve this motion planning problem later in this paper.
\end{example}

\begin{example}[(Walking robot)]\label{example:biped}
	The state $x$ of the compass model of a walking robot is composed of the angle vector $\theta$ and the velocity vector $\omega$ \cite{grizzle2001asymptotically}. 
	The angle vector $\theta$ contains the planted leg angle $\theta_{p}$, the swing leg angle $\theta_{s}$, and the torso angle $\theta_{t}$. The velocity vector $\omega$ contains the planted leg angular velocity $\omega_{p}$, the swing leg angular velocity $\omega_{s}$, and the torso angular velocity $\omega_{t}$. The input $u$ is the input torque, where $u_{p}$ is the torque applied on the planted leg from the ankle, $u_{s}$ is the torque applied on the swing leg from the hip, and $u_{t}$ is the torque applied on the torso from the hip.
	The continuous dynamics of $x = (\theta, \omega)$ are obtained from the Lagrangian method and are given by $\dot{\theta} = \omega, \dot{\omega} = D_{f}(\theta)^{-1}( - C_{f}(\theta, \omega)\omega - G_{f}(\theta) + Bu) = : \alpha(x, u)$,
	where $D_{f}$ and $C_{f}$ are the inertial and Coriolis matrices, respectively, and $B$ is the actuator  relationship matrix.
		\begin{figure}[htbp] 
			\centering
	\def\svgwidth{0.5\columnwidth}
	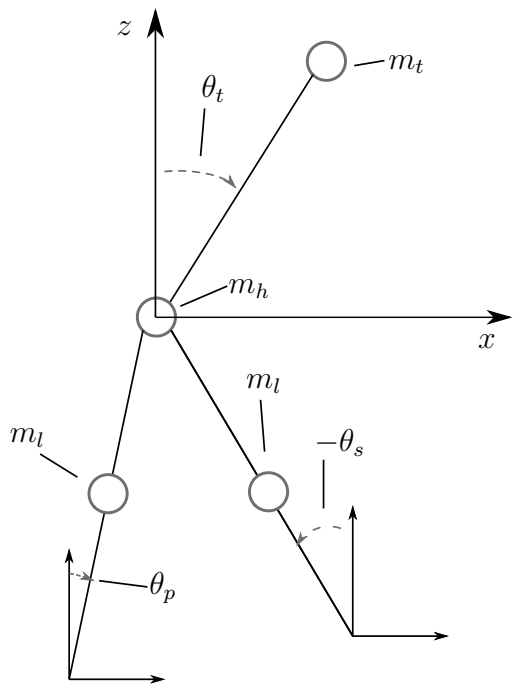

			\caption{The biped system in Example \ref{example:biped}. The angle vector $\theta$ contains the planted leg angle $\theta_{p}$, the swing leg angle $\theta_{s}$, and the torso angle $\theta_{t}$. The velocity vector $\omega$ contains the planted leg angular velocity $\omega_{p}$, the swing leg angular velocity $\omega_{s}$, and the torso angular velocity $\omega_{t}$. The input $u$ is the input torque, where $u_{p}$ is the torque applied on the planted leg from the ankle, $u_{s}$ is the torque applied on the swing leg from the hip, and $u_{t}$ is the torque applied on the torso from the hip.} \label{fig:bipedsystem} 
		\end{figure} 
	In \cite{short2018hybrid}, the input torques that produce an acceleration  $a$ for a special state $x$ are determined by a function $\mu$, defined as $\mu(x, a) := B^{-1}(D_{f}(\theta)a + C_{f}(\theta, \omega)\omega + G_{f}(\theta)).$
	By applying $u = \mu(x, a)$ to $\dot{\omega} = \alpha(x, u)$, we obtain
	$
	\label{equation:omegaflowmap}
	\dot{\omega} = a.
	$
	Then, the flow map $f$ is defined as 
	$$
	\label{model:bipedflowmap}
	f(x, a) := \left[\begin{matrix}
	\omega\\
	a
	\end{matrix}\right] \quad (x, a)\in C.
	$$
	
	Flow is allowed when only one leg is in contact with the ground. To determine if the biped has reached the end of a step, we define
	$
	h(x) := \phi_{s} - \theta_{p}$ for all $x\in \mathbb{R}^{6}
	$
	where $\phi_{s}$ denotes the step angle.
	The condition $h(x) \geq 0$ indicates that only one leg is in contact with the ground. Thus, the flow set is defined as 
	$$
	\label{model:bipedflowset}
	C:= \{(x, a)\in \mathbb{R}^{6}\times \mathbb{R}^{3}: h(x)\geq 0\}.
	$$
	Furthermore, a step occurs when the change of $h$ is such that $\theta_{p}$  is approaching $\phi_{s}$, and $h$ equals zero. Thus, the jump set $D$ is defined as
	$$
	\label{model:bipedjumpset}
	D := \{(x, a)\in \mathbb{R}^{6}\times \mathbb{R}^{3}: h(x) = 0, \omega_{p} \geq 0\}.
	$$
	
	Following \cite{grizzle2001asymptotically}, when a step occurs, the swing leg becomes the planted leg, and the planted leg becomes the swing leg. The function $\Gamma$ is defined to swap angles and velocity variables as
	$
	\label{equation:thetajumpmap}
	\theta^{+} = \Gamma(\theta).
	$
	The angular velocities after a step are determined by a contact model denoted as
	$\Omega(x) := (\Omega_{p}(x), \Omega_{s}(x), \Omega_{t}(x))$, where $\Omega_{p}$, $\Omega_{s}$, and $\Omega_{t}$ are the angular velocity of the planted leg, swing leg, and torso, respectively. Then, the jump map $g$ is defined as
	\begin{equation}
	\label{equation:bipedjumpmap}
	g(x, a) := \left[\begin{matrix}
	\Gamma(\theta)\\
	\Omega(x)
	\end{matrix}\right]\quad \forall (x, a)\in D.
	\end{equation} For more information about the contact model, see \cite[Appendix A]{grizzle2001asymptotically}.
	
	An example of a motion planning problem for the walking robot system is as follows: using a bounded input signal, find a solution pair to (\ref{model:generalhybridsystem}) associated to the walking robot system completing a step of a walking circle. One way to characterize a walking cycle is to define the final state and the initial state as the states before and after a jump occurs, respectively.  One such motion planning problem is given by defining the final state set as $X_{f} = \{(\phi_{s}, -\phi_{s}, 0, 0.1, 0.1, 0)\}$, the initial state set as $X_{0} = \{x_{0}\in \mathbb{R}^{6}: x_{0} = g(x_{f}, 0), x_{f}\in X_{f}\}$ where the input argument of $g$ can be set arbitrarily because input does not affect the value of $g$; see (\ref{equation:bipedjumpmap}), and the unsafe set as $X_{u} = \{(x, a)\in \mathbb{R}^{6}\times \mathbb{R}^{3}: a_{1} \notin [a_{1}^{\min}, a_{1}^{\max}]\text{ or }  a_{2} \notin [a_{2}^{\min}, a_{2}^{\max}]\text{ or } a_{3} \notin [a_{3}^{\min}, a_{3}^{\max}]\text{ or } (x, a)\in D  \}$, where $a_{1}^{\min}$, $a_{2}^{\min}$, and $a_{3}^{\min}$ are the lower bounds of $a_{1}$, $a_{2}$, and  $a_{3}$, respectively, and $a_{1}^{\max}$, $a_{2}^{\max}$, and $a_{3}^{\max}$ are the upper bounds of $a_{1}$, $a_{2}$, and  $a_{3}$, respectively. We also solve this motion planning problem later in this paper.
	
\end{example}

\label{section:problemstatement}

\section{A Forward Propagation Algorithm Template}
\label{section:bruteforcealgorithm}
In this section, an algorithm template, which we refer to as \emph{forward propagation algorithm template}, is proposed for the design of motion planning algorithms to solve the motion planning problem $\mathcal{P} = (X_{0}, X_{f}, X_{u}, (C, f, D, g))$. The main steps of the forward propagation algorithm template are as follows.
\begin{steps}
	\item Propagate the state from the initial state set  $X_{0}$ forward in hybrid time without intersecting with the unsafe set $X_{u}$. 
	\item During the propagation, if a solution pair to the hybrid system $(C, f, D, g)$ ending in the final state set $X_{f}$ is found, then return it as a motion plan.
\end{steps}
In the remainder of this section, first, we formalize the propagation operation used in the algorithm template and show that the result constructed by the propagation operation satisfies the requirements in Problem \ref{problem:motionplanning}. Then, we present the algorithm template framework and analyze its soundness and exactness properties. A sampling-based motion planning algorithm, which we call HyRRT, that is designed using the proposed algorithm template is presented in the forthcoming Section \ref{section:hybridRRT}.

\subsection{Propagation Operation}
The proposed template requires an operation to collect all solution pairs that start from the initial state set, without reaching the unsafe set. The definition of the forward propagation operation is formalized as follows.
\begin{definition}
		\label{definition:propagation}
		(Forward propagation operation) Given a hybrid system $\mathcal{H} = (C, f, D, g)$, an initial state set $X_{0}$, and an unsafe set $X_{u}$, the forward propagation operation with parameters $(X_{0}, X_{u}, (C, f, D, g))$ constructs a set $\mathcal{S}$ that collects all the solution pairs $\psi = (\mystatet, u)$ that satisfy the following conditions:
		\begin{enumerate}
			\item $\psi = (\mystatet, u)$ is a solution pair to $\mathcal{H}$ (see Definition \ref{definition:solution});
			\item $\mystatet(0,0)\in X_{0}$;
			\item$\psi(t, j)\notin X_{u}$ for all $(t, j)\in \dom \psi$.
		\end{enumerate}
\end{definition}
\begin{remark}
	Note that Definition \ref{definition:propagation} describes an exact searching process since it constructs a set that includes all possible solution pairs. The implementation of an algorithm constructing such set is difficult in practice because there can be infinite many possible values for the input at each hybrid time instance. Additionally, the solution pairs collected in $\mathcal{S}$ are not limited to maximal solutions, which means that any truncation of a solution pair in $\mathcal{S}$ should also be collected in~$\mathcal{S}$. 
	
	Although it is practically difficult to implement the propagation operation exactly, however, this operation can be implemented incrementally and numerically. The search trees constructed in the classic motion planning algorithms, such as RRT, can be seen as an approximation of the set $\mathcal{S}$ because any path in the search trees starting from $X_{0}$ can be seen as an element in $\mathcal{S}$. An RRT-type implementation of the propagation operation is presented in the forthcoming Section \ref{section:hybridRRT}.\EndRemark
\end{remark}
\subsection{Forward Propagation Algorithm Template and Property Analysis}
The forward propagation algorithm is given in Algorithm \ref{algo:theoreticalalgo}. The inputs of the proposed algorithm are as follows:
\begin{itemize}
	\item The motion planning problem is given as $\mathcal{P} = (X_{0}, X_{f}, X_{u}, (C, f, D, g))$ (defined in Problem \ref{problem:motionplanning}). 
	\item  The set of solution pairs $\mathcal{S}$ to store forward propagation operation result.
\end{itemize}
\begin{algorithm}[htbp]
	\caption{Forward propagation algorithm template}
	\label{algo:theoreticalalgo}
	\begin{algorithmic}[1]\raggedright
		\State Forward propagate with parameters  $(X_{0}, X_{u}, (C, f, D, g))$ (see Definition \ref{definition:propagation}) and store results in  $\mathcal{S}$.
		\If{$\exists \psi = (\phi, u)\in \mathcal{S}$ such that $\phi(T, J)\in X_{f}$ where $(T, J) = \max \dom \psi$}
		\State \Return  $\psi$.
		\Else
		\State \Return Failure.
		\EndIf
	\end{algorithmic}
\end{algorithm}
The algorithm template collects all collision-free solution pairs to $\mathcal{H}$ starting from $X_{0}$ and stores them in the set $\mathcal{S}$ (Line 1). Then, the algorithm template searches for a solution pair in the set $\mathcal{S}$ that ends in $X_{f}$ (Line 2). If such a solution pair is found, then the algorithm returns this solution pair as a motion plan (Line 3). If such a solution pair is not found, then the algorithm template returns failure (Line 5).

Soundness guarantees the correctness of the returned solution, as defined next.
\begin{definition}[(Soundness \cite{fink1994prodigy})]
	\label{definition:soundness}
	  An algorithm is said to be sound for a problem if for any possible data defining the problem (e.g., $(X_{0}, X_{f}, X_{u}, (C, f, D, g))$ in Problem \ref{problem:motionplanning}), the result returned by the algorithm is always a correct solution to the problem. 
\end{definition}
The next result shows that the proposed forward propagation algorithm template is sound for the motion planning problem in Problem \ref{problem:motionplanning}. 
\begin{theorem}
	\label{theorem:soundness}
	The forward propagation algorithm in Algorithm \ref{algo:theoreticalalgo} is sound for the motion planning problem $\mathcal{P}$ in Problem \ref{problem:motionplanning}.
\end{theorem}
\begin{proof}
	If non failure is returned (Line 3), from Definition \ref{definition:propagation}, the solution $\psi$ is a solution pair to $\hs$ that starts from $X_{0}$, and does not intersect with $X_{u}$. In addition, Line 2 in Algorithm \ref{algo:theoreticalalgo} guarantees that $\psi$ ends in $X_{f}$.  Therefore, the returned result is a motion plan to Problem \ref{problem:motionplanning}. 
	
	If the failure signal is returned (Line 5), then there does not exist an element in $\mathcal{S}$ that starts from $X_{0}$, ends in $X_{f}$, and does not intersect with $X_{u}$ (Line 2). This means that there does not exist a motion plan to Problem \ref{problem:motionplanning}. 
	
	In conclusion, the returned result is a correct solution to Problem \ref{problem:motionplanning}.
\end{proof}

Exactness guarantees that the algorithm will always return a solution when one exists, as defined next.
\begin{definition}[(Exactness \cite{goldberg1994completeness})]
	\label{definition:exactness}
	 An algorithm is said to be exact for a problem if for any possible data defining the problem (e.g., $(X_{0}, X_{f}, X_{u}, (C, f, D, g))$ in Problem \ref{problem:motionplanning}), it finds a solution when one exists.
\end{definition}
The next result shows that the proposed forward propagation algorithm template is exact for the motion planning problem in Problem \ref{problem:motionplanning}. 
\begin{theorem}
	\label{theorem:exactness}
	The forward propagation algorithm in Algorithm \ref{algo:theoreticalalgo} is an exact algorithm for the motion planning problem $\mathcal{P}$ in Problem \ref{problem:motionplanning}.
\end{theorem}
\begin{proof}
	If a solution $\psi$ to the given motion planning problem exists, then from Problem \ref{problem:motionplanning}, $\psi$ is a solution pair to $\hs$ that starts from $X_{0}$, ends at $X_{f}$, and does not have intersection with $X_{u}$. This means that $\psi$ should be collected by the forward propagation operation with parameters $(X_{0}, X_{u}, (C, f, D, g))$, be stored in $\mathcal{S}$, and the condition in Line 2 is $\texttt{true}$. Therefore, $\psi$ is returned by Algorithm \ref{algo:theoreticalalgo}.
\end{proof}
	\begin{remark}
		The exactness property guarantees that the proposed algorithm returns a solution when one exists. However, when no solution exists, the proposed algorithm is not guaranteed to return a failure signal in finite time because the forward propagation operations are not always guaranteed to be completed within finite time. This is not a unique problem of our algorithm template, but rather a general issue in motion planning algorithms; see \cite{karaman2011anytime}. 
		\EndRemark
	\end{remark}

\section{HyRRT: A Sampling-based Motion Planning Algorithm for Hybrid Systems}
\label{section:hybridRRT}
In this section, an RRT-type motion planning algorithm for hybrid systems, called HyRRT, is proposed. This algorithm is an implementation of the algorithm template in Algorithm~\ref{algo:theoreticalalgo} that approximates the propagation operation introduced in Definition \ref{definition:propagation} by incrementally constructing a search tree.
	\subsection{Overview}
	\label{section:algorithmoverview}
	HyRRT searches for a motion plan by incrementally constructing a search tree. 
	The search tree is modeled by a directed tree. A directed tree $\mathcal{T}$ is a pair $\mathcal{T} = (V, E)$, where $V$ is a set whose elements are called vertices and $E$ is a set of paired vertices whose elements are called edges. The edges in the directed tree are directed, which means the pairs of vertices that represent edges are ordered. The set of edges $E$ is defined as
	$
	E \subseteq \{(v_{1}, v_{2}): v_{1}\in V, v_{2}\in V, v_{1}\neq v_{2}\}.
	$
	The edge $e = (v_{1}, v_{2})\in E$ represents an edge from $v_{1}$ to $v_{2}$.  
	A path in $\mathcal{T} = (V, E)$ is a sequence of vertices
	$
	p = (v_{1}, v_{2}, ..., v_{k})
	$ such that $(v_{i}, v_{i + 1})\in E$ for all $i= 1, 2,..., k - 1$.
	 
	\begin{figure}[htbp]
		\centering
		\subfigure[States and solution pairs.\label{fig:searchgraph_statespace}]{%
	\def\svgwidth{0.7\columnwidth}
	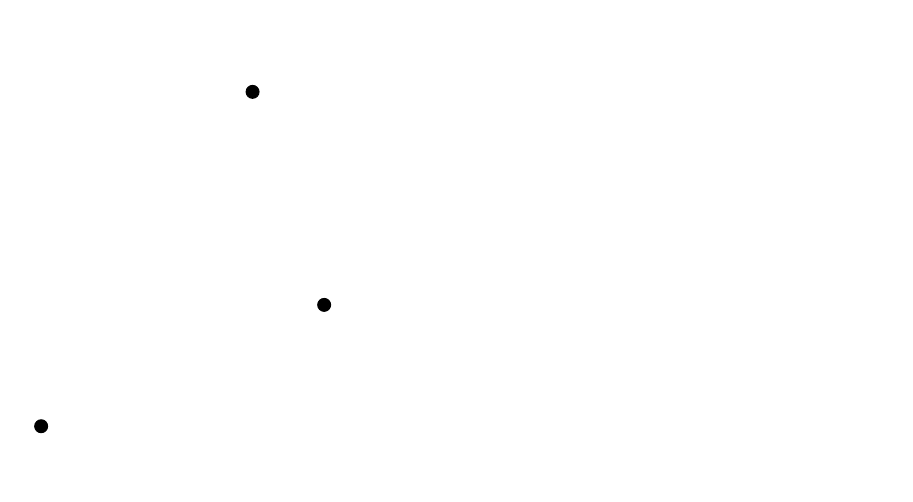
}
		\subfigure[Search tree associated with the states and solution pairs in Figure \ref{fig:searchgraph_statespace}.\label{fig:searchgraph_searchgraph}]{%
	\def\svgwidth{0.7\columnwidth}
	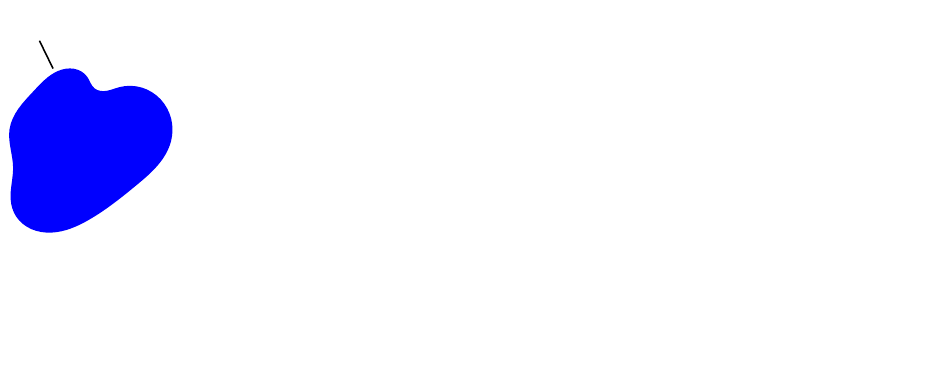
}
		\caption{The association between states/solution pairs and the vertices/edges in the search tree. The blue region denotes $X_{0}$, the green region denotes $X_{f}$, and the black region denotes $X_{u}$. The dots and lines between dots in Figure \ref{fig:searchgraph_searchgraph} denote the vertices and edges associated with the states and solution pairs in Figure \ref{fig:searchgraph_statespace}. The path $p = (v_{1}, v_{2}, v_{3}, v_{cur}, v_{new})$ in the search graph in Figure \ref{fig:searchgraph_searchgraph} represents the solution pair $\tilde{\psi}_{p} = \overline{\psi}_{e_{1}}|\overline{\psi}_{e_{2}}|\overline{\psi}_{e_{3}}|\psi_{new}$ in Figure \ref{fig:searchgraph_statespace}.}
		\label{fig:searchgraph}
	\end{figure}
	
	Each vertex in the search tree $\mathcal{T}$ is associated with a state value of $\mathcal{H}$. Each edge in the search tree is associated with a solution pair to $\mathcal{H}$ that connects the state values associated with their endpoint vertices. The state value associated with vertex $v\in V$ is denoted as $\overline{x}_{v}$ and the solution pair associated with edge $e\in E$ is denoted as $\overline{\psi}_{e}$, as shown in Figure \ref{fig:searchgraph}. 
	
	The concatenation operation is formulated as follows and heavily used in constructing the search tree. 
	\begin{definition}[(Concatenation operation)]
		\label{definition:concatenation}
		 Given two functions $\mystatet_{1}: \dom \mystatet_{1} \to \mathbb{R}^{n}$ and $\mystatet_{2}:\dom \mystatet_{2} \to \mathbb{R}^{n}$, where $\dom \mystatet_{1}$ and $\dom \mystatet_{2}$ are hybrid time domains, $\mystatet_{2}$ can be concatenated to $\mystatet_{1}$ if $\mystatet_{1}$ is compact and $\mystatet: \dom \mystatet \to \mathbb{R}^n$ is the concatenation of $\mystatet_{2}$ to $\mystatet_{1}$, denoted $\mystatet = \mystatet_{1}|\mystatet_{2}$, namely,
		\begin{enumerate}[label = \arabic*)]
			\item $\dom \mystatet = \dom \mystatet_{1} \cup (\dom \mystatet_{2} + \{(T, J)\}) $, where $(T, J) = \max$ $\dom \mystatet_{1}$ and the plus sign denotes Minkowski addition;
			\item $\mystatet(t, j) = \mystatet_{1}(t, j)$ for all $(t, j)\in \dom \mystatet_{1}\backslash \{(T, J)\}$ and $\mystatet(t, j) = \mystatet_{2}(t - T, j - J)$ for all $(t, j)\in \dom \mystatet_{2} + \{(T, J)\}$.
		\end{enumerate}
	\end{definition}
	The following example in Figure \ref{fig:concatenation} illustrates the concatenation of the functions on hybrid time domains. Let $\mystatet_{1}$ and $\mystatet_{2}$ be two functions defined on hybrid time domains and $(T, J) = \max \dom \mystatet_{1}$. When $\mystatet_{2}$ is concatenated to $\mystatet_{1}$, the domain of the concatenation is constructed by translating the domain of $\mystatet_{2}$ by $(T, J)$ and concatenating the translated domain to the domain of $\mystatet_{1}$. Particularly, if $\mystatet_{1}$ and $\mystatet_{2}$ are hybrid arcs, their concatenation is not guaranteed to be a hybrid arc. The reason is that the concatenation result may not be continuous at $(T, J)$. The concatenation of two trajectories is illustrated in Figure \ref{fig:concatenation}. In the figure, the solid lines denote the flows and the dotted lines denote the jumps. In Figure \ref{fig:concatenation1},  the trajectories $\mystatet_{1}$ and $\mystatet_{2}$ denote the original trajectories and the trajectory $\mystatet$ in Figure \ref{fig:concatenation2} denotes the concatenation of $\mystatet_{2}$ to $\mystatet_{1}$.  
	\begin{figure}[htbp] 
	\centering
	\subfigure[The original trajectory $\mystatet_{1}$ and $\mystatet_{2}$.]{%
	\def\svgwidth{0.5\columnwidth}
\begingroup%
  \makeatletter%
  \providecommand\color[2][]{%
    \errmessage{(Inkscape) Color is used for the text in Inkscape, but the package 'color.sty' is not loaded}%
    \renewcommand\color[2][]{}%
  }%
  \providecommand\transparent[1]{%
    \errmessage{(Inkscape) Transparency is used (non-zero) for the text in Inkscape, but the package 'transparent.sty' is not loaded}%
    \renewcommand\transparent[1]{}%
  }%
  \providecommand\rotatebox[2]{#2}%
  \newcommand*\fsize{\dimexpr\f@size pt\relax}%
  \newcommand*\lineheight[1]{\fontsize{\fsize}{#1\fsize}\selectfont}%
  \ifx\svgwidth\undefined%
    \setlength{\unitlength}{466.81004247bp}%
    \ifx\svgscale\undefined%
      \relax%
    \else%
      \setlength{\unitlength}{\unitlength * \real{\svgscale}}%
    \fi%
  \else%
    \setlength{\unitlength}{\svgwidth}%
  \fi%
  \global\let\svgwidth\undefined%
  \global\let\svgscale\undefined%
  \makeatother%
  \begin{picture}(1,0.52674018)%
    \lineheight{1}%
    \setlength\tabcolsep{0pt}%
    \put(0,0){\includegraphics[width=\unitlength,page=1]{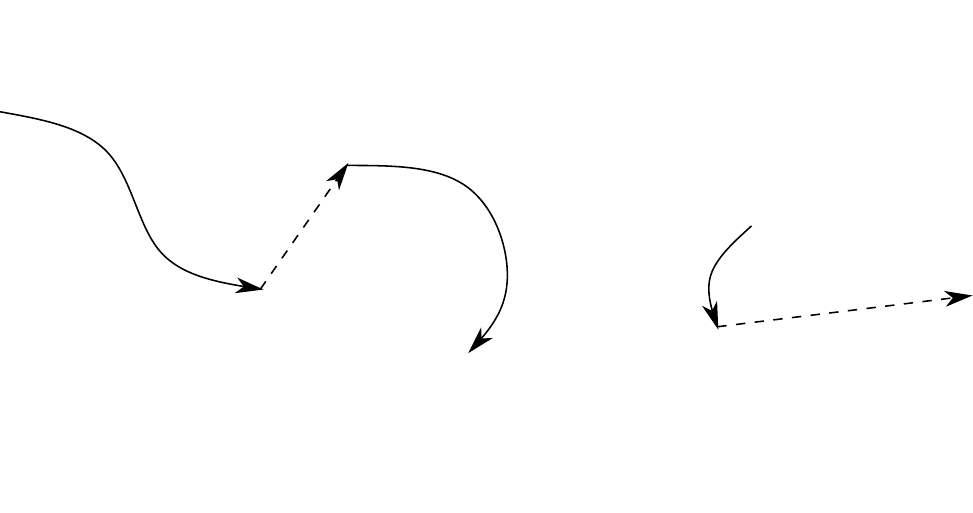}}%
    \put(0.04588105,0.28613208){\makebox(0,0)[lt]{\lineheight{1.25}\smash{\begin{tabular}[t]{l}$\phi_{1}$\end{tabular}}}}%
    \put(0.07199552,0.50643615){\makebox(0,0)[lt]{\lineheight{1.25}\smash{\begin{tabular}[t]{l}$\phi_{1}(0, 0)$\end{tabular}}}}%
    \put(0,0){\includegraphics[width=\unitlength,page=2]{concatenation1.pdf}}%
    \put(0.39445423,0.01610917){\makebox(0,0)[lt]{\lineheight{1.25}\smash{\begin{tabular}[t]{l}$\phi_{1}(T, J)$\end{tabular}}}}%
    \put(0.73851136,0.41649475){\makebox(0,0)[lt]{\lineheight{1.25}\smash{\begin{tabular}[t]{l}$\phi_{2}(0, 0)$\end{tabular}}}}%
    \put(0.76973638,0.13857947){\makebox(0,0)[lt]{\lineheight{1.25}\smash{\begin{tabular}[t]{l}$\phi_{2}$\end{tabular}}}}%
    \put(0,0){\includegraphics[width=\unitlength,page=3]{concatenation1.pdf}}%
  \end{picture}%
\endgroup%

 \label{fig:concatenation1} }
	\centering
	\subfigure[The trajectory $\phi$ resulting from concatenating $\mystatet_{2}$ to $\mystatet_{1}$.]{%
	\def\svgwidth{0.5\columnwidth}
\begingroup%
  \makeatletter%
  \providecommand\color[2][]{%
    \errmessage{(Inkscape) Color is used for the text in Inkscape, but the package 'color.sty' is not loaded}%
    \renewcommand\color[2][]{}%
  }%
  \providecommand\transparent[1]{%
    \errmessage{(Inkscape) Transparency is used (non-zero) for the text in Inkscape, but the package 'transparent.sty' is not loaded}%
    \renewcommand\transparent[1]{}%
  }%
  \providecommand\rotatebox[2]{#2}%
  \newcommand*\fsize{\dimexpr\f@size pt\relax}%
  \newcommand*\lineheight[1]{\fontsize{\fsize}{#1\fsize}\selectfont}%
  \ifx\svgwidth\undefined%
    \setlength{\unitlength}{369.04484598bp}%
    \ifx\svgscale\undefined%
      \relax%
    \else%
      \setlength{\unitlength}{\unitlength * \real{\svgscale}}%
    \fi%
  \else%
    \setlength{\unitlength}{\svgwidth}%
  \fi%
  \global\let\svgwidth\undefined%
  \global\let\svgscale\undefined%
  \makeatother%
  \begin{picture}(1,0.44695498)%
    \lineheight{1}%
    \setlength\tabcolsep{0pt}%
    \put(0,0){\includegraphics[width=\unitlength,page=1]{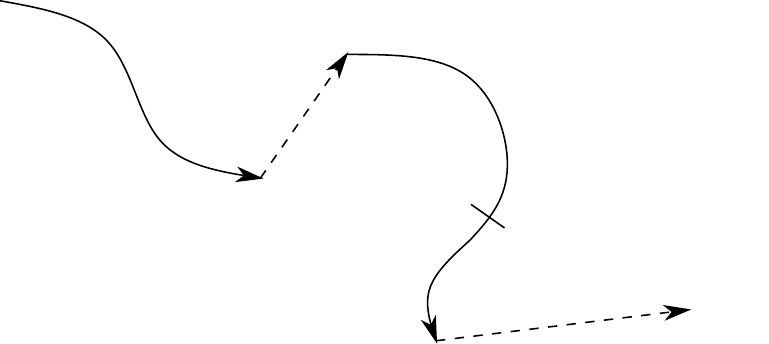}}%
    \put(0.67863118,0.16705473){\makebox(0,0)[lt]{\lineheight{1.25}\smash{\begin{tabular}[t]{l}concatenation point\end{tabular}}}}%
    \put(0.07023412,0.33480688){\makebox(0,0)[lt]{\lineheight{1.25}\smash{\begin{tabular}[t]{l}$\phi$\end{tabular}}}}%
  \end{picture}%
\endgroup%

 \label{fig:concatenation2} }
	\caption{The concatenation $\phi$ of trajectory $\mystatet_{2}$ to trajectory $\mystatet_{1}$.}\label{fig:concatenation} 
	\end{figure} 
	
	Proposition \ref{theorem:concatenationsolution}  below shows that the concatenation of solution pairs satisfies the definition of solution pair in Definition \ref{definition:solution} under mild conditions.
	\begin{proposition}
		\label{theorem:concatenationsolution}
		Given two solution pairs $\psi_{1} = (\mystatet_{1}, \myinputt_{1})$ and $\psi_{2} = (\mystatet_{2}, \myinputt_{2})$ to a hybrid system $\mathcal{H}$, their concatenation $\psi = (\mystatet, \myinputt) = (\mystatet_{1}|\mystatet_{2}, \myinputt_{1}|\myinputt_{2})$, denoted $\psi = \psi_{1}|\psi_{2}$,  is a solution pair to $\mathcal{H}$ if the following hold:
		\begin{enumerate}[label = \arabic*)]
			\item $\psi_{1} = (\mystatet_{1}, \myinputt_{1})$ is compact;
			\item $\mystatet_{1}(T, J) = \mystatet_{2}(0,0)$, where $(T, J) = \max \dom \psi_{1}$;
			\item If both $I_{\psi_{1}}^{J}$ and $I_{\psi_{2}}^{0}$ have nonempty interior, where, for each $j\in \{0, J\}$, $I_{\psi}^{j} = \{t: (t, j)\in \dom \psi\}$ and $(T, J) = \max \dom \psi_{1}$, then $\psi_{2}(0, 0)\in C$.
		\end{enumerate}
	\end{proposition}
\begin{proof}
	See Appendix \ref{section:concatenationsolution}.
\end{proof}
	\begin{remark}
		Item 1 in Proposition \ref{theorem:concatenationsolution} guarantees that $\psi_{2}$ can be concatenated to $\psi_{1}$. Definition \ref{definition:concatenation} suggests that if $\psi_{2}$ can be concatenated to $\psi_{1}$, $\psi_{1}$ is required to be compact. 
		Item 2 in Proposition \ref{theorem:concatenationsolution} guarantees that the $\phi$ is a hybrid arc or $\psi$ satisfies item 3 in Definition \ref{definition:solution} at hybrid time $(T, J)$, where $(T, J) = \max \dom \psi_{1}$.
		Item 3 in Proposition~\ref{theorem:concatenationsolution} guarantees that the concatenation result $\psi$ satisfies item 2 in Definition \ref{definition:solution} at hybrid time $(T, J)$. Note that item 2 therein does not require that $\psi_{1}(T, J)\in C$ and $\psi_{2}(0,0)\in C$ since $T\notin \interior I_{\psi_{1}}^{J}$ and $0\notin \interior I_{\psi_{2}}^{0}$. However, $T$ may belong to the interior of $I_{\psi}^{J}$ after concatenation. Hence, item 3 guarantees that if $T$ belongs to the interior of $I_{\psi}^{J}$ after concatenation, then $\psi(T, J)\in C$.
		$\hfill \blacktriangle$
	\end{remark}

	The solution pair that the path $p = (v_{1}, v_{2}, ..., v_{k})$ represents is the concatenation of all the solutions associated with the edges therein, namely,
	\begin{equation}
	\label{equation:concatenationpath}
	\tilde{\psi}_{p} := \overline{\psi}_{(v_{1}, v_{2})}|\overline{\psi}_{(v_{2}, v_{3})}|\ ...\  |\overline{\psi}_{(v_{k-1}, v_{k})}	
	\end{equation}
	where $\tilde{\psi}_{p}$ denotes the solution pair associated with the path~$p$. 
	An example of the path $p$ and its associated solution pair $\tilde{\psi}_{p}$ is shown in Figure \ref{fig:searchgraph}.
	
	The proposed HyRRT algorithm requires a library of possible inputs. The input library $(\mathcal{U}_{C}, \mathcal{U}_{D})$ includes the input signals that can be applied during flows (collected in $\mathcal{U}_{C}$) and the input values that can be applied at jumps (collected in $\mathcal{U}_{D}$). More details about the input library are presented in the forthcoming Section \ref{section:inputlibrary}.

	Next, we introduce the main steps executed by HyRRT. Given the motion planning problem $\mathcal{P} = (X_{0}, X_{f}, X_{u}, (C, f, D, g))$ and the input library~$(\mathcal{U}_{C}, \mathcal{U}_{D})$, HyRRT performs the following steps:
	\begin{steps}
		\item Sample a finite number of points from $X_{0}$ and initialize a search tree $\mathcal{T} = (V, E)$ by adding vertices associated with each sampling point. 
		\item Randomly select one regime among flow regime and jump regime for the evolution of $\mathcal{H}$.
		\item Randomly select a point $x_{rand}$ from $C'$ ($D'$) if the flow (jump, respectively) regime is selected in Step 2.
		\item Find the vertex associated with the state value that has minimal Euclidean distance to $x_{rand}$, denoted $v_{cur}$, as is shown in Figure \ref{fig:searchgraph_searchgraph}. 
		\item Randomly select an input signal (value) from $\mathcal{U}_{C}$ ($\mathcal{U}_{D}$) if the flow (jump, respectively) regime is selected. Then, compute a solution pair using continuous (discrete, respectively) dynamics simulator starting from $\overline{x}_{v_{cur}}$ with the selected input applied, denoted $\psi_{new} = (\mystatet_{new}, \myinputt_{new})$. Denote the final state of $\mystatet_{new}$ as $x_{new}$, as is shown in Figure \ref{fig:searchgraph_statespace}. If $\psi_{new}$ does not intersect with $X_{u}$, add a vertex $v_{new}$ associated with $x_{new}$ to $V$ and an edge $(v_{cur}, v_{new})$ associated with $\psi_{new}$ to $E$. Then, go to \textbf{Step 2}.
	\end{steps}
	In the remainder of this section, first, we formalize the input library for the motion planning problem for hybrid systems. Then, the continuous dynamics simulator and discrete dynamics simulator to compute a new solution pair from the given starting state and input are introduced. After those components being nicely introduced, the HyRRT algorithm to solve the motion planning problem for hybrid systems is presented.

\subsection{Hybrid Input Library}
\label{section:inputlibrary}
HyRRT requires a library of inputs to simulate solution pairs. Note that inputs are constrained by the flow set $C$ and the jump set $D$ of the hybrid system $\hs$. Specifically, given $C$ and $D$, the set of input signals allowed during flows, denoted  $\mathcal{U}_{C}$, and the set of input values at jumps, denoted $\mathcal{U}_{D}$, are described as follows. 
\begin{enumerate}
	\item The input signal applied during flows is a continuous-time signal, denoted $\tilde{\myinputt}$, that is specified by a function from an interval of time of the form $[0, t^{*}]$ to $U_{C}$, namely, for some $t^{*}\in \mathbb{R}_{\geq 0}$,
	$$
	\tilde{\myinputt}: [0, t^{*}]\to U_{C}.
	$$
	Definition \ref{definition:hybridinput} also requires that $\tilde{\myinputt}$ is Lebesgue measurable and locally essentially bounded.
	Then, the set $\mathcal{U}_{C}$ collects each such $\tilde{\myinputt}$. Given $\tilde{\myinputt}\in\mathcal{U}_{C}$, the functional $\overline{t}: \mathcal{U}_{C} \to [0, \infty)$ returns the time duration of $\tilde{\myinputt}$. Namely, given $\tilde{\myinputt}: [0, t^{*}]\to U_{C}$, $\overline{t}(\tilde{\myinputt}) = t^{*}$.
	\item The input applied at a jump is specified by the set $U_{D}$. The set $\mathcal{U}_{D}$ collects possible input values that can be applied at jumps, namely, 
	$$
		\mathcal{U}_{D}\subset U_{D}.
	$$
\end{enumerate}
The pair of sets $(\mathcal{U}_{C}, \mathcal{U}_{D})$ defines the input library, denoted $\mathcal{U}$, namely,
$$
	\mathcal{U} := (\mathcal{U}_{C}, \mathcal{U}_{D}).
$$

\label{section:samplingmotionprimitives}
{\bf Input Library Construction Procedure:} A procedure to construct $\mathcal{U}_{C}$ using constant inputs is given as follows:
\begin{steps}
	\item Set $T_{m}$ to a positive constant. Construct the safe input set during the flows from $U_{C}$ as $U^s_C := \{u\in U_{C}: \exists x\in C', (x, u)\notin X_{u}\}$. 
	\item For each point $u^{s}\in U_{C}^{s}$ and $t_{m}\in (0, T_{m}]$, construct an input signal $[0, t_{m}]\to \{u^{s}\}$ and add it to $\mathcal{U}_{C}$. 
\end{steps}
Set $\mathcal{U}_{D}$ can be constructed as $\mathcal{U}_{D} := \{u\in U_{D}: \exists x\in D', (x, u)\notin X_{u}\}$ from $U_{D}$. 

\subsection{Continuous Dynamics Simulator}
\label{section:flowsystemsimulator}
HyRRT requires a simulator to compute the solution pair starting from a given initial state $x_{0}\in C'$ with a given input signal $\tilde{\myinputt}\in \mathcal{U}_{C}$ applied, following continuous dynamics. The initial state $x_{0}$, the flow set $C$, and the flow map $f$ are used in the simulator. 

Note that when the simulated solution enters the intersection between the flow set $C$ and the jump set $D$, it can either keep flowing or jump. In \cite{sanfelice2013toolbox}, the hybrid system simulator HyEQ uses a scalar priority option flag $\texttt{rule}$ to show whether the simulator gives priority to jumps ($\texttt{rule}= 1$), priority to flows ($\texttt{rule}= 2$), or no priority ($\texttt{rule}= 3$) when both $x\in C$ and $x\in D$ hold. When no priority is selected, then the simulator randomly selects to flow or jump. In this paper, the case $\texttt{rule} = 3$ is not considered\endnote{In this paper, $\texttt{rule} = 3$ is not considered since the input signal $\tilde{\myinputt}$ is randomly selected from $\mathcal{U}_{C}$ which satisfies the forthcoming Assumption \ref{assumption:inputlibrary}. Therefore, the time duration of the simulation result has been randomized. However, $\texttt{rule} = 3$ requires an additional random process to determine whether to proceed with the flow or the jump when the simulation result enters $C\cap D$ leading to a redundant random procedure.}.

The proposed simulator should be able to solve the following problem.
\begin{problem}
	\label{problem:flowsimulator}
	Given the flow set $C$, the flow map $f$, and the jump set $D$ of a hybrid system $\mathcal{H}$ with input $u\in \mathbb{R}^{m}$ and state $x\in \mathbb{R}^{n}$, a priority option flag $\texttt{rule}\in \{1,  2\}$, an initial state $x_{0}\in \mathbb{R}^{n}$, and an input signal $\tilde{\myinputt}\in \mathcal{U}_{C}$ such that $(x_{0}, \tilde{\myinputt}(0))\in C$, find a pair $(\mystatet, \myinputt):  [0, t^{*}]\times \{0\} \to \mathbb{R}^{n} \times \mathbb{R}^{m}$, where $t^{*}\in [0, \overline{t}(\tilde{\myinputt})]$, such that the following hold:
	\begin{enumerate}
		\item $\phi(0, 0) = x_{0}$;
		\item For all $t\in [0, t^{*}]$, $\myinputt(t, 0) = \tilde{\myinputt}(t)$;
		\item If $[0, t^{*}]$ has nonempty interior,
		\begin{enumerate}
			\item the function $t\mapsto \phi(t, 0)$ is locally absolutely continuous,
			\item for all $t\in (0, t^{*})$,
				$$
					\begin{aligned}
					(\phi(t, 0),\myinputt(t, 0))&\in C\backslash D &\text{ if }\texttt{rule} = 1\\
						(\phi(t, 0),\myinputt(t, 0))&\in C & \text{ if }\texttt{rule} = 2,\\
					\end{aligned}
				$$
			\item for almost all $t\in [0, t^{*}]$,
			\begin{equation}
			\label{equation:differentialequation}
			\frac{\text{d}}{\text{d} t} {\phi}(t,0) = f(\phi(t,0), \myinputt(t,0)).
			\end{equation}
		\end{enumerate}
	\end{enumerate}
\end{problem}
\begin{remark}
	In general, a solution pair to $\hs$ that solves Problem \ref{problem:flowsimulator} may not be unique. Note that we can impose assumptions to get uniqueness as in \cite[Proposition 2.11]{goebel2009hybrid}, as follows:
		\begin{enumerate}
			\item for every $x_{0}\in \overline{C}\backslash D$, $T>0$ and $\tilde{\myinputt}\in \mathcal{U}_{C}$, if two absolutely continuous $z_{1},z_{2} : [0,T] \to \myreals[n]$ are such that  $z_{i}(t) = f(z_{i}(t), \tilde{\myinputt}(t))$ for almost all $t \in [0,T]$, $(z_{i}(t), \tilde{\myinputt}(t)) \in C$ for all $t \in (0,T]$, and $z_{i}(0) = x_{0}$, $i =1, 2$, then $z_{1}(t)=z_{2}(t)$ for all $t \in [0,T]$;
			\item for every $x_{0} \in C\cap D$, there does not exist $\epsilon>0$ and an absolutely continuous function $z : [0,\epsilon] \to \myreals[n]$ such that $z(0) = x_{0}$, $\dot{z}(t) = f(z(t), \myinputt(t))$ for almost all $t \in [0,\epsilon]$ and $(z_{i}(t), \myinputt(t)) \in C$ for all $t \in (0,\epsilon]$.
		\end{enumerate}
\end{remark}

In this paper, the simulator is designed to simulate a maximal solution pair $(\mystatet, \myinputt)$ solving Problem \ref{problem:flowsimulator}. The definition of maximal solution is given as follows; see \cite{goebel2009hybrid}. 
\begin{definition}[(Maximal solution)]
	 A solution pair $\psi$ to $\hs$ that solves Problem \ref{problem:flowsimulator} is said to be maximal if there does not exist another solution pair $\psi'$ to $\hs$ that solves Problem \ref{problem:flowsimulator} such that $\dom \psi$ is a proper subset of $\dom \psi'$ and $\psi(t, 0) = \psi'(t, 0)$ for all $t\in \dom_{t} \psi$.
\end{definition}

The module to simulate the maximal solution pair to $\hs$ that solves Problem \ref{problem:flowsimulator} is called the \emph{continuous dynamics simulator}. 
The inputs of this module are the flow set $C$, the flow map $f$, the jump set $D$, the priority option $\texttt{rule}$, the initial state $x_{0}$, and the input signal $\tilde{\myinputt}$. The output of this module is the maximal solution pair $(\phi, u)$ to $\hs$ that solves Problem \ref{problem:flowsimulator}. This module is denoted as 
\begin{equation}
\label{simulator:flow}
(\phi, u) \leftarrow \texttt{continuous\_simulator}(C, f, D, \texttt{rule},  x_{0}, \tilde{\myinputt}).
\end{equation}
%

The continuous dynamics simulator performs the following steps. Given the flow set $C$, the flow map $f$, the jump set $D$  the priority option $\texttt{rule}$, the initial state $x_{0}$, and the input signal $\tilde{\myinputt}$
\begin{steps}
	\item Solve for $\hat{\phi}: [0, \overline{t}(\tilde{\myinputt})]\to \mathbb{R}^{n}$ to satisfy
	\begin{equation}
	\label{equation:flowequationintegration}
			\hat{\phi}(t) = x_{0} + \int_{0}^{t}f(\hat{\phi}(\tau),\tilde{\myinputt}(\tau))d\tau \quad t\in [0, \overline{t}(\tilde{\myinputt})].
	\end{equation} 
	\item Calculate the largest time $t$ such that over $[0, t)$, $(\mystatet, \myinputt)$ is in $C\backslash D$ if $\texttt{rule} = 1$, or, if $\texttt{rule} = 2$, is in $C$, as follows:
	\begin{equation}
	\label{equation:integrationduration}
		\hspace{-1cm}\hat{t} := 
			\begin{cases}
			\max \{t\in [0, \overline{t}(\tilde{\myinputt})]: (\hat{\phi}(t'), \tilde{\myinputt}(t'))\in C\backslash D, \\
			\forall t'\in (0, t)\}\hfill\text{ if }\texttt{rule} = 1,\\
			\max \{t\in [0, \overline{t}(\tilde{\myinputt})]:  (\hat{\phi}(t'), \tilde{\myinputt}(t'))\in C,\\
			\forall t'\in (0, t)\}
			\hfill\text{ if } \texttt{rule} = 2.\\
			\end{cases}
	\end{equation}
	\item Construct the solution function pair $(\mystatet, \myinputt): [0, \hat{t}]\times \{0\}\to \mathbb{R}^{n}\times \mathbb{R}^{m}$ by 
	\begin{equation}
	\label{equation:flowsegment}
		\begin{aligned}
			\phi(t, 0) = \hat{\phi}(t),
			\myinputt(t, 0) = \tilde{\myinputt}(t)\  \forall t\in [0, \hat{t}].\\
		\end{aligned}
	\end{equation}
\end{steps}
The solution function pair $(\phi, u)$ is a maximal solution to $\hs$ that solves Problem \ref{problem:flowsimulator} and, hence, the output of module $\texttt{continuous\_simulator}$. The construction of $\hat{\phi}$ in Step 1 can be approximated by employing numerical integration methods. To determine $\hat{t}$ in Step 2, zero-crossing detection algorithms can be used to detect the largest time $t$ such that over $[0, t)$, $(\mystatet, \myinputt)$ is in $C\backslash D$ if $\texttt{rule} = 1$, or, if $\texttt{rule} = 2$, is in $C$. For the computation framework that implements the simulator of continuous dynamics, see Appendix \ref{section:computationsimulation}.

\subsection{Discrete Dynamics Simulator}\label{section:jumpsystemsimulator}
HyRRT algorithm also requires a simulator to compute the solution pair with a single jump, starting from an initial state $x_{0}\in D'$ with a given input value $u_{D}\in \mathcal{U}_{D}$ applied. Such a simulator only requires evaluating $g$ as the following problem states.
\begin{problem}
	\label{problem:jumpsimulator}
	Given the jump set $D$ and jump map $g$ of hybrid system $\mathcal{H}$ with input $u\in \mathbb{R}^{m}$, state $x\in \mathbb{R}^{n}$, an initial state $x_{0}\in \mathbb{R}^{n}$, and an input value $u_{D}\in \mathcal{U}_{D}$ such that $(x_{0}, u_{D})\in D$, find a pair $(\phi, \myinputt):  \{0\}\times \{0, 1\} \to \mathbb{R}^{n} \times \mathbb{R}^{m}$ such that the following hold:
	\begin{enumerate}
		\item $\phi(0, 0) = x_{0}$;
		\item $\myinputt(0, 0) = u_{D}$;
		\item $\phi(0, 1) = g(\phi(0, 0), \myinputt(0, 0))$.
	\end{enumerate}
\end{problem}
Problem \ref{problem:jumpsimulator} can be solved by constructing a function pair $(\phi, \myinputt):  \{0\}\times \{0, 1\} \to \mathbb{R}^{n} \times \mathbb{R}^{m}$ in the following way:
\begin{equation}\label{equation:jumpsegment}
\begin{aligned}
\phi(0, 0) \leftarrow x_{0},\quad
\phi(0, 1) \leftarrow g( x_{0}, u_{D})\\
\myinputt(0, 0) \leftarrow u_{D},\quad
\myinputt(0, 1) \leftarrow p\in \mathbb{R}^{m}.
\end{aligned}
\end{equation}
In (\ref{equation:jumpsegment}), we can implement $\myinputt(0, 1) \leftarrow p \in \mathbb{R}^{m}$ by selecting an arbitrary point $p$ in $ \mathbb{R}^{m}$ such that $(\phi(0, 1), p)\notin X_{u}$.

The function pair in (\ref{equation:jumpsegment}) can be constructed by a module called \emph{discrete dynamics simulator}. 
The inputs of this module are the jump set $D$, the jump map $g$, the initial state $x_{0}\in D'$, and the input $u_{D}\in \mathcal{U}_{D}$ such that $(x_{0}, u_{D})\in D$. The output of this module is the solution pair $(\phi, \myinputt)$ constructed in (\ref{equation:jumpsegment}). This module is denoted as
\begin{equation}
\label{simulator:jump}
(\phi, \myinputt) \leftarrow \texttt{discrete\_simulator}(D, g, x_{0}, u_{D}).
\end{equation}
\subsection{HyRRT Algorithm}\label{section:hybridRRTframewrok}
Following the overview in Section \ref{section:algorithmoverview}, the proposed algorithm is given in Algorithm \ref{algo:hybridRRT}. The inputs of Algorithm \ref{algo:hybridRRT} are the problem $\mathcal{P} = (X_{0}, X_{f}, X_{u}, (C, f, D, g))$, the input library  $ (\mathcal{U}_{C}, \mathcal{U}_{D})$, a parameter $p_{n}\in (0, 1)$, which tunes the probability of proceeding with the flow regime or the jump regime, an upper bound $K\in \mathbb{N}_{>0}$ for the number of iterations to execute, and two tunable sets $X_{c}\supset \overline{C'}$ and $X_{d}\supset D'$, which act as constraints in finding a closest vertex to $x_{rand}$. HyRRT is implemented in the following algorithm:
\begin{algorithm}[htbp]
	\caption{HyRRT algorithm}
	\label{algo:hybridRRT}
	\hspace*{\algorithmicindent} \textbf{Input: $X_{0}, X_{f}, X_{u}, \mathcal{H} = (C, f, D, g), (\mathcal{U}_{C}, \mathcal{U}_{D}), p_{n} \in (0, 1)$, $K\in \mathbb{N}_{>0}$}
	\begin{algorithmic}[1]
		\State $\mathcal{T}.\texttt{init}(X_{0})$.
		\For{$k = 1$ to $K$}
		\State randomly select a real number $r$ from $[0, 1]$.
		\If{$r\leq p_{n}$}
		\State $x_{rand}\leftarrow \texttt{random\_state}(\overline{C'})$.
		\State $\texttt{extend}(\mathcal{T}, x_{rand}, (\mathcal{U}_{C}, \mathcal{U}_{D}), \mathcal{H}, X_{u}, X_{c})$.
		\Else
		\State $x_{rand}\leftarrow \texttt{random\_state}(D')$.
		\State $\texttt{extend}(\mathcal{T}, x_{rand}, (\mathcal{U}_{C}, \mathcal{U}_{D}), \mathcal{H}, X_{u}, X_{d})$.
		\EndIf
		\EndFor
		\State \Return $\mathcal{T}$.
	\end{algorithmic}
\end{algorithm}
\begin{algorithm}[htbp]
	\caption{Extend function}
	\label{algo:extend}
	\begin{algorithmic}[1]
		\Function{extend}{$(\mathcal{T}, x, (\mathcal{U}_{C}, \mathcal{U}_{D}), \mathcal{H}, X_{u}, X_{*})$}
		\State $v_{cur}\leftarrow \texttt{nearest\_neighbor}(x, \mathcal{T}, \mathcal{H}, X_{*})$;
		\State $(\texttt{is\_a\_new\_vertex\_generated}, x_{new}, \psi_{new} ) \leftarrow \texttt{new\_state}(v_{cur}, (\mathcal{U}_{C},\newline \mathcal{U}_{D}), \mathcal{H}, X_{u})$
		\If {$\texttt{is\_a\_new\_vertex\_generated} = true$}
		\State $v_{new} \leftarrow \mathcal{T}.\texttt{add\_vertex}(x_{new})$;
		\State $\mathcal{T}.\texttt{add\_edge}(v_{cur}, v_{new}, \psi_{new})$;
		\State \Return $\texttt{Advanced}$;
		\EndIf
		\State \Return $\texttt{Trapped}$;
		\EndFunction
	\end{algorithmic}
\end{algorithm}

\textbf{Step} 1 in the overview provided in Section \ref{section:algorithmoverview} corresponds to the function call $\mathcal{T}.\texttt{init}$ in line 1 of Algorithm \ref{algo:hybridRRT}. \textbf{Step} 2 is implemented in line 3. \textbf{Step} 3 is implemented by the function call $\texttt{random\_state}$ in lines 5 and 8. \textbf{Step} 4 corresponds to the function call $\texttt{nearest\_neighbor}$ in line 1 of the function call  $\texttt{extend}$. \textbf{Step} 5 is implemented by the function calls $\texttt{new\_state}$, $\mathcal{T}.\texttt{add\_vertex}$, and $\mathcal{T}.\texttt{add\_edge}$ in lines 3, 5, and 6 of the function call $\texttt{extend}$.

Each function in Algorithm \ref{algo:hybridRRT} is defined next.
\subsubsection{$\mathcal{T}.\texttt{init}(X_{0})$:}\label{section:init} 
	The function call $\mathcal{T}.\texttt{init}$ is used to initialize a search tree $\mathcal{T} = (V, E)$.  It randomly selects a finite number of points from $X_{0}$. For each sampling point $x_{0}\in X_{0}$, a vertex $v_{0}$ associated with $x_{0}$ is added to $V$. At this step, no edge is added to $E$.
	\subsubsection{$x_{rand}$$\leftarrow$$\texttt{random\_state}(S)$:} 
	The function call $\texttt{random\_state}$ randomly selects a point from the set $S\subset \mathbb{R}^{n}$. Rather than to select from $\overline{C'}\cup D'$, it is designed to select points from $\overline{C'}$ and $D'$ separately depending on the value of $r$ . The reason is that if $\overline{C'}$ (or $D'$) has zero measure while $D'$ (respectively, $\overline{C'}$) does not, the probability that the point selected from $\overline{C'}\cup D'$ lies in $\overline{C'}$ (respectively, $D'$) is zero, which would prevent establishing probabilistic completeness.
	
	\subsubsection{$v_{cur}$$\leftarrow$$\texttt{nearest\_neighbor}$$(x_{rand}, $$\mathcal{T}, $$\mathcal{H}, X_{*})$:} \label{section:nearestneighbor}
	The function call $\texttt{nearest\_neighbor}$ searches for a vertex $v_{cur}$ in the search tree $\mathcal{T} = (V, E)$ such that its associated state value has minimal distance to $x_{rand}$. This function is implemented as solving the following optimization problem over $X_{\star}$, where $\star$ is either $c$ or $d$.
	\begin{problem}
		\label{problem:nearestneighbor}
		Given a hybrid system $\mathcal{H} = (C, f, D, g)$, $x_{rand}\in \myreals[n]$, and a search tree $\mathcal{T} = (V, E)$, solve
		$$
		\begin{aligned}
		\argmin_{v\in V}& \quad |\overline{x}_{v} -  x_{rand}|\\
		\textrm{s.t.}& \quad\overline{x}_{v} \in X_{\star}.
		\end{aligned}
		$$
	\end{problem}
The data of Problem \ref{problem:nearestneighbor} comes from the arguments of the $\texttt{nearest\_neighbor}$ function call. This optimization problem can be solved by traversing all the vertices in $\mathcal{T} = (V, E)$.
	\subsubsection{$(\texttt{is\_a\_new\_vertex\_generated}, x_{new}, \psi_{new} ) \leftarrow \texttt{new\_state}(v_{cur}, (\mathcal{U}_{C}, \mathcal{U}_{D}) , \mathcal{H} = (C, f, D, g) , X_{u})$:} 	\label{section:newstate}
	Given $v_{cur}$, if $\overline{x}_{v_{cur}} \in \overline{C'}\backslash D'$, the function call $\texttt{new\_state}$ randomly selects an input signal $\tilde{\myinputt}$ from $\mathcal{U}_{C}$ such that $(\overline{x}_{v_{cur}}, \tilde{\myinputt}(0))\in \overline{C}$ and generates a new maximal solution pair, denoted $\psi_{new} = (\mystatet_{new}, \myinputt_{new})$, by
	\begin{equation}
	\label{newstate:flow}
		\psi_{new} \leftarrow \texttt{continuous\_simulator}(C, f, D, 2,  \overline{x}_{v_{cur}}, \tilde{\myinputt})
	\end{equation}
	where $\texttt{continuous\_simulator}$ is formulated as in (\ref{simulator:flow}).
	
	If $\overline{x}_{v_{cur}} \in D'\backslash \overline{C'}$, the function call $\texttt{new\_state}$ randomly selects an input signal $u_{D}$ from $\mathcal{U}_{D}$ such that $(\overline{x}_{v_{cur}}, u_{D})\in D$ and generates a new solution pair, denoted $\psi_{new} = (\mystatet_{new}, \myinputt_{new})$, by
	\begin{equation}\label{newstate:jump}
		\psi_{new} \leftarrow \texttt{discrete\_simulator}(D, g, \overline{x}_{v_{cur}}, u_{D}).
	\end{equation}
	where $\texttt{discrete\_simulator}$ is formulated as in (\ref{simulator:jump}).
	
	If $\overline{x}_{v_{cur}}$$\in\overline{C'}\cap D'$, then this function generates $\psi_{new}$ by randomly selecting flows or jump. This random selection is implemented by randomly selecting a real number $r_{D}$ from the interval $[0, 1]$ and comparing $r_{D}$ with a user-defined parameter $p_{D}\in (0, 1)$. If $r_{D} \leq p_{D}$, then the function call $\texttt{new\_state}$ generates $\psi_{new}$ by flow, otherwise, by jump. The final state of $\psi_{new}$ is denoted as $x_{new}$. 

After $\psi_{new}$ and $x_{new}$ are generated, the function $\texttt{new\_state}$ checks if $\psi_{new}$ is trivial. If so, then $\psi_{new}$ does not explore any unexplored space and is not necessary to be added into $\mathcal{T}$. Hence, $$\texttt{is\_a\_new\_vertex\_generated}\leftarrow \texttt{false}$$ and the function call $\texttt{new\_state}$ is returned. Else, the function $\texttt{new\_state}$ checks if there exists $(t, j)\in \dom \psi_{new}$ such that $\psi_{new}(t, j)\in X_{u}$. If so, then $\psi_{new}$ intersects with the unsafe set and $$\texttt{is\_a\_new\_vertex\_generated}\leftarrow \texttt{false}.$$ Otherwise, $$\texttt{is\_a\_new\_vertex\_generated}\leftarrow \texttt{true}.$$

	\subsubsection{$v_{new}\leftarrow\mathcal{T}.\texttt{add\_vertex}(x_{new})$ and  $\mathcal{T}.\texttt{add\_edge}\newline(v_{cur}, v_{new}, \psi_{new})$:} 
	The function call $\mathcal{T}.\texttt{add\_vertex}\newline(x_{new})$ adds a new vertex $v_{new}$ associated with $x_{new}$ to $\mathcal{T}$ and returns $v_{new}$. The function call $\mathcal{T}.\texttt{add\_edge}(v_{cur},\newline  v_{new}, \psi_{new})$ adds a new edge $e_{new} = (v_{cur}, v_{new})$ associated with $\psi_{new}$  to $\mathcal{T}$. 
	\subsubsection{Solution Checking during HyRRT Construction}
	\label{section:checksolution}
	When the function call $\texttt{extend}$ returns $\texttt{Advanced}$,  HyRRT checks if a path in $\mathcal{T}$ can be used to construct a motion plan solving the given motion planning problem. If this function finds a path
	$
	p = ((v_{0}, v_{1}), (v_{1}, v_{2}), ..., (v_{n - 1}, v_{n})) = : (e_{0}, e_{1}, ..., e_{n - 1})
	$
	in $\mathcal{T}$ such that 
	\begin{enumerate}[label=\arabic*)]
		\item $\overline{x}_{v_{0}} \in X_{0}$,
		\item $\overline{x}_{v_{n}} \in X_{f}$,
		\item for $i \in \{0, 1, ..., n - 2\}$, if $\overline{\psi}_{e_{i}}$ and $\overline{\psi}_{e_{i + 1}}$ are both purely continuous, then $\overline{\psi}_{e_{i + 1}}(0, 0)\in C$,
\end{enumerate}
then the solution pair $\tilde{\psi}_{p}$ associated with path $p$, defined in (\ref{equation:concatenationpath}), is a motion plan to the given motion planning problem. More specifically, the items listed above guarantee that $\psi$ satisfies all the conditions in Problem \ref{problem:motionplanning}. Items 1 and 2 guarantee that $\tilde{\psi}_{p}$ starts from $X_{0}$ and ends within $X_{f}$.  Then, we show that each condition in Proposition \ref{theorem:concatenationsolution} is satisfied such that Proposition \ref{theorem:concatenationsolution} guarantees that $\psi$ is a solution pair to $\hs$ as follows:
	\begin{enumerate}
		\item Because the input signal in $\mathcal{U}_{C}$ is compact, the solution pairs generated by $\texttt{new\_state}$ by flow are compact. The solution pairs generated by $\texttt{new\_state}$ by jump are compact for free. Therefore, the first condition in Proposition \ref{theorem:concatenationsolution} is satisfied. 
		\item Note that for $i \in\{0, 1, 2, ..., n - 2\}$, the final state of $\overline{\psi}_{e_{i}} = \overline{\psi}_{(v_{i}, v_{i + 1})}$ equals the initial state of $\overline{\psi}_{e_{i + 1}} = \overline{\psi}_{(v_{i + 1}, v_{i + 2})}$ because both of them are $\overline{x}_{v_{i + 1}}$. Then, the second condition in Proposition \ref{theorem:concatenationsolution} is satisfied.
		\item Then, item 3 above guarantees that the third condition in Proposition \ref{theorem:concatenationsolution} is satisfied. 
	\end{enumerate} Therefore, Proposition \ref{theorem:concatenationsolution} guarantees that $\psi_{p}$ is a solution pair to $\mathcal{H}$. Note that for each $e$ in $p$, $\overline{\psi}_{e}$ does not intersect with unsafe set because the solution pairs that intersect the unsafe set have been excluded by the function $\texttt{new\_state}$. Therefore, $\psi_{p}$ satisfies all the conditions in Problem \ref{problem:motionplanning}.
	\begin{remark}
		Note that the choices of inputs in the fucntion call $\texttt{new\_state}$ are random. Some RRT variants choose the optimal input that drives $x_{new}$ closest to $x_{rand}$. However, \cite{kunz2015kinodynamic} proves that such a choice makes the RRT algorithm probabilistically incomplete.
		\EndRemark
	\end{remark}
\begin{remark}
	In practice, item 2 above is too restrictive. Given $\epsilon > 0$ representing the tolerance associated with this condition, we implement item 2 as
	\begin{equation}
	\label{equation:tolerance}
	|\overline{x}_{v_{n}}|_{X_{f}} \leq	\epsilon.
	\end{equation}\EndRemark
\end{remark}

\section{Probabilistic Completeness Analysis}\label{section:pc}
This section establishes probabilistic completeness of the HyRRT algorithm. Probabilistic completeness means that the probability that the planner fails to return a motion plan, if one exists, approaches zero as the number of samples approaches infinity. 
\subsection{Clearance of Motion Plan and Inflation of a         Hybrid System}\label{section:inflation}
The clearance of a motion plan captures the distance between the motion plan and the boundary of the constraint sets, which in Problem \ref{problem:motionplanning}, includes the initial state set $X_{0}$, the final state set $X_{f}$, the unsafe set $X_{u}$, the flow set $C$, and the jump set $D$. We propose two different clearances, \emph{safety clearance} and \emph{dynamics clearance}, that capture the distance to the constraint sets $(X_{0}, X_{f}, X_{u})$ and $(C,  D)$, respectively. 
\begin{definition}[(Safety clearance of a motion plan)]
	\label{definition:safetyclearance}
	Given a motion plan $\psi = (\phi, \myinputt)$ to the motion planning problem $\mathcal{P} = (X_{0}, X_{f}, X_{u}, (C, f, D, g))$, the safety clearance of $\psi = (\phi, \myinputt)$ is given by $\delta_{s} > 0$ if, for each $\delta' \in [0, \delta_{s}]$, the following conditions are satisfied:
	\begin{enumerate}[label=\arabic*)]
		\item  $\phi(0, 0) + \delta'\mathbb{B}\subset X_{0}$;
		\item  $\phi(T, J) + \delta'\mathbb{B}\subset X_{f}$, where $(T, J) = \max \dom \psi$;
		\item  For all $(t, j)\in \dom \psi$, $(\phi(t, j) + \delta'\mathbb{B}, \myinputt(t, j) + \delta'\mathbb{B}) \cap X_{u} =\emptyset$.
	\end{enumerate} 
\end{definition}
\begin{definition}[(Dynamics clearance of a motion plan)]
	\label{definition:dynamicsclearance}
	Given a motion plan $\psi = (\phi, \myinputt)$ to the motion planning problem $\mathcal{P} = (X_{0}, X_{f}, X_{u}, (C, f, D, g))$, the dynamics clearance of $\psi = (\phi, \myinputt)$ is given by $\delta_{d}> 0$ if, for each $\delta' \in [0, \delta_{d}]$, the following conditions are satisfied:
	\begin{enumerate}[label=\arabic*)]
		\item For all $(t, j)\in \dom \psi $ such that $I^{j}$ has nonempty interior, $(\phi(t, j) + \delta'\mathbb{B}, \myinputt(t, j) + \delta'\mathbb{B}) \subset C$;
		\item For all $(t, j)\in \dom \psi $ such that $(t, j + 1)\in \dom \psi$, $(\phi(t, j) + \delta'\mathbb{B}, \myinputt(t, j) + \delta'\mathbb{B})\subset D$.
	\end{enumerate} 
\end{definition}
\begin{remark}
	The definition of the two types of clearance is analogous to the clearance in \cite{kleinbort2018probabilistic}, albeit with a specific consideration for the boundaries of different constraint sets. Safety clearance in Definition \ref{definition:safetyclearance} is defined as the minimal positive distance between the motion plan and the nearest boundaries of the initial state set, final state set, and the unsafe set. Dynamics clearance in Definition \ref{definition:dynamicsclearance}, on the other hand, represents the minimal positive distance between the motion plan and the closest boundaries of the flow set and the jump set, depending on whether the motion plan is during a flow or at a jump. 
\end{remark}
With both safety clearance and dynamics clearance defined, we are ready to define the clearance of a motion plan.
\begin{definition}[(Clearance of a motion plan)]
	\label{definition:clearance}
	Given a motion plan $\psi = (\mystatet, \myinputt)$ to the motion planning problem $\mathcal{P} = (X_{0}, X_{f} , X_{u}, (C, f, D, g))$, the clearance of $\psi$, denoted $\delta$, is defined as the minimum between its safety clearance $\delta_{s}$ and dynamics clearance $\delta_{d}$, i.e., $\delta := \min \{\delta_{s}, \delta_{d}\}$.
\end{definition}

In the probabilistic completeness result in \cite[Theorem 2]{kleinbort2018probabilistic}, a motion plan with positive clearance is assumed to exist. However, the assumption that there exists a positive dynamics clearance is restrictive for hybrid systems. Indeed, if the motion plan reaches the boundary of the flow set or of the jump set, then the motion plan has no (dynamics) clearance; see Definition \ref{definition:dynamicsclearance}. Figure \ref{fig:motionplannoclearance} shows a motion plan to the sample motion planning problem for the actuated bouncing ball system in Example ~\ref{example:bouncingball} without clearance.

To overcome this issue and to assure that HyRRT is probabilistically complete, we inflate the hybrid system $\mathcal{H} = (C, f, D, g)$ as follows. 
\begin{definition}[($\delta$-inflation of a hybrid system)]
	\label{definition:inflation}
	Given a hybrid system $\mathcal{H} = (C, f, D, g)$ and $\delta > 0$, the $\delta$-inflation of the hybrid system $\mathcal{H}$, denoted $\mathcal{H}_{\delta}$ with data $(C_{\delta}, f_{\delta}, D_{\delta}, g_{\delta})$, is given by\endnote{The flow set $C$ (and the jump set $D$) inflates the state $x$ and input $u$, respectively, yielding their independent ranges. This ensures that the sampling process is a Bernoulli process.}
	\begin{equation}
		\mathcal{H}_{\delta}: \left\{              
		\begin{aligned}               
		\dot{x} & = f_{\delta}(x, u)     &(x, u)\in C_{\delta}\\                
		x^{+} & =  g_{\delta}(x, u)      &(x, u)\in D_{\delta}\\                
		\end{aligned}   \right. 
		\label{model:inflatedhybridsystem}
		\end{equation}
	 where
	\begin{enumerate}[label=\arabic*)]
		\item The $\delta$-inflation of the flow set is constructed as
		\begin{equation}\label{equation:inflatedflowset}
		\begin{aligned}
		C_{\delta} := \{(x, u)\in \mathbb{R}^{n}\times \mathbb{R}^{m}: \exists (y,  v)\in C:  \\
		x\in y + \delta \mathbb{B}, 
		u\in v+ \delta \mathbb{B}\},
		\end{aligned}
		\end{equation}
			\item The $\delta$-inflation of the flow map is constructed as
			\begin{equation}\label{equation:inflatedflowmap}
						f_{\delta}(x, u) := f(x, u)\quad \forall (x, u)\in C_{\delta},
			\end{equation} 
		\item The $\delta$-inflation of the jump set is constructed as\begin{equation}\label{equation:inflatedjumpset}
		\begin{aligned}
		D_{\delta} := \{(x, u)\in \mathbb{R}^{n}\times \mathbb{R}^{m}: \exists (y,  v)\in D: \\
		x\in y + \delta \mathbb{B}, u\in v+ \delta \mathbb{B}\},
		\end{aligned}
		\end{equation}
		\item The $\delta$-inflation of the jump map is constructed as \begin{equation}\label{equation:inflatedjumpmap}
			\begin{aligned}
			g_{\delta}(x, u) &:= g(x, u)\quad \forall (x, u)\in D_{\delta}.\\
			\end{aligned}
			\end{equation}
	\end{enumerate}
\end{definition}
The above outlines a general method of constructing the $\delta$-inflation of the given hybrid system. Next, this method is exemplified in the actuated bouncing ball system.
\begin{example}[(Actuated bouncing ball system in Example \ref{example:bouncingball}, revisited)]\label{example:bouncingballinflate}
	Given $\delta > 0$, the $\delta$-inflation of hybrid system of the actuated bouncing ball system is constructed as follows.
	\begin{itemize}
		\item From (\ref{equation:inflatedflowset}), the $\delta$-inflation of flow set $C$, denoted $C_{\delta}$, is given by
		\begin{equation}
		\label{equation:bouncingballbwflowset}
		C_{\delta}= \{(x, u)\in \mathbb{R}^{2}\times \mathbb{R}: x_{1} \geq -\delta\}.
		\end{equation}
		\item From (\ref{equation:inflatedflowmap}), the $\delta$-inflation of flow map $f$, denoted $f_{\delta}$, is given by
					\begin{equation}
					\label{equation:bouncingballbwflowmap}
					f_{\delta}(x, u) =  f(x, u) = 
					\left[ \begin{matrix}
					x_{2} \\
					-\gamma
					\end{matrix}\right] \qquad \forall(x, u)\in C_{\delta}
					\end{equation}
		\item From (\ref{equation:inflatedjumpset}), the $\delta$-inflation of jump set $D$, denoted $D_{\delta}$, is given by
		\begin{equation}
		\label{equation:bouncingballbwjumpset}
		\begin{aligned}
		D_{\delta} := (\{(0, 0)\} + \delta\mathbb{B})
		\cup \{(x, u)\in \mathbb{R}^{2}\times \mathbb{R}: \\
		x_{2} \leq 0 \text{ and } -\delta\leq x_{1} \leq \delta\}. 
		\end{aligned}
		\end{equation}
				\item From (\ref{equation:inflatedjumpmap}), the $\delta$-inflation of jump map $g$, denoted $g_{\delta}$, is given by 
			\begin{equation}
			\label{equation:bouncingballbwjumpmap}
			g_{\delta}(x, u) :=  g(x, u) = 
			\left[ \begin{matrix}
			x_{1} \\
			-\lambda x_{2} + u
			\end{matrix}\right] 
			\  \forall(x, u)\in D_{\delta}.
			\end{equation}
	\end{itemize}
\end{example}
As is shown in Figure \ref{fig:motionplannoclearance_inflated}, the inflation of a hybrid system defined above serves to extend the boundary of both the flow and jump sets in $\mathcal{H}$ by $\delta$ and to capture the same continuous and discrete dynamics in those extended sets as in $\hs$. Consequently, under the assumption that a motion plan to $\mathcal{P}$ with positive safety clearance exists, this methodology facilitates the establishment of such a motion plan with positive dynamics clearance and, in turn, with positive clearance.
\begin{figure}[htbp]
	\centering
	\subfigure[\label{fig:motionplannoclearance}]{\includegraphics[width = 0.45\textwidth]{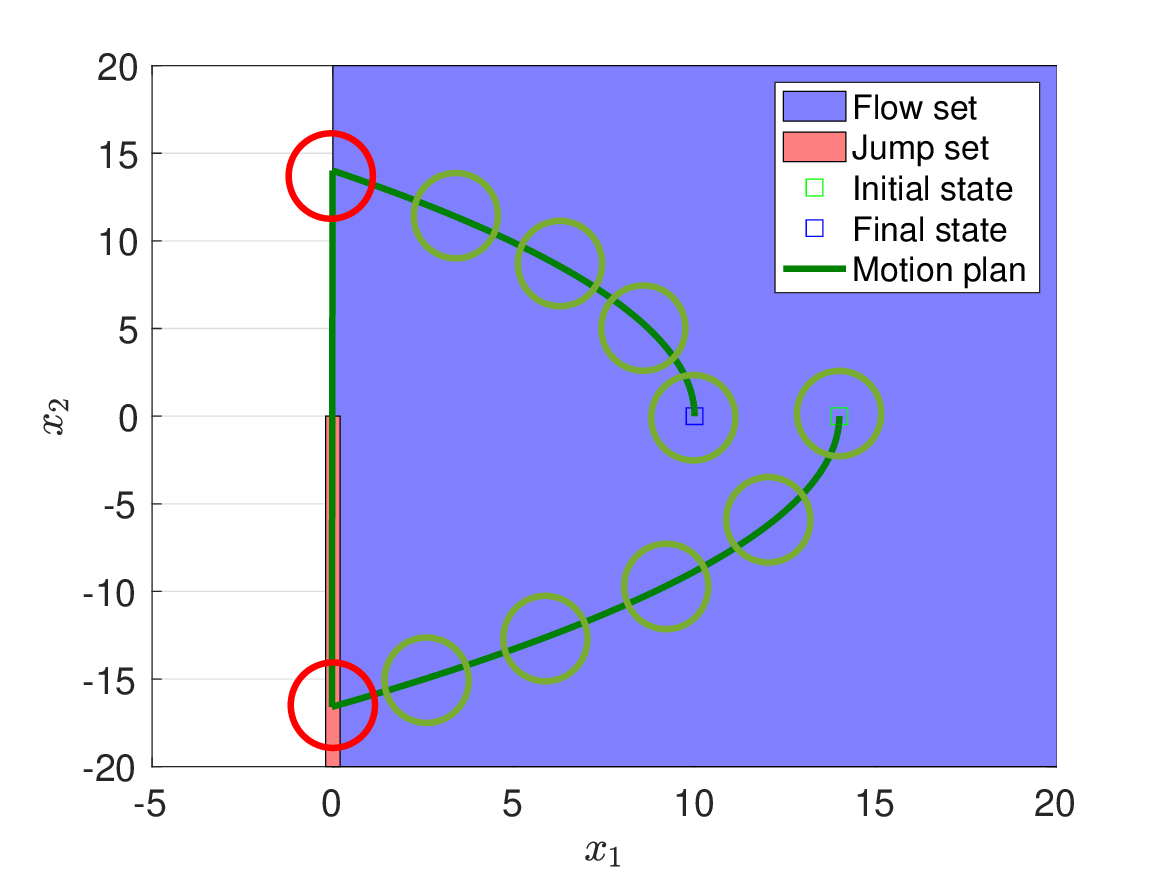}}
	\subfigure[\label{fig:motionplannoclearance_inflated}]{\includegraphics[width = 0.45\textwidth]{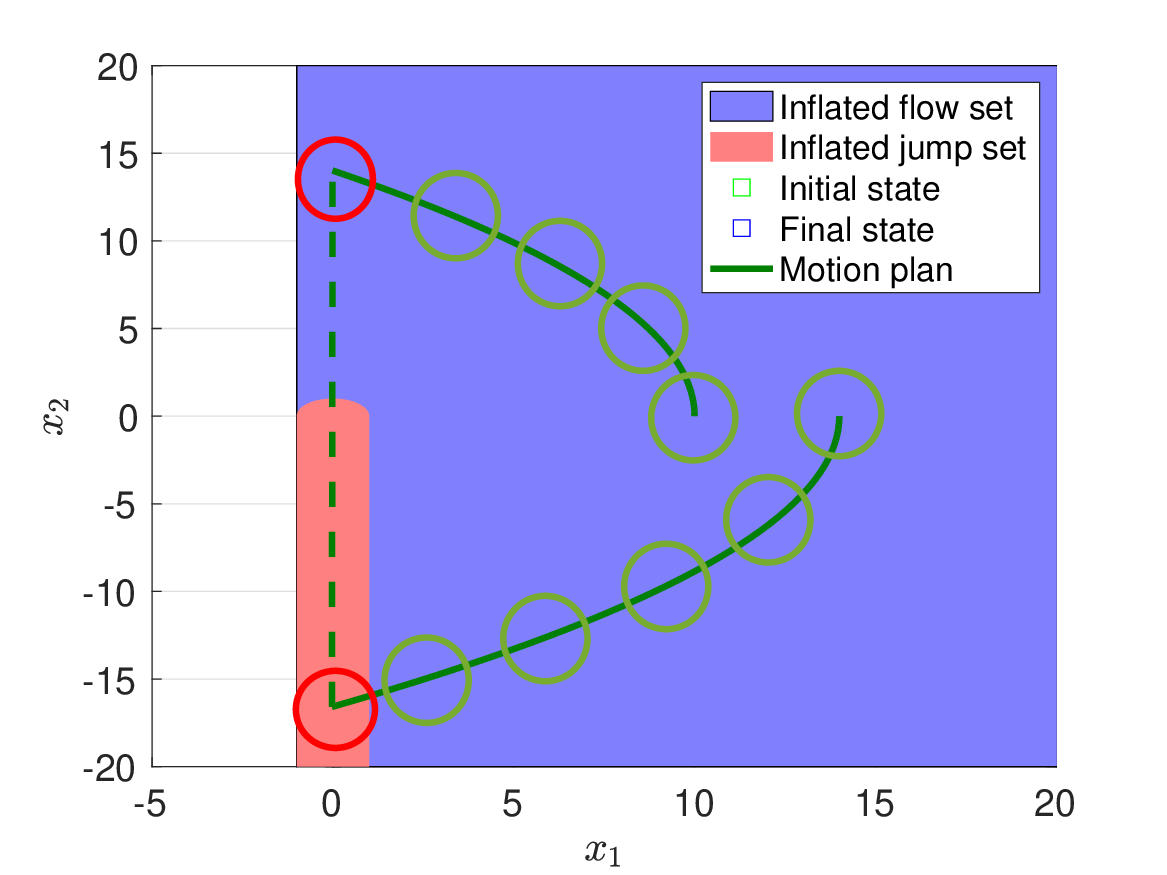}}
	\caption{Figure \ref{fig:motionplannoclearance} shows a sample motion plan for the bouncing ball system in Example \ref{example:bouncingball} without clearance. The dark green trajectory shows the motion plan. The blue region denotes the projection of flow set on the state space. The circles denote the boundaries of the balls along the motion plan at specific hybrid time instances. Note that the red balls depicted in Figure \ref{fig:motionplannoclearance}, which encircle the state of the motion plan at the boundary of the flow set, are not subsets of the flow set. In this case, the clearance $\delta$ is zero. However, for the inflated system in (\ref{model:inflatedhybridsystem}), the existence of the circles along the motion plan with radius $\delta$, as shown in Figure \ref{fig:motionplannoclearance_inflated}, implies the positive clearance $\delta$ of the motion plan.}
\end{figure}
Next, we show that a motion plan to the original motion planning problem is also a motion plan to the motion planning problem for its $\delta$-inflation.
\begin{proposition}
	\label{proposition:motioninflated}
	Given a motion planning problem $\mathcal{P} = (X_{0}, X_{f}, X_{u}, (C, f, D, g))$ in Problem \ref{problem:motionplanning} with positive safety clearance, if $\psi$ is a motion plan to $\mathcal{P}$, then for each $\delta > 0$, $\psi$ is also a motion plan to the motion planning problem $\mathcal{P}_{\delta} = (X_{0}, X_{f}, X_{u},  (C_{\delta}, f_{\delta}, D_{\delta}, g_{\delta}))$, where $(C_{\delta}, f_{\delta}, D_{\delta}, g_{\delta})$ is the $\delta$-inflation of the hybrid system defined by $(C, f, D, g)$.
\end{proposition}
\begin{proof}
	See Appendix \ref{section:proofinflated}.
\end{proof}
Next we show that the existing motion plan with positive safety clearance has positive clearance for the motion planning problem for the $\delta$-inflation of the original hybrid system $\hs$.
\begin{lemma}
	\label{lemma:motionplanintlated}
	Let $\psi$ be a motion plan to the motion planning problem $\mathcal{P} = (X_{0}, X_{f}, X_{u}, (C, f, D, g))$ formulated as Problem \ref{problem:motionplanning} with positive safety clearance $\delta_{s} > 0$. Then, for each $\delta_{f} > 0$, $\psi$ is a motion plan to the motion planning problem $\mathcal{P}_{\delta_{f}} = (X_{0}, X_{f}, X_{u}, (C_{\delta_{f}}, f_{\delta_{f}}, D_{\delta_{f}}, g_{\delta_{f}}))$ with clearance $\delta = \min \{\delta_{s}, \delta_{f}\}$, where $\mathcal{H}_{\delta_{f}} = (C_{\delta_{f}}, f_{\delta_{f}}, D_{\delta_{f}}, g_{\delta_{f}})$ is the $\delta_{f}$-inflation of $\mathcal{H} = (C, f, D, g)$.
\end{lemma}
\begin{proof}
	See Appendix \ref{section:proofinflatedlemma}.
\end{proof}
\subsection{Assumptions}\label{section:assumption}
Similar to \cite[Definition 2]{kleinbort2018probabilistic}, we consider input functions that are piecewise constant in the following sense. 
\begin{definition}[(Piecewise-constant function)]
	\label{definition:piecewiseconstantfun}
	A function $\tilde{\myinputt}_{c}: [0, T]\to U_{C}$ is said to be a piecewise-constant function for probabilistic completeness if there exists $\Delta t \in \mathbb{R}_{>0}$, called resolution, such that 
	\begin{enumerate}[label=\arabic*)]
		\item $k:=\frac{T}{\Delta t}\in \mathbb{N}$;
		\item for each $i\in \{1, 2,..., k\}$, $t \mapsto \tilde{\myinputt}_{c}(t)$ is constant over $[(i - 1)\Delta t, i\Delta t)$.
	\end{enumerate}
\end{definition}
From Definition \ref{definition:piecewiseconstantfun}, we define a motion plan notion with piecewise-constant inputs as follows.
\begin{definition}[(Motion plan with piecewise-constant input)]\label{definition:mppiecewiseconstant}
	Given a motion planning problem $\mathcal{P}$, a motion plan $(\mystatet, \myinputt)$ to $\mathcal{P}$ is said to be a \emph{motion plan with piecewise-constant input with resolution $\Delta t$} if, for all $j\in \mathbb{N}$ such that $I^{j}_{\myinputt}$ has nonempty interior, $t\mapsto \myinputt(t, j)$ is a piecewise constant function with resolution $\Delta t$.
\end{definition}
Similar to \cite{li2016asymptotically}, this paper considers that HyRRT is executed uniformly in relation to the random selection involved in \textbf{Step} 1, \textbf{Step} 3, and \textbf{Step} 5 of HyRRT, in the following sense.
\begin{definition}[(Uniform HyRRT)]
	\label{assumption:uniformsample}
	HyRRT is said to be executed uniformly if the probability distribution of the random selection in the function calls $\mathcal{T}.\texttt{init}$, $\texttt{random\_state}$, and  $\texttt{new\_state}$ is the uniform distribution.
\end{definition}
\begin{remark}
	By Definition \ref{assumption:uniformsample}, the computation of the probability of randomly selecting a point that lies in a given set is simplified. When randomly selecting a point $s$ from a set $S$, according to \cite[Page 257 (20.9)]{billingsley2017probability}, the probability that $s$ belongs to a subset $R\subset S$ is 
	\begin{equation}
	\label{equation:probabilitylebesgue}
	\mbox{\rm Prob}[s\in R] = \int_{s\in R} \frac{1}{\mu(S)}\text{d}s = \frac{\mu(R)}{\mu(S)}
	\end{equation}
	where $\mu$ denotes the Lebesgue measure of a set.
\end{remark}
Subsequently, we define a complete input library that encompasses each possible input value that can be utilized either during the flow or at a jump.
\begin{definition}[(Complete input library)]
	\label{assumption:inputlibrary}
	The input library $(\mathcal{U}_{C}, \mathcal{U}_{D})$ is said to be a complete input library if
	\begin{enumerate}[label=\arabic*)]
		\item $\mathcal{U}_{C}$ is the set of constant input signals and includes all possible input signals such that, for some $T_{m} > 0$, their time domains are closed subintervals of the interval $[0, T_{m}]$ including zero and their images belong to $U_{C}$. In other words, there exists $T_m > 0$ such that $\mathcal{U}_{C} = \{ \tilde{\myinputt} : \dom \tilde{\myinputt} = [0,T] \subset [0,T_m], \tilde{\myinputt} \text{ is } \text{constant}, \tilde{\myinputt} \in U_C\}$;
		\item $\mathcal{U}_D = U_D$.
	\end{enumerate}
\end{definition}
\begin{remark}
	From Definition \ref{assumption:inputlibrary}, the input signals in $\mathcal{U}_{C}$ are all constant functions. This property of $\mathcal{U}_{C}$ allows for the inputs of a motion plan with piecewise-constant input to be constructed by concatenating constant input signals in $\mathcal{U}_{C}$.  The set  $\mathcal{U}_{C}$ collects each possible constant input signal taking values from $U_{C}$ and with maximal duration $[0, T_{m}]$. The upper bound $T_m$ on their duration ensures a positive lower bound on the probability of sampling an appropriate input duration by the function call $\texttt{new\_state}$ in Section~\ref{section:newstate}, where an input signal $\tilde{\myinputt}$ is randomly selected from $\mathcal{U}_{C}$.
	Without this upper bound $T_{m}$, according to (\ref{equation:probabilitylebesgue}), the probability of selecting any finite subintervals from $[0, \infty)$ is $0$ because the Lebesgue measure of $[0, \infty)$ is infinity. The selection process involves a random choice of $t_{m}$ from the interval $ [0, T_{m}]$ and $u_{C}$ from $U_{C}$, from where a constant input signal is constructed as $\tilde{\myinputt}:[0, t_{m}]\to \{u_{C}\}$.\EndRemark
\end{remark}
To ensure that each random process in HyRRT, specially within the function calls  $\mathcal{T}.\texttt{init}$, $\texttt{random\_state}$, and  $\texttt{new\_state}$, returns a suitable sample with positive probability, it is essential to assume that the Lebesgue measure of the sets being sampled is both nonzero and finite.
\begin{assumption}
	\label{assumption:finitesets}
	The sets $X_{0}$, $C'$, $D'$, $U_{C}$, $U_{D}$, which HyRRT makes random selection from, have finite and positive Lebesgue measure.
\end{assumption}
\begin{remark}
	 Under Assumption \ref{assumption:finitesets}, it is guaranteed that $\mu(S) < \infty$. Then, if HyRRT is executed uniformly and $\mu(R)$ is positive, from (\ref{equation:probabilitylebesgue}), it follows that $\mbox{\rm Prob}[s\in R] \in (0, 1]$.\EndRemark
\end{remark}
The following assumption is imposed on the flow map $f$ of the hybrid system $\mathcal{H}$ in (\ref{model:generalhybridsystem}).
\begin{assumption}
	\label{assumption:flowlipschitz}
	The flow map $f$ is Lipschitz continuous. In particular, there exist $K^{f}_{x}, K^{f}_{u}\in \mathbb{R}_{>0}$ such that, for each $(x_{0}, x_{1}, u_{0}, u_{1})$ such that $(x_{0}, u_{0}) \in C$, $(x_{0}, u_{1}) \in C$, and $(x_{1}, u_{0}) \in C$,
	$$
	\begin{aligned}
	|f(x_{0}, u_{0}) - f(x_{1}, u_{0})|&\leq K^{f}_{x}|x_{0} - x_{1}|\\
	|f(x_{0}, u_{0}) - f(x_{0}, u_{1})|&\leq K^{f}_{u}|u_{0} - u_{1}|.
	\end{aligned}
	$$
\end{assumption}
\begin{remark}
	Assumption \ref{assumption:flowlipschitz} adheres to the Lipschitz continuity assumption on differential constraints, as outlined in \cite{kleinbort2018probabilistic}. Assumption \ref{assumption:flowlipschitz} guarantees that the flow map is Lipschitz continuous for both state and input arguments. This assumption establishes an explicit upper bound, parameterized by both state and input, on the distance between the motion plan and the simulated purely continuous solution pair in the forthcoming Lemma~\ref{lemma:pccontinuouslowerbound}. Following the methodology in the proof of \cite[Lemma 3]{kleinbort2018probabilistic}, by ensuring that this upper bound is less than the motion plan's clearance, a range of the input signals in $\mathcal{U}_{C}$ with positive Lebesgue measure is determined that are capable of simulating a purely continuous solution pair that stays within the motion plan's clearance. As per~(\ref{equation:probabilitylebesgue}), there is a guaranteed positive probability of randomly sampling an input signal from $\mathcal{U}_{C}$ that falls within this specified range. This positive probability is instrumental in ensuring that in each Bernoulli trials in the proof of the upcoming Proposition \ref{theorem:pc}, there is a positive probability of achieving a successful outcome, which eventually leads to the probabilistic completeness guarantee.\EndRemark
\end{remark}
\begin{example}[(Actuated bouncing ball system in Example \ref{example:bouncingball}, revisited)]
	 In the bouncing ball system, the flow map is defined as $ f(x, u) = \left[ \begin{matrix}
	x_{2} \\
	-\gamma
	\end{matrix}\right]$. The flow set is defined as $C := \{(x, u)\in \mathbb{R}^{2}\times \mathbb{R}: x_{1}\geq 0\}.$ 
	For each $(x_{0}, x_{1}, u_{0})$ such that $(x_{0}, u_{0}) \in C$ and $(x_{1}, u_{0}) \in C$, we have 
	$$
	\begin{aligned}
	|f(x_{0}, u_{0}) - f(x_{1}, u_{0})| &=  \left|\left[ \begin{matrix}
	x_{0, 2} - x_{1, 2} \\
	0
	\end{matrix}\right]\right| = |x_{0, 2} - x_{1, 2}| \\
	&\leq |x_{0} - x_{1}|
	\end{aligned}
	$$ where $x_{0, 2}$ and $x_{1, 2} $ denote the second component of $x_{0}$ and $x_{1}$, respectively. Therefore, $K_{x}^{f} = 1 > 0$ is such that $|f(x_{0}, u_{0}) - f(x_{1}, u_{0})|\leq K^{f}_{x}|x_{0} - x_{1}|$ for each $(x_{0}, x_{1}, u_{0})$ such that $(x_{0}, u_{0}) \in C$ and $(x_{1}, u_{0}) \in C$. Now, for each  $(x_{0}, u_{0}, u_{1})$ such that $(x_{0}, u_{0}) \in C$ and $(x_{0}, u_{1}) \in C$, we have 
	$
		|f(x_{0}, u_{0}) - f(x_{0}, u_{1})| = 0.
	$
	Therefore,  any $K_{u}^{f} > 0$ is such that $|f(x_{0}, u_{0}) - f(x_{0}, u_{1})|\leq K^{f}_{u}|u_{0} - u_{1}|$ for each $(x_{0}, u_{0}, u_{1})$ such that $(x_{0}, u_{0}) \in C$ and $(x_{0}, u_{1}) \in C$. Therefore, the flow map and flow set in Example \ref{example:bouncingball} satisfy Assumption ~\ref{assumption:flowlipschitz}.
\end{example}
The following assumption is imposed on the jump map $g$ of the hybrid system $\mathcal{H}$ in (\ref{model:generalhybridsystem}).
\begin{assumption}
	\label{assumption:pcjumpmap}
	The jump map $g$ is such that there exist $K^{g}_{x}, K^{g}_{u}\in \mathbb{R}_{>0}$ such that, for each $(x_{0}, u_{0}) \in D$ and each $(x_{1}, u_{1}) \in D$,
	$$
	|g(x_{0}, u_{0}) - g(x_{1}, u_{1})|\leq K^{g}_{x}|x_{0} - x_{1}| + K^{g}_{u}|u_{0} - u_{1}|.
	$$
\end{assumption}
\begin{remark}
	Assumption \ref{assumption:pcjumpmap} enables to establish an explicit upper bound for the distance between the motion plan and the simulated purely discrete solution pair, as will be further detailed in the Lemma \ref{lemma:pcdiscretelowerbound}, which is forthcoming. Adopting the approach from the proof of \cite[Lemma 3]{kleinbort2018probabilistic}, by keeping this upper bound beneath the motion plan's clearance, a specific range of the input values in $\mathcal{U}_{D}$, possessing a positive Lebesgue measure, is identified that are capable of simulating a purely discrete solution pair that stays within the motion plan's clearance. As per  (\ref{equation:probabilitylebesgue}), there is a guaranteed positive probability of randomly sampling an input value from $\mathcal{U}_{D}$ that falls within this specified range. This positive probability ensures that, in the Bernoulli trials in the proof of the upcoming Proposition \ref{theorem:pc}, aside from those already covered by Assumption \ref{assumption:flowlipschitz}, each trial has a positive probability of yielding a successful outcome, eventually ensuring the probabilistic completeness of HyRRT.
	\EndRemark
\end{remark}
\begin{example}[(Actuated bouncing ball example in Example \ref{example:bouncingball}, revisited)] 
	In Example \ref{example:bouncingball}, the jump map is defined as $g(x, u) = \left[ \begin{matrix}
		x_{1} \\
		-\lambda x_{2}+u
	\end{matrix}\right]$. The jump set is defined as $D:= \{(x, u)\in \mathbb{R}^{2}\times \mathbb{R}: x_{1} = 0, x_{2} \leq 0, u\geq 0\}$. For each pair of $(x_{0}, u_{0}) \in D$ and $(x_{1}, u_{1}) \in D$, we have 
	$$
	\begin{aligned}
	|g(x_{0}, u_{0}) - &g(x_{1}, u_{1})| = \left|\left[ \begin{matrix}
	x_{0, 1} - x_{1, 1} \\
	-\lambda (x_{0, 2} - x_{1, 2}) + u_{0} - u_{1}
	\end{matrix}\right]\right| \\
	& = \left|\left[ \begin{matrix}
	x_{0, 1} - x_{1, 1} \\
	-\lambda (x_{0, 2} - x_{1, 2})
	\end{matrix}\right] + \left[ \begin{matrix}
	0\\
	u_{0} - u_{1}
	\end{matrix}\right]\right|\\
	&\leq \left|\left[ \begin{matrix}
	x_{0, 1} - x_{1, 1} \\
	-\lambda (x_{0, 2} - x_{1, 2})
	\end{matrix}\right]\right| + \left|\left[ \begin{matrix}
	0\\
	u_{0} - u_{1}
	\end{matrix}\right]\right|\\
	&\leq \sqrt{1 + \lambda^2} |x_{0} - x_{1}| + |u_{0} - u_{1}|.
	\end{aligned}
	$$
	Therefore, $K_{x}^{g} = \sqrt{1 + \lambda^2} $ and $K_{u}^{g} = 1$ are such that $|g(x_{0}, u_{0}) - g(x_{1}, u_{1})|\leq K^{g}_{x}|x_{0} - x_{1}| + K^{g}_{u}|u_{0} - u_{1}|$ for each pair of $(x_{0}, u_{0}) \in D$ and $(x_{1}, u_{1}) \in D$.
\end{example}
\subsection{Probabilistic Completeness Guarantee}\label{section:pc_relax}
Our main result shows that HyRRT is probabilistically complete without assuming positive clearance, achieved through properly exploiting the inflation~$\hs_{\delta_{f}}$.
\begin{theorem}
	\label{theorem:inflatedpc}
	Given a motion planning problem $\mathcal{P} = (X_{0}, X_{f}, X_{u}, (C, f, D, g))$, suppose that Assumptions \ref{assumption:finitesets}, \ref{assumption:flowlipschitz}, and \ref{assumption:pcjumpmap} are satisfied, and that there exists a compact motion plan $(\phi, \myinputt)$ to $\mathcal{P}$ with \emph{safety} clearance $\delta_{s} > 0$ and piecewise-constant input. Then, using a complete input library and when executed uniformly (as defined in Definition \ref{assumption:uniformsample}) to solve the problem $\mathcal{P}_{\delta_{f}} = (X_{0}, X_{f}, X_{u},(C_{\delta_{f}}, f_{\delta_{f}}, D_{\delta_{f}}, g_{\delta_{f}}))$, where, for some $\delta_{f} > 0$, $(C_{\delta_{f}}, f_{\delta_{f}}, D_{\delta_{f}}, g_{\delta_{f}})$ denotes the $\delta_{f}$-inflation of $(C, f, D, g)$ in (\ref{model:inflatedhybridsystem}), the probability that HyRRT fails to find a motion plan $\psi' = (\phi', \myinputt')$ to $\mathcal{P}_{\delta_{f}}$ such that $\phi'$ is $(\tilde{\tau}, \tilde{\delta})$-close to $\phi$ after $k$ iterations is at most $a\exp(-bk)$, where $a, b\in \mathbb{R}_{> 0}$, $\tilde{\tau} = \max\{T+ J, T' + J'\}$, $(T, J) = \max \dom \psi$, $(T', J') = \max \dom \psi'$,  and $\tilde{\delta} = \min \{\delta_{s}, \delta_{f}\}$.
\end{theorem}
The proof of Theorem \ref{theorem:inflatedpc} is established as follows. By the positive safety clearance assumption and exploiting the inflation of $\hs$ in (\ref{model:inflatedhybridsystem}), we establish that $(\mystatet, \myinputt)$ is a motion plan to $\mathcal{P}_{\delta}$ with positive clearance (see Lemma \ref{lemma:motionplanintlated} in Section \ref{section:inflation}). Then, given that $(\mystatet, \myinputt)$ is a motion plan to $\mathcal{P}_{\delta}$ with positive clearance, 
we demonstrate that the probability of HyRRT failing to find a motion plan \emph{with positive clearance} is converging to zero as the number of iterations approaches infinity. To demonstrate this result, we first establish that the probabilities of the following probabilistic events are positive:
\begin{enumerate}[label=E\arabic*)]
	\item The function call $\texttt{nearest\_neighbor}$ in Section ~\ref{section:nearestneighbor} returns a current vertex in the search tree within the clearance of $(\mystatet, \myinputt)$ (see Lemma \ref{lemma:nearestvertex} in the forthcoming Section \ref{section:probabilisticguarantees});
	\item If E1 occurs, the function call $\texttt{new\_state}$ in Section \ref{section:newstate} adds a new vertex and a new edge within the clearance of $(\mystatet, \myinputt)$ to the search tree (see Lemma \ref{lemma:pccontinuouslowerbound} and Lemma \ref{lemma:pcdiscretelowerbound} in the forthcoming Section \ref{section:probabilisticguarantees}).
\end{enumerate}
Therefore, the probability that both $E1$ and $E2$ occur, which is denoted as $E$,
resulting in adding a new vertex and a new edge within the clearance of $(\mystatet, \myinputt)$ to the search tree, is positive.
By the truncation operation in Definition \ref{definition: truncation}, the compact motion plan $(\mystatet, \myinputt)$ is discretized into a finite number, say $m\in\nnumbers$, of segments. Each of those segments can be approximated by a solution pair associated with an edge that was added when the event $E$ occurs. 
Therefore, if $E$ occurs less than $m$ times, then HyRRT will fail to generate a motion plan that is close to $(\mystatet, \myinputt)$. 
Proposition \ref{theorem:pc} in the forthcoming Section \ref{section:probabilisticproposition} demonstrates that the probability that $E$ occurs less than $m$ times, leading to the failure of HyRRT, approaches zero as the number of iterations increases, thereby establishing the property in Theorem \ref{theorem:inflatedpc}. A proof of Theorem \ref{theorem:inflatedpc} following these steps is given in the forthcoming Section \ref{section:mainresultproof}.

\subsection{Probabilistic Guarantees on the Function Calls $\texttt{nearest\_neighbor}$ and $\texttt{new\_state}$}\label{section:probabilisticguarantees}
We first characterize the probability that a vertex in the search tree that is close to an existing motion plan is selected as $v_{cur}$ by the function call $\texttt{nearest\_neighbor}$ in Algorithm \ref{algo:hybridRRT}. 
\begin{lemma}
	\label{lemma:nearestvertex}
	Given a hybrid system $\mathcal{H} = (C, f, D, g)$ with state $x\in \mathbb{R}^{n}$ and some $\delta \in \myreals_{> 0}$, let $x_{c}\in \mathbb{R}^{n}$ be such that $x_{c} + \delta\mathbb{B}\subset S$, where $S$ is either $C'$ or $D'$. Suppose that there exists a vertex $v$ in the search graph $\mathcal{T} = (V, E)$ such that $\overline{x}_{v}\in x_{c} + \pn{(}2\delta/5\pn{)}\mathbb{B}$.  Denote by $v_{cur}$ the return of the function call $\texttt{nearest\_neighbor}$ in Algorithm \ref{algo:extend}. When HyRRT is executed uniformly as defined in Definition \ref{assumption:uniformsample}, the probability that $\overline{x}_{v_{cur}}\in x_{c} + \delta\mathbb{B}$ is at least $\frac{\zeta_{n}(\delta/5)^{n}}{\mu(S)}$, where $\zeta_{n}$ is given in ~(\ref{equation:zetan}) and $\mu(S)$ denotes the Lebesgue measure of set $S$. 
\end{lemma}
\begin{proof}
	See Appendix \ref{section:proof_near}.
\end{proof}
\begin{remark}
	Lemma \ref{lemma:nearestvertex} shows that given $x_{c}\in \mathbb{R}^{n}$, when there exists a vertex $v$  such that $\overline{x}_{v}\in x_{c} + \pn{(}2\delta/5\pn{)}\mathbb{B}$, then the probability that the function call $\texttt{nearest\_neighbor}$ selects a vertex that is close to $x_{c}$ is bounded from below by a positive constant. This lemma is used to provide a positive lower bound over the probability that a vertex that is close enough to the motion plan is returned by the function $\texttt{nearest\_neighbor}$ in Algorithm \ref{algo:hybridRRT}.\EndRemark
\end{remark}
The following lemma characterizes the probability that, given an initial state near a specific motion plan as input to the function call $\texttt{new\_state}$, the simulated solution pair in the flow regime is within the clearance of this motion plan. 
\begin{lemma}
	\label{lemma:pccontinuouslowerbound}
	Given a hybrid system $\mathcal{H}$ with state $x\in \myreals[n]$ and input $u\in\myreals[m]$ that satisfies Assumption \ref{assumption:flowlipschitz} and an input library that satisfies item 1 of Definition \ref{assumption:inputlibrary}, let $\psi = (\phi, \myinputt)$ be a purely continuous solution pair to $\mathcal{H}$ with clearance $\delta > 0$, $(\tau, 0) = \max \dom \psi$, and constant input function $\myinputt$. Suppose that $\tau\in (0, T_{m}]$, where $T_{m}$ comes from item 1 in Definition \ref{assumption:inputlibrary}. 
	Let $\psi_{new} = (\phi_{new}, \myinputt_{new})$ be the purely continuous solution pair generated by (\ref{newstate:flow}) in the function call $\texttt{new\_state}$ and initial state $\overline{x}_{v_{cur}} = \phi_{new}(0, 0) \in \phi(0, 0) + \kappa_{1}\delta\mathbb{B}$ for some $\kappa_{1}\in (0, 1/2)$.
	When HyRRT is executed uniformly as defined in Definition \ref{assumption:uniformsample}, for each $\kappa_{2}\in (2\kappa_{1}, 1)$ and each $\epsilon\in (0, \frac{\kappa_{2}\delta}{2})$, there exists $p_{t}\in (0, 1]$ such that 
	\begin{equation}
	\label{equation:lemmaflow}
	\begin{aligned}
	&\mbox{\rm Prob}[E_{1} \& E_{2}]\geq \\
	&p_{t}\frac{\zeta_{n} \left(\max \left\{\min \left\{\frac{\frac{\kappa_{2}\delta}{2} - \epsilon -\exp(K_{x}^{f}\tau)\kappa_{1}\delta}{K^{f}_{u}\tau \exp(K_{x}^{f}\tau)}, \delta\right\}, 0\right\}\right)^{m}}{\mu(U_{C})}.
	\end{aligned}
	\end{equation}
	where 
	\begin{enumerate}
		\item $E_{1}$ denotes the probabilistic event that $\phi$ and $\phi_{new}$ are $(\overline{\tau}, \kappa_{2}\delta)$-close\pn{,} where $(\tau', 0) = \max \dom \phi_{new}$ and $\overline{\tau} = \max \{\tau, \tau'\}$;
		\item $E_{2}$ denotes the probabilistic event that $x_{new} = \phi_{new}(\tau', 0)\in \phi(\tau, 0) + \kappa_{2}\delta \mathbb{B}$, where $x_{new}$ stores the final state of $\phi_{new}$ in the function call $\texttt{new\_state}$ as is introduced in Section \ref{section:newstate},
	\end{enumerate} $\zeta_{n}$ is given in (\ref{equation:zetan}), $\mu(U_{C})$ denotes the Lebesgue measure of $U_{C}$, and $K_{x}^{f}$ and $K_{u}^{f}$ come from Assumption \ref{assumption:flowlipschitz}.  
\end{lemma}
\begin{proof}
	See Appendix \ref{section:proof_continuous}.
\end{proof}
\begin{remark}
	Since $\kappa_{1}\in (0, 1/2)$, $\kappa_{2}\in (2\kappa_{1}, 1)$ and $\delta > 0$, therefore, we have
	\begin{subequations}
		\label{equation:lowboundflow3}
		\begin{align}
		\kappa_{1}\delta < \frac{1}{2}\delta\label{equation:lowboundflow3b}\\
		\kappa_{2}\delta < \delta,\label{equation:lowboundflow3d}
		\end{align}
	\end{subequations}
	ensuring that there is no intersection between $\psi_{new}$ and $X_{u}$, which prevents the function call $\texttt{new\_state}$ from returning $\psi_{new}$, as stated in Section \ref{section:newstate}.

To ensure that Lemma \ref{lemma:pccontinuouslowerbound} provides a positive lower bound, the following additional requirement is imposed on $\kappa_{1}$ and $\kappa_{2}$:
\begin{equation}
\frac{\kappa_{2}}{2} > \kappa_{1}\label{equation:lowboundflow3c}
\end{equation}
Then, there exists $\epsilon\in (0, \frac{\kappa_{2}\delta}{2})$ satisfying
	\begin{equation}\label{equation:lowboundflow1}
		\epsilon < \frac{\kappa_{2}\delta}{2} - \exp(K_{x}^{f}\tau)\kappa_{1}\delta
	\end{equation}
	to guarantee that the first argument in the $\min$ operator in (\ref{equation:lemmaflow}) is positive. 
	From (\ref{equation:lowboundflow1}) and $\epsilon > 0$, any $\tau \in \myreals_{> 0}$ satisfying the following condition suffices for the existence of $\epsilon$ satisfying (\ref{equation:lowboundflow1}):
	\begin{equation}\label{equation:lowboundflow2}
		\tau < \frac{\ln{\frac{\kappa_{2}}{2\kappa_{1}}}}{K^{f}_{x}}\\
	\end{equation}
	where, by (\ref{equation:lowboundflow3c}), $\frac{\ln{\frac{\kappa_{2}}{2\kappa_{1}}}}{K^{f}_{x}} > \frac{\ln{\frac{2\kappa_{1}}{2\kappa_{1}}}}{K^{f}_{x}} = 0.$ 
	
	Therefore, if $\kappa_{1}$ and $\kappa_{2}$ adhere to the constraints in (\ref{equation:lowboundflow3b}), (\ref{equation:lowboundflow3d}) and (\ref{equation:lowboundflow3c})  and if $\kappa_{1}$, $\kappa_{2}$, and $\tau$ satisfy the conditions in (\ref{equation:lowboundflow2}), then the lower bound in (\ref{equation:lemmaflow}) is guaranteed to be positive.
\end{remark}
The following result characterizes the probability that, given an initial state near a specific motion plan as input to the function call $\texttt{new\_state}$, the simulated solution pair computed by the function call $\texttt{new\_state}$ in the jump regime is within the clearance of this motion plan. 
\begin{lemma}
	\label{lemma:pcdiscretelowerbound}
	Given a hybrid system $\mathcal{H}$ with state $x\in \myreals[n]$ and input $u\in\myreals[m]$ that satisfies Assumption \ref{assumption:pcjumpmap} and an input library that satisfies item 2 of Definition \ref{assumption:inputlibrary}, let $\psi = (\phi, \myinputt)$ be a purely discrete solution pair to $\mathcal{H}$ with a single jump, i.e., $\max \dom \psi = (0, 1)$ and clearance $\delta > 0$.
	Let $\psi_{new} = (\phi_{new}, \myinputt_{new})$ be the purely discrete solution pair generated by (\ref{newstate:jump}) in the function call $\texttt{new\_state}$ and initial state $\overline{x}_{v_{cur}} = \phi_{new}(0, 0)\in \phi(0, 0) + \kappa_{1}\delta\mathbb{B}$ for some positive $\kappa_{1}\in (0, 1)$.
	When HyRRT is executed uniformly as defined in Definition \ref{assumption:uniformsample}, for each $\kappa_{2}\in (0, 1)$,  we have that
	\begin{equation}\label{equation:lemmadiscrete}
	\mbox{\rm Prob}[E]\geq \frac{\zeta_{n} \left(\max \left\{\min \left\{\frac{(\kappa_{2} - K^{g}_{x}\kappa_{1})\delta}{K^{g}_{u}}, \delta\right\}, 0\right\}\right)^{m}}{\mu(U_{D})}
	\end{equation}
	where $E$ denotes the probabilistic event that $x_{new} = \phi_{new}(0, 1) \in \phi(0, 1) + \kappa_{2} \delta\mathbb{B}$, $x_{new}$ stores the final state of $\phi_{new}$ in the function call $\texttt{new\_state}$ as is introduced in Section \ref{section:newstate}, $\zeta_{n}$ is given in (\ref{equation:zetan}), $\mu(U_{D})$ denotes the Lebesgue measure of $U_{D}$, and $K^{g}_{x}$ and $K^{g}_{u}$ come from Assumption \ref{assumption:pcjumpmap}.  
\end{lemma}
\begin{proof}
	See Appendix \ref{section:proof_discrete}.
\end{proof}
\begin{remark}
	Since $\kappa_{1}\in (0, 1]$, $\kappa_{2}\in (0, 1]$ and $\delta > 0$, therefore, we have
	\begin{subequations}
		\label{equation:lowboundjump3}
		\begin{align}
		\kappa_{1}\delta < \delta\label{equation:lowboundjump3b1}\\
		\kappa_{2}\delta < \delta,\label{equation:lowboundjump3d}
		\end{align}
	\end{subequations}
ensuring that there is no intersection between $\psi_{new}$ and $X_{u}$, which prevents the function call $\texttt{new\_state}$ from returning $\psi_{new}$. 

	To ensure that the lower bound in (\ref{equation:lemmadiscrete}) is positive, an additional requirement is imposed on $\kappa_{1}$ and $\kappa_{2}$:
	\begin{equation}
	\label{equation:jumppossibilityinequality}
	\frac{\kappa_{2}}{\kappa_{1}} > K^{g}_{x},
	\end{equation}
	Since this implies that $\frac{(\kappa_{2} - K^{g}_{x}\kappa_{1})\delta}{K^{g}_{u}} > 0$,
	the first argument of the $\min$ operator in (\ref{equation:lemmadiscrete}) is positive.
	Therefore, if $\kappa_{1}$ and $\kappa_{2}$ adhere to the constraints in (\ref{equation:lowboundjump3b1}), (\ref{equation:lowboundjump3d}), and (\ref{equation:jumppossibilityinequality}), then the lower bound in (\ref{equation:lemmadiscrete}) is guaranteed to be positive.
	
	Note that if $\phi_{new}(0, 0)\in \phi(0, 0) + \kappa_{1}\delta\mathbb{B}$ and $\phi_{new}(0, 1)\in \phi(0, 1) + \kappa_{2}\delta\mathbb{B}$ are satisfied, $\phi$ and $\phi_{new}$ are $(1, \max \{\kappa_{1}\delta, \kappa_{2}\delta\})$-close.
	\EndRemark
\end{remark}
\subsection{Probabilistic Completeness Guarantee of Finding a Motion Plan with Positive Clearance}\label{section:probabilisticproposition}
We then establish a preliminary result demonstrating that,  if there exists a motion plan to $\mathcal{P}$ with positive clearance and all assumptions outlined in Section \ref{section:assumption} are met, HyRRT is probabilistically complete. 
\begin{proposition}
	\label{theorem:pc}
	Given a motion planning problem $\mathcal{P} = (X_{0}, X_{f}, X_{u}, (C, f,  D, g))$, suppose that Assumptions \ref{assumption:finitesets}, \ref{assumption:flowlipschitz}, and \ref{assumption:pcjumpmap} are satisfied, and there exists a compact motion plan $(\mystatet, \myinputt)$ to $\mathcal{P}$ with clearance $\delta > 0$ and piecewise-constant input. Then, using a complete input library and when executed uniformly as defined in Definition \ref{assumption:uniformsample}, the probability that HyRRT  fails to find a motion plan $\psi' = (\phi', \myinputt')$ such that $\phi'$ is $(\tilde{\tau}, \delta)$-close to $\phi$ after $k$ iterations is at most $a e^{-bk}$, where $a, b\in \mathbb{R}_{> 0}$, $\tilde{\tau} = \max\{T+ J, T' + J'\}$, $(T, J) = \max \dom \psi$, and $(T', J') = \max \dom \psi'$.
\end{proposition}
\begin{proof}
	To show the claim, first, given a compact motion plan $(\mystatet, \myinputt)$ to $\mathcal{P}$ with clearance $\delta > 0$, we construct a sequence of hybrid time instances 
	\begin{equation}\label{equation:calT}
		\mathbb{T} := \{(T_{i}, J_{i})\in\dom (\mystatet, \myinputt)\}_{i = 0}^{m}
	\end{equation}
	for some $m \in \nnumbers\backslash\{0\}$ to be chosen later. Alongside, we construct a corresponding geometric sequence\endnote{A geometric sequence is a sequence of non-zero numbers where each term after the first is found by multiplying the previous one by a fixed, non-zero number called the common ratio.} of positive real numbers
	\begin{equation}
		\mathbb{D} := \{r_{i} \in \mathbb{R}_{> 0}\}_{i = 0}^{m},
	\end{equation} with positive common ratio $q\in \mathbb{R}_{> 0}$, in which case, 
	$$
	r_{i} := q^{i}r_{0} \quad \forall i \in \{0, 1, ..., m\}.
	$$

	Later in this proof, for each $i \in \{0, 1, ..., m - 1\}$, 	Lemmas \ref{lemma:nearestvertex}, \ref{lemma:pccontinuouslowerbound}, and \ref{lemma:pcdiscretelowerbound} are utilized to provide positive lower bounds for their respective probabilistic events for the balls $\mystatet(T_{i}, J_{i}) + r_{i}\myball$ and $\mystatet(T_{i + 1}, J_{i + 1}) + r_{i + 1}\myball$, where $(T_{i}, J_{i}), (T_{i + 1}, J_{i + 1}) \in \mathbb{T}$ and $r_{i}, r_{i + 1}\in \mathbb{D}$. This requires that $\mystatet(T_{i}, J_{i}) + r_{i}\myball$ and $\mystatet(T_{i + 1}, J_{i + 1}) + r_{i + 1}\myball$ meet each condition to establish these positive lower bounds.
	
	The construction of $\mathbb{D}$ starts with selecting a proper common ratio $q$ and the initial term $r_{0}$ in $\mathbb{D}$. To ensure that the positive lower bounds stated in Lemmas \ref{lemma:nearestvertex}, \ref{lemma:pccontinuouslowerbound}, and \ref{lemma:pcdiscretelowerbound} hold, the geometric sequence $\mathbb{D}$ must meet the following conditions:
	\begin{enumerate}[label=D\arabic*)]
		\item For each $i\in \{0, 1, ..., m - 1\}$, $\frac{5}{2}r_{i} < r_{i + 1}$ to meet the requirements for Lemma \ref{lemma:nearestvertex};
		\item For each $i\in \{0, 1, ..., m - 1\}$, $2r_{i} < \min \{\delta, r_{i + 1}\}$  to meet the conditions for Lemma \ref{lemma:pccontinuouslowerbound} in (\ref{equation:lowboundflow3b}) and (\ref{equation:lowboundflow3c}) and the condition for Lemma \ref{lemma:pcdiscretelowerbound} in (\ref{equation:lowboundjump3}). In addition, it is required that $r_{m} <  \delta$ as per  (\ref{equation:lowboundflow3d});
		\item For each $i\in \{0, 1, ..., m - 1\}$, the condition $K^{g}_{x}r_{i} < r_{i + 1}$ is essential to comply with the prerequisites for Lemma \ref{lemma:pcdiscretelowerbound} in (\ref{equation:jumppossibilityinequality}), where $K^{g}_{x}$ comes from Assumption \ref{assumption:pcjumpmap} and is fixed;
		\item $\sum_{i = 0}^{m} r_{i} < \delta$, thereby ensuring, as per Proposition \ref{proposition:concatenateclose}, that the concatenation of all the $\phi_{new}$'s, which are returned by the function call $\texttt{new\_state}$ in Section \ref{section:newstate} and satisfy the properties introduced in probabilistic events in Lemmas  \ref{lemma:pccontinuouslowerbound} and \ref{lemma:pcdiscretelowerbound}, remains within the clearance $\delta$.
	\end{enumerate} 
	To satisfy D1 - D3, the common ratio $q$ is selected as any positive number satisfying the following inequality:
	$$
	q > \max\left\{\frac{5}{2}, 2, K_{x}^{g}\right\} = \max\left\{\frac{5}{2}, K_{x}^{g}\right\}
	$$
	which inherently implies $q > \frac{5}{2}$. 
	Given that $q > \frac{5}{2} > 1$, the sequence $\mathbb{D}$ is, therefore, monotonically increasing.  With the selected $q$, the initial term in $\mathbb{D}$ is selected as
	$$
	r_{0} = \frac{\delta}{q^{m + 1}},
	$$ ensuring that, for each $i\in \{0, 1, ..., m -1\}$,
	$2r_{i} < 2r_{m - 1}  = \frac{2\delta}{q^2} < \frac{8\delta}{25} < \delta$ and $r_{m} = \frac{\delta}{q} < \frac{2\delta}{5} < \delta$.  Furthermore, the sum\endnote{The closed-form of the sum of a geometric sequence is $\sum_{i = 0}^{m} r_{i} = \frac{r_{0}(1 - q^{m + 1})}{1 - q}$} of $\mathbb{D}$, given by 
	$$
	\begin{aligned}
	\sum_{i = 0}^{m} r_{i} &= \frac{\delta}{q^{m + 1}(1 - q)} - \frac{\delta}{1 - q} = \frac{\delta}{q - 1}\left(1 - \frac{1}{q^{m + 1}}\right) \\
	&< \frac{\delta}{q - 1}
	\end{aligned}
	$$ where $q > \frac{5}{2}$, demonstrates that $\sum_{i = 0}^{m} r_{i} < \frac{\delta}{\frac{5}{2} - 1}  = \frac{2\delta}{3} < \delta$, thereby satisfying D4.
	Therefore, this selection method satisfies D1 - D4 for each $m \in \nnumbers\backslash\{0\}$, with the specific value of $m$ to be selected subsequently.
	
	The sequence $\mathbb{T}$ is constructed such that for each $i\in \{0, 1, ..., m - 1\}$, $(T_{i}, J_{i})$ and $(T_{i + 1}, J_{i + 1})$ satisfy one of the following two conditions:  
	\begin{enumerate}[label=T\arabic*)]
		\item $T_{i} = T_{i + 1}$ and $J_{i + 1} = J_{i}  + 1$, indicating a jump from $(T_{i}, J_{i})$ to $(T_{i + 1}, J_{i + 1})$;
		\item $T_{i + 1} - T_{i} = \Delta t'$ and $J_{i + 1} = J_{i}$ for some $\Delta t' = \frac{\Delta t}{k}$, implying a purely continuous evolution from $(T_{i}, J_{i})$ to $(T_{i + 1}, J_{i + 1})$, where $\Delta t$ is the resolution of the piecewise-constant input as in Definition \ref{definition:piecewiseconstantfun}, $T_{m}$ is the upper bound for the time duration of the input signals in the complete $\mathcal{U}_{C}$ as outlined in Definition \ref{assumption:inputlibrary}, and $k\in \nnumbers\backslash\{0\}$ is chosen large enough such that $\Delta t' = \frac{\Delta t}{k}$ is small enough to satisfy
			$
			\Delta t'\leq \frac{\ln{\frac{q}{2}}}{K^{f}_{x}},
			$ thereby satisfying (\ref{equation:lowboundflow2}), and
			$
			 \Delta t'\leq T_{m}
			$ to establish a positive lower bound in Lemma~\ref{lemma:pccontinuouslowerbound}.
	\end{enumerate} 
Since $\myinputt$ is piecewise-constant, as per Definition \ref{definition:piecewiseconstantfun}, for any $k\in\nnumbers\backslash\{0\}$, $\Delta t' = \frac{\Delta t}{k}$ is also a resolution of the piecewise-constant input $\myinputt$. Furthermore, Definition \ref{definition:mppiecewiseconstant} implies that if $(t, j)\in \dom \myinputt$, $(t, j + 1) \notin \dom \myinputt$ and $t + \Delta t' \leq \max_{t} \dom \myinputt$, it follows that $(t + \Delta t', j)\in \dom (\mystatet, \myinputt)$; see (\ref{equation:maxt}) for the definition of $\max_{t}$. Therefore, it is indeed feasible to construct $\mathbb{T}$ by sweeping through the hybrid time domain of $(\mystatet, \myinputt)$ as follows: Starting from the initial assignment $i \leftarrow 0$ and $(T_{i}, J_{i}) \leftarrow (0, 0)$,  \begin{enumerate}
	\item if $(T_{i}, J_{i} + 1) \in \dom (\mystatet, \myinputt)$, it follows that $(T_{i + 1}, J_{i + 1})\leftarrow (T_{i}, J_{i } + 1)$, which satisfies T1;
	\item if $(T_{i}, J_{i} + 1) \notin \dom (\mystatet, \myinputt)$, it follows that $(T_{i + 1}, J_{i + 1})\leftarrow (T_{i} + \Delta t', J_{i })$, which satisfies~T2. 
\end{enumerate} This process is continued by incrementing $i\leftarrow i + 1$ until $(T_{i}, J_{i}) = \max \dom (\mystatet, \myinputt)$. 
This construction of $\mathbb{T}$ results in
$$
m = \frac{T}{\Delta t'} + J + 1
$$
where $(T, J) = \max \dom (\mystatet, \myinputt)$.

With this construction, the motion plan between $(T_{i}, J_{i})$ and $(T_{i + 1}, J_{i + 1})$ is either purely continuous or purely discrete, which allows the use of Lemmas~\ref{lemma:pccontinuouslowerbound} and ~\ref{lemma:pcdiscretelowerbound} in each respective regime. Therefore, the construction of $\mathbb{T}$ and $\mathbb{D}$ adheres to all the conditions required to apply Lemmas~\ref{lemma:nearestvertex}, \ref{lemma:pccontinuouslowerbound}, and \ref{lemma:pcdiscretelowerbound}.

Next, we partition the compact motion plan into a finite number of segments. By demonstrating that the probability of HyRRT generating $\mystatet_{new}$ close to each segment is positive, we establish that the probability of HyRRT failing to find a motion plan converges to zero as the number of iterations approaches infinity. Using the truncation and translation operations defined in Definition ~\ref{definition: truncation}, for each $i\in \{0, 1, ..., m - 1\}$, let $(\mystatet_{i}, \myinputt_{i})$ represent the truncation of $(\mystatet, \myinputt)$ between $(T_{i}, J_{i})$ and $(T_{i + 1}, J_{i + 1})$ following the translation by $(T_{i}, J_{i})$. Given that $(T_{i}, J_{i})$ and $(T_{i + 1}, J_{i + 1})$ satisfy T1 or T2, it follows that each element in $\{\phi_{i}\}_{i = 0}^{m - 1}$ is either purely continuous with a constant input or purely discrete with a single jump.  
Then, with the search tree denoted $\mathcal{T} = (V, E)$, the proof proceeds by showing that the probability of each of the following probabilistic events is positive:
\begin{enumerate}[label=P\arabic*)]
	\item The function call $\mathcal{T}.\texttt{init}(X_{0})$ in Section \ref{section:init} adds a vertex associated with some $x_{0}\in \phi(T_{0}, J_{0}) + \frac{2}{5}r_{0}\mathbb{B}$ to $\mathcal{T}$. 
	\item For each $i\in \{0, 1, , m - 1\}$, given the existence of a vertex $v\in V$ such that $x_{v}\in \phi(T_{i}, J_{i}) + \frac{2}{5}r_{i}\mathbb{B}$, the function call $\texttt{extend}$ in Algorithm \ref{algo:extend} adds a new vertex to $\mathcal{T}$ associated with $x_{new}$, along with a corresponding new edge to $\mathcal{T}$ associated with $\psi_{new} = (\mystatet_{new}, \myinputt_{new})$, such that 
	\begin{enumerate}[label=PE\arabic*)]
		\item $x_{new}\in \phi(T_{i + 1}, J_{i + 1}) + \frac{2}{5}r_{i + 1}\mathbb{B}$;
		\item $\mystatet_{new}$ is $(\overline{\tau}, r_{i + 1})$-close to $\mystatet_{i}$, where\endnote{Given that $\phi_{i}$ is purely continuous, by T2 and Definition \ref{definition: truncation}, it follows that $\max_{t} \dom \mystatet_{i} = \Delta t'$} 
		$$
		\overline{\tau} = \left\{
		\begin{aligned}
		&\max \hspace{-0.2cm}&\{\text{max}_{t} \dom \mystatet_{new}, \text{max}_{t} \dom \mystatet_{i}\}\\
		&&\text{ if $\phi_{i}$ is purely continuous}\\
		&1&\text{ if $\phi_{i}$ is purely discrete}.\\
		\end{aligned}\right.
		$$
	\end{enumerate}
\end{enumerate}
The process of HyRRT in finding a  motion plan can be viewed as a Bernoulli trial\endnote{In the theory of probability and statistics, a Bernoulli trial is a random experiment with exactly two possible outcomes, success and failure.} (P1) following repeated Bernoulli trials (P2).
Denote the probability that the probabilistic event P1 occurs as $p_{init}$ and that P2 occurs as $p_{extend}$. Next, we show that $p_{init}$ and $p_{extend}$ are positive.
\begin{enumerate}
	\item 
	Given that HyRRT is executed uniformly (see Definition \ref{assumption:uniformsample}), by (\ref{equation:probabilitylebesgue}) and $\mu(\phi(T_{0}, J_{0}) + \frac{2}{5}r_{0}\mathbb{B}) > 0$, where $\mu(\phi(T_{0}, J_{0}) + \frac{2}{5}r_{0}\mathbb{B})$ represents the Lebesgue measure of the ball $\phi(T_{0}, J_{0}) + \frac{2}{5}r_{0}\mathbb{B}$, it follows that  $p_{init} > 0$.
	\item 
	The function calls $\texttt{nearest\_neighbor}$ and $\texttt{new\_state}$ are executed in the function call $\texttt{extend}$. We first show the probability that each of $\texttt{nearest\_neighbor}$ and $\texttt{nearest\_neighbor}$ contributes to an output of $\texttt{extend}$ satisfying PE1 and PE2 is positive:
	\begin{enumerate}
		\item Given the existence of a vertex $v\in V$ such that $x_{v}\in \phi(T_{i}, J_{i}) + \frac{2}{5}r_{i}\mathbb{B}$, Lemma ~\ref{lemma:nearestvertex} implies that the probability that the function call $\texttt{nearest\_neighbor}$ returns a vertex $v_{cur}$ such that $\overline{x}_{v_{cur}}\in \phi(T_{i}, J_{i}) + r_{i}\mathbb{B}$, denoted $p_{near}$, is positive.
		\item Under the condition that $\overline{x}_{v_{cur}}\in \phi(T_{i}, J_{i}) + r_{i}\mathbb{B}$ and $r_{i} < \frac{2}{5}r_{i + 1}$ ensured by D1, the probability that the function call $\texttt{new\_state}$ generates a new solution pair $\psi_{new} = (\phi_{new}, \myinputt_{new})$ satisfying PE1 and PE2, denoted $p_{new}$, is positive. This property is guaranteed by Lemma~\ref{lemma:pccontinuouslowerbound}, if $\phi_{i}$ is purely continuous and by Lemma~\ref{lemma:pcdiscretelowerbound} if $\phi_{i}$ is purely discrete.
	\end{enumerate}
	Therefore, since $p_{near}$ and $p_{new}$ are positive, it follows that $p_{extend} := p_{near}p_{new}$ is positive.
\end{enumerate}
Given that the probability of P2 occurring is positive, and denoting the total number of occurrences of P2 trials as $X_{n}$, where $n\in\mathbb{N}\backslash\{0\}$, we characterize the probability of HyRRT failing to identify a motion plan, which is equivalent to P2 occurring fewer than $m$ times, as follows:
\begin{equation}
\begin{aligned}
\lim_{n\to \infty} &\mbox{\rm Prob}[X_{n} < m] \\
&= \lim_{n\to \infty} \sum_{i =0}^{m - 1}\left(\begin{matrix}
n\\
i
\end{matrix}\right)p_{extend}^{i}(1 - p_{extend})^{n - i}\\
\end{aligned}
\end{equation}
Since $i\in \{0, 1, ..., m - 1\}$, we have $0\leq i \leq m - 1$. As  $n\to \infty$, there exists an $N \in \mathbb{N}$ such that for all $n \geq N$, we have $0\leq i \leq m - 1 \leq \frac{n}{2}$. Such an $N$ can be chosen as $2(m - 1)$.
Due to the increasing monotonicity\endnote{The increasing monotonicity of $\left(\begin{matrix}
	n\\ 
	i
	\end{matrix}\right)$ for $i$ over the interval $(0, \frac{n}{2})$ is exploited here. We can prove it by arbitrarily selecting $k_{1}, k_{2}\in (0, \frac{n}{2})\cap\naturals$ such that $k_{1} > k_{2}$ and then showing that 
\begin{equation}
\begin{aligned}
\frac{\left(\begin{matrix}
	n\\ 
	k_{1}
	\end{matrix}\right)}{\left(\begin{matrix}
	n\\ 
	k_{2}
	\end{matrix}\right)} &= \frac{\frac{n!}{k_{1}(n - k_{1})!} }{\frac{n!}{k_{2}(n - k_{2})!}}= \frac{\frac{(n - k_{2})!}{(n - k_{1})!}}{\frac{k_{1}!}{k_{2}!}}  \\
& = \frac{(n - k_{2})(n - k_{2} + 1)\cdots (n - k_{1} + 1)}{k_{1}(k_{1} - 1)\cdots(k_{2} + 1)} \\
 &= \frac{n - k_{2}}{k_{1}}\times\frac{n - k_{2} - 1}{k_{1} - 1}\times\cdots\times\frac{n - k_{1} + 1}{k_{2} + 1} > 1,
\end{aligned}
\end{equation} implying $\left(\begin{matrix}
n\\ 
k_{1}
\end{matrix}\right) > \left(\begin{matrix}
n\\ 
k_{2}
\end{matrix}\right)$} of $\left(\begin{matrix}
	n\\
	k
	\end{matrix}\right)$ for $k$ from $0$ to $\frac{n}{2}$ and $0\leq i \leq m - 1 \leq \frac{n}{2}$, we have $\left(\begin{matrix}
n\\
i
\end{matrix}\right) \leq \left(\begin{matrix}
n\\
m - 1
\end{matrix}\right)$. Hence,
\begin{equation}
\begin{aligned}
\lim_{n\to \infty} &\mbox{\rm Prob}[X_{n} < m] \\
&\leq \lim_{n\to \infty} \sum_{i = 0}^{m - 1} \left(\begin{matrix}
n\\
m - 1
\end{matrix}\right)p_{extend}^{i}(1 - p_{extend})^{n - i}\\
&\leq \lim_{n\to \infty} \left(\begin{matrix}
n\\
m - 1
\end{matrix}\right)\sum_{i = 0}^{m - 1}p_{extend}^{i}(1 - p_{extend})^{n - i}.\\
\end{aligned}
\end{equation}

Define $p := \max\{p_{extend}, 1 - p_{extend}\} \in (0, 1)$.
It follows that
\begin{equation}
\begin{aligned}
\lim_{n\to \infty} \mbox{\rm Prob}[X_{n} < m] &\leq \lim_{n\to \infty} \left(\begin{matrix}
n\\
m - 1
\end{matrix}\right)\sum_{i = 0}^{m - 1}p^{i} p^{n - i}\\
& \leq \lim_{n\to \infty} \left(\begin{matrix}
n\\
m - 1
\end{matrix}\right)\sum_{i = 0}^{m - 1}p^{n}.\\
\end{aligned}
\end{equation}
Since $p^{n}$ is not related to $i$, we have $\sum_{i = 0}^{m - 1}p^{n} = mp^{n}$. It follows that
\begin{equation}
\begin{aligned}
\lim_{n\to \infty} \mbox{\rm Prob}[X_{n} < m] &\leq\lim_{n\to \infty}\left(\begin{matrix}
n\\
m - 1
\end{matrix}\right)mp^{n}\\
&\leq \lim_{n\to \infty}\frac{n!}{(m - 1)!(n - m + 1)!}mp^{n} 
\end{aligned}
\end{equation}
Since $\frac{n!}{(n - m + 1)!}  = \Pi_{i = n - m + 2}^{n}i \leq \Pi_{i = n - m + 2}^{n} n = n^{m - 1}$, it follows that
\begin{equation}
\begin{aligned}
\lim_{n\to \infty} \mbox{\rm Prob}[X_{n} < m] &\leq\lim_{n\to \infty}\frac{n^{m - 1}}{(m - 1)!}mp^{n} \\
\end{aligned}
\end{equation}
Since $p\in (0, 1)$ and $m$ is finite, then $\lim_{n\to \infty}\frac{n^{m - 1}}{(m - 1)!}mp^{n} = 0$. Therefore, the probability that HyRRT fails to find a motion plan is converging to zero as the number of iterations approaches infinity.

In addition, Proposition \ref{proposition:concatenateclose} establishes that the concatenation of all the $\phi_{new}$'s satisfying PE2, which is returned by HyRRT as $\mystatet'$, is $(\tilde{\tau}, \sum_{i = 0}^{m} r_{i})$-close to $\phi = \mystatet_{0}|\mystatet_{1}|...|\mystatet_{m}$, where $\sum_{i = 0}^{m} r_{i} < \delta$ is guaranteed by D4.
%
%
	
\end{proof}

\subsection{Proof of Theorem \ref{theorem:inflatedpc}}\label{section:mainresultproof}
	Given that $\psi$ is a motion plan to $\mathcal{P}$ with safety clearance $\delta_{s} > 0$ and that $\delta_{f} > 0$, Lemma \ref{lemma:motionplanintlated} establishes that $\psi$ is a motion plan to $\mathcal{P}_{\delta}$ with clearance $\delta = \min
	\{\delta_{s}, \delta_{f}\} >0$. Furthermore, by Proposition \ref{theorem:pc}, it follows that the probability that HyRRT fails to find $\psi' = (\mystatet', \myinputt')$ is at most $ae^{-bk}$ and the generated $\phi'$ is $(\tilde{\tau}, \tilde{\delta})$-close to $\phi$ where $\tilde{\tau} = \max\{T+ J, T' + J'\}$, $(T, J) = \max \dom \psi$, $(T', J') = \max \dom \psi'$, and $\tilde{\delta} = \min \{\delta_{s}, \delta_{f}\}$.

\section{HyRRT Software Tool for Motion Planning for Hybrid Systems and Examples}
\label{section:illustration}
Algorithm \ref{algo:hybridRRT} leads to a software tool\endnote{Code at \href{https://github.com/HybridSystemsLab/hybridRRT}{https://github.com/HybridSystemsLab/hybridRRT}.} to solve the motion planning problems for hybrid systems. This software only requires the motion planning problem data $(X_{0}, X_{f}, X_{u}, (C, f, D, g))$, an input library $(\mathcal{U}_{C}, \mathcal{U}_{D})$, a tunable parameter $p_{n}\in (0, 1)$, an upper bound $K$ over the iteration number and two constraint sets $X_{c}$ and $X_{d}$. The tool is illustrated in Examples \ref{example:bouncingball} and \ref{example:biped}.

\begin{example}[(Actuated bouncing ball system in Example \ref{example:bouncingball}, revisited)]\label{example:illustrationbb}
	This example serves to demonstrate that the HyRRT algorithm is proficient in solving the specific instance of the motion planning problem as illustrated in Example \ref{example:bouncingball}.
	The coefficients of the hybrid model in Example \ref{example:bouncingball} are given as follows:
	\begin{enumerate}
		\item The gravity constant $\gamma$ is set to $9.81$.
		\item The coefficient of restitution is set to $0.8$.
	\end{enumerate} 
	
	The inputs fed to the proposed algorithm are given as follows:
	\begin{enumerate}
	\item The tune parameter $p_{n}$ is set as $0.5$.
	\item The upper bound $K$ is set to $1000$.
	\item The construction of the input library $(\mathcal{U}_{C}, \mathcal{U}_{D})$ adheres to the methodology delineated in {\bf Input Library Construction Procedure}, as referenced in Section \ref{section:samplingmotionprimitives}. Here, the upper bound $T_{m}$ is explicitly configured to $0.1$. From the setting of $X_{u}$,  It is imperative to note that each input falling outside the interval $(0, 5)$ is unsafe. As a result, both $U^{s}_{C}$ and $U_{D}$ are restricted to $(0, 5)$. Consequently, both $\mathcal{U}_{C}$ and $\mathcal{U}_{D}$ are constructed in compliance with the specifications outlined in the {\bf Input Library Construction Procedure}.
		\item The constraint sets $X_{c}$ and $X_{d}$ are set as $X_{c} = C'$ and $X_{d} = D'$.
		\item The tolerance $\epsilon$ in (\ref{equation:tolerance}) is set to $0.2$. 
	\end{enumerate}
	
	The simulation result is shown in Figure \ref{fig:hybridrrtresultbb}. The simulation is implemented in MATLAB software and processed by a $2.2$ GHz Intel Core i7 processor. On average, over the course of $20$ runs, the simulation is observed to require approximately $0.72$ seconds for completion and results in the generation of $34.2$ vertices during the propagation phase.
	\begin{figure}[htbp]
		\centering
		\includegraphics[width = 0.4\textwidth]{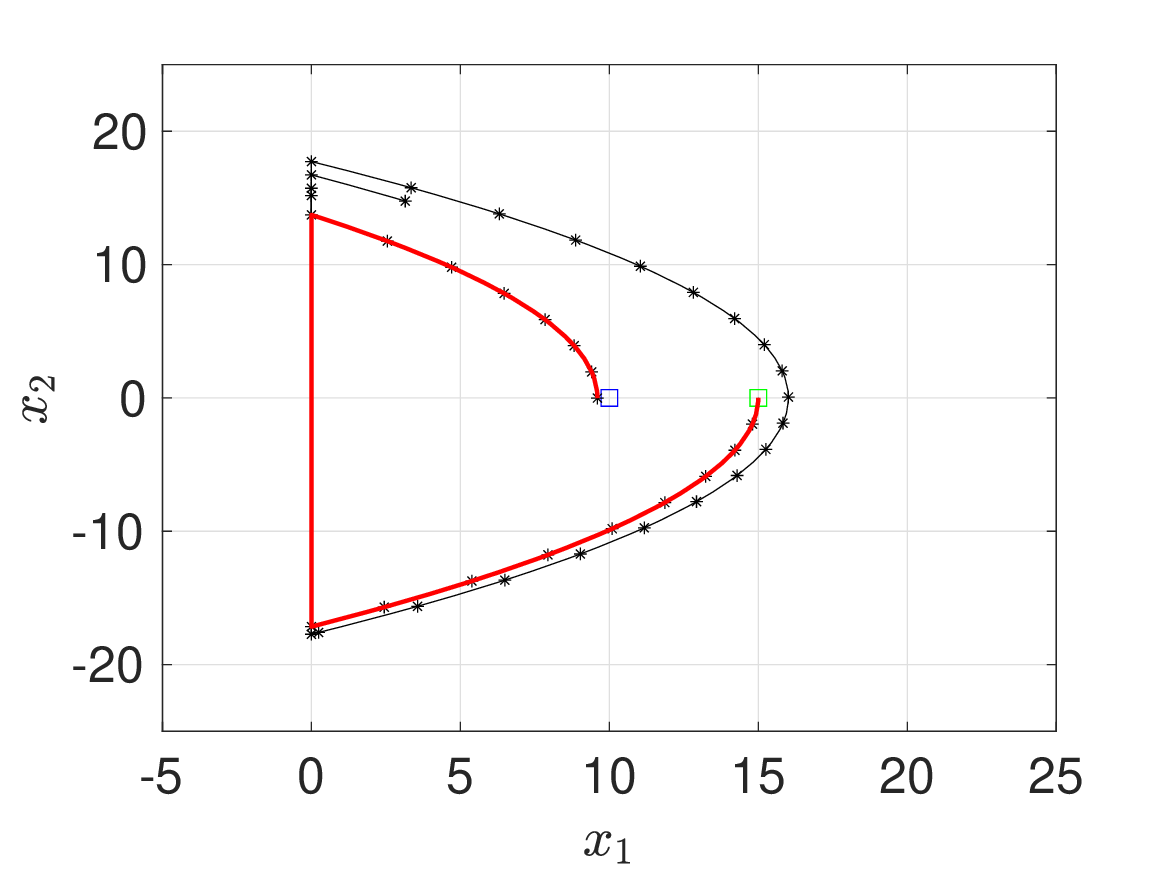}
		\caption{The above shows the search results of HyRRT algorithm to solve the sample motion planning problem in Example \ref{example:bouncingball}. Green square denotes $X_{0}$ and blue square denotes $X_{f}$. The states represented by vertices in the search tree are denoted by $\star$'s. The lines between $\star$'s denote the state trajectories of the solution pairs associated with edges in the search tree. The red trajectory denotes the state trajectory of a motion plan to the given motion planning problem.}\label{fig:hybridrrtresultbb}
	\end{figure}
\end{example}
\begin{example}[(Walking robot system in Example \ref{example:biped}, revisited)] \label{example:bipedillustration}
	The coefficients of the hybrid model in Example \ref{example:biped} are given as follows:
	\begin{enumerate}
		\item The step angle $\phi_{s}$ is set to $0.7$.
		\item $a_{1}^{\min} = a_{2}^{\min} = -3$, $a_{1}^{\max} = a_{2}^{\max} = 3$, $a_{3}^{\min} = - 0.2$, $a_{3}^{\max} =  0.2$.
	\end{enumerate}
	
	The settings for HyRRT planner are given as follows:
	\begin{enumerate}[label=\arabic*)]
		\item The tune parameter $p_{n}$ is set as $0.9$ to encourage the flow regime. 
		\item The upper bound $K$ is set as $2000$.
		\item The input library $(\mathcal{U}_{C}, \mathcal{U}_{D})$  is formulated in accordance with {\bf Input Library Construction Procedure}. Specifically, the upper bound $T_{m}$ is set to $0.4$. Given the constraints imposed by $X_{u}$,  each input deviating from the set $ (-3, 3)\times(-3, 3)\times(-0.2, 0.2)$ is classified as unsafe. Consequently, both $U^{s}_{C}$ and $U_{D}$ are restricted to $(-3, 3)\times(-3, 3)\times(-0.2, 0.2)$.  In light of these constraints, both $\mathcal{U}_{C}$ and $\mathcal{U}_{D}$ are constructed in compliance with the specifications outlined in the \textbf{Input Library Construction Procedure}.
		\item The constraint set $X_{c}$ is chosen as $\{(x, a)\in \mathbb{R}^{6}\times \mathbb{R}^{3}: h(x)\geq -s\}$ and $X_{d}$ as $\{(x, a)\in \mathbb{R}^{6}\times \mathbb{R}^{3}: h(x) = 0, \omega_{p} \geq -s\}$ with a tunable parameter $s$ set to $0$, $0.3$, $0.5$, $1$, and $2$, such that
		$C = X_{c}|_{s = 0} \subsetneq X_{c}|_{s = 0.3}\subsetneq X_{c}|_{s = 0.5}\subsetneq X_{c}|_{s = 1}\subsetneq X_{c}|_{s = 2}$ and $D = X_{d}|_{s = 0} \subsetneq X_{d}|_{s = 0.3}\subsetneq X_{d}|_{s = 0.5}\subsetneq X_{d}|_{s = 1}\subsetneq X_{d}|_{s = 2}$.
		
		\item The tolerance $\epsilon$ in (\ref{equation:tolerance}) is set to $0.3$
	\end{enumerate} 
The simulation result in Figure \ref{fig:illustrationbiped} shows that HyRRT is able to solve the instance of motion planning problem for the walking robot. The simulation is implemented in MATLAB and processed by a $3.5$ GHz Intel Core i5 processor. Over the course of $20$ runs, on average, the simulation takes $63.26/78.3/100.4/183.7/280.8$ seconds with $s$ set to $0/0.3/0.5/1.0/2.0$, respectively. Among all the runs conducted, the simulation takes a minimum of $31.2$ seconds to complete. In contrast, when employing the forward/backward propagation algorithm based on breadth-first search operating under identical settings, the motion planner takes $1608.2$ seconds to solve the same problem. The results demonstrate a noteworthy improvement provided by the rapid exploration technique, with average and fastest run computation time improvements of $96.1\%$ and $98.1\%$, respectively. It is also observed that as the sets $X_{c}$ and $X_{d}$ grow, HyRRT considers more vertices in solving Problem \ref{problem:nearestneighbor} leading to higher computation time.
\end{example}
\begin{figure}[htbp]
	\centering
	\subfigure[The trajectory of $x_{1}$ component of the generated motion plan.]{\includegraphics[width=0.23\textwidth]{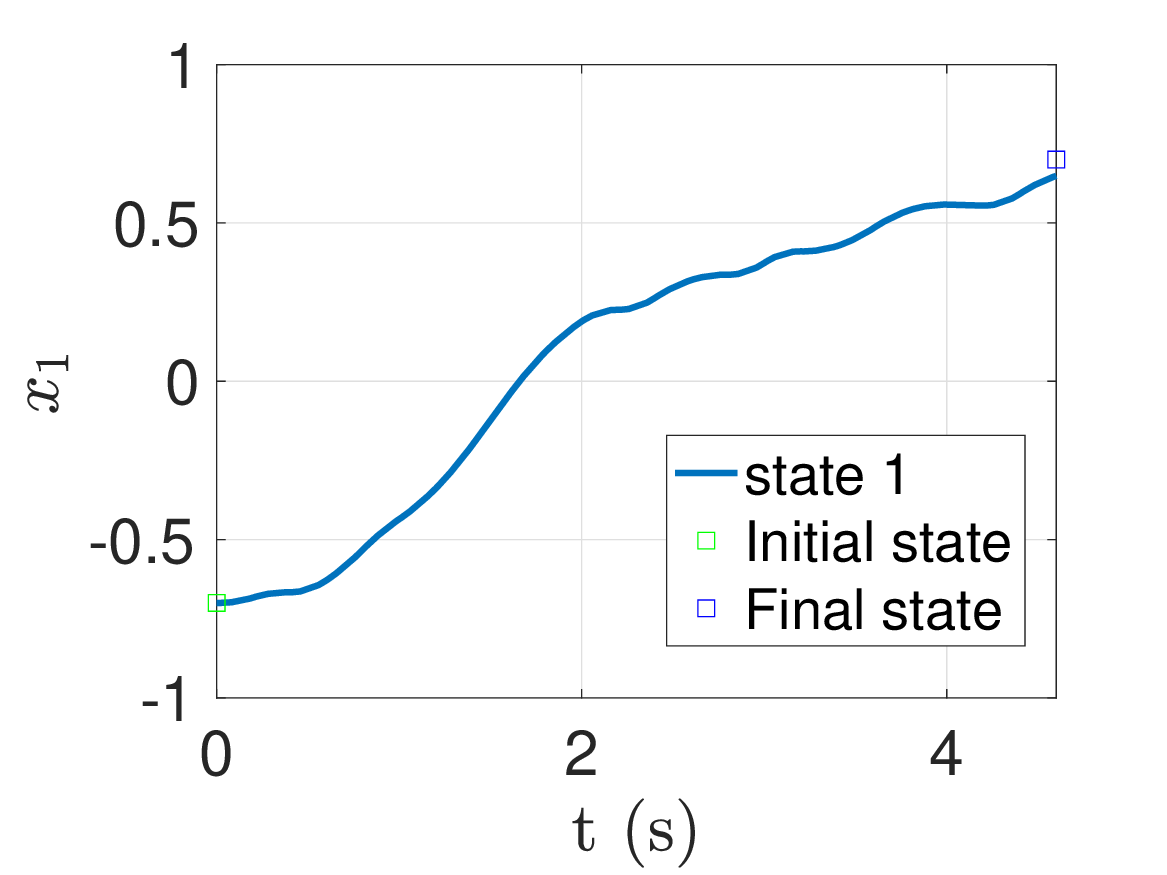}}
	\subfigure[The trajectory of $x_{2}$ component of the generated motion plan.]{\includegraphics[width=0.23\textwidth]{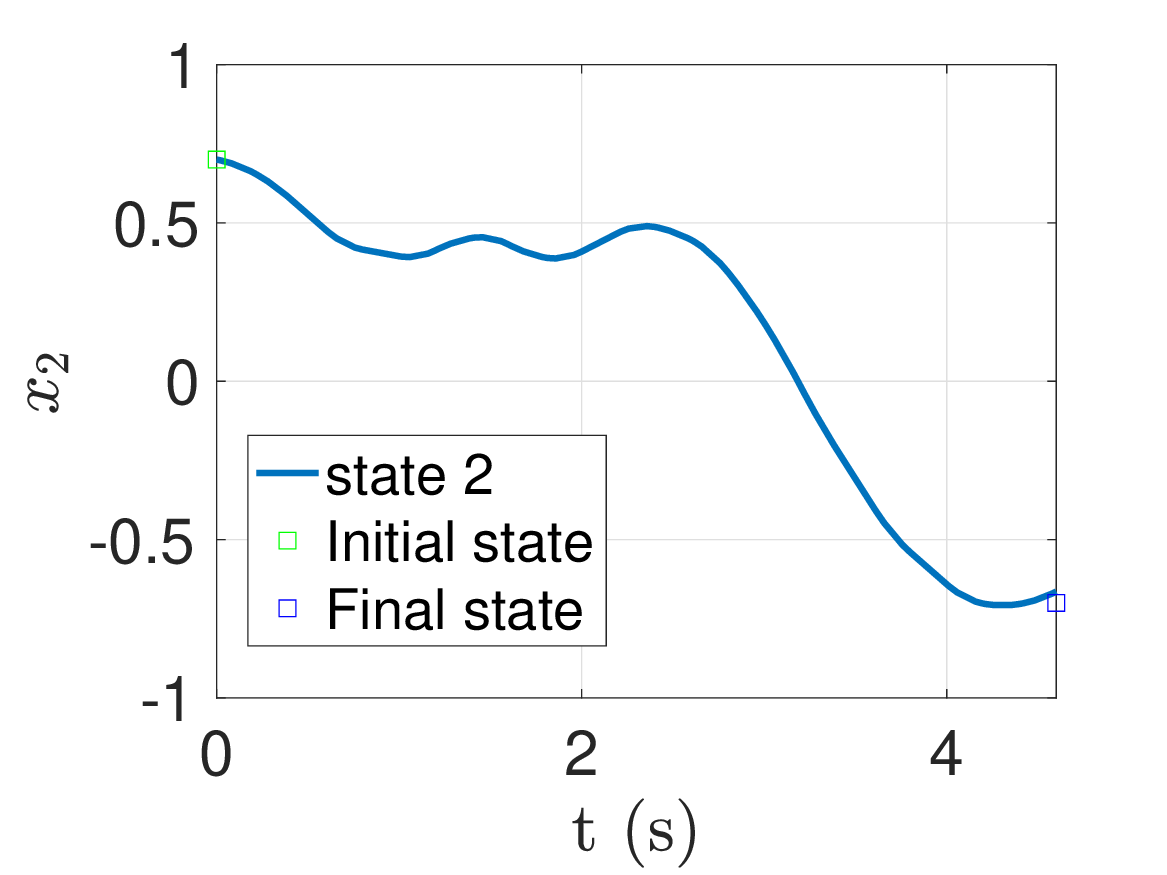}}
	\subfigure[The trajectory of $x_{3}$ component of the generated motion plan.]{\includegraphics[width=0.23\textwidth]{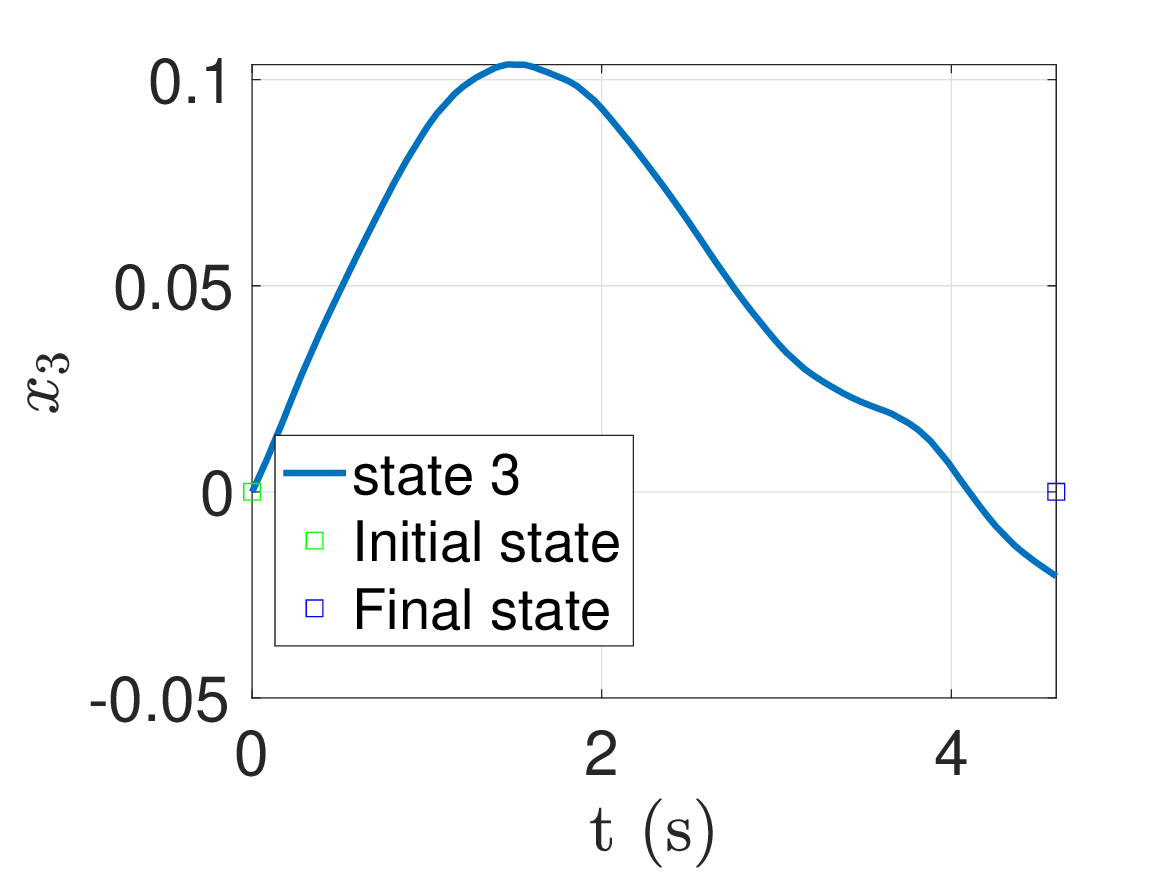}}
	\subfigure[The trajectory of $x_{4}$ component of the generated motion plan.]{\includegraphics[width=0.23\textwidth]{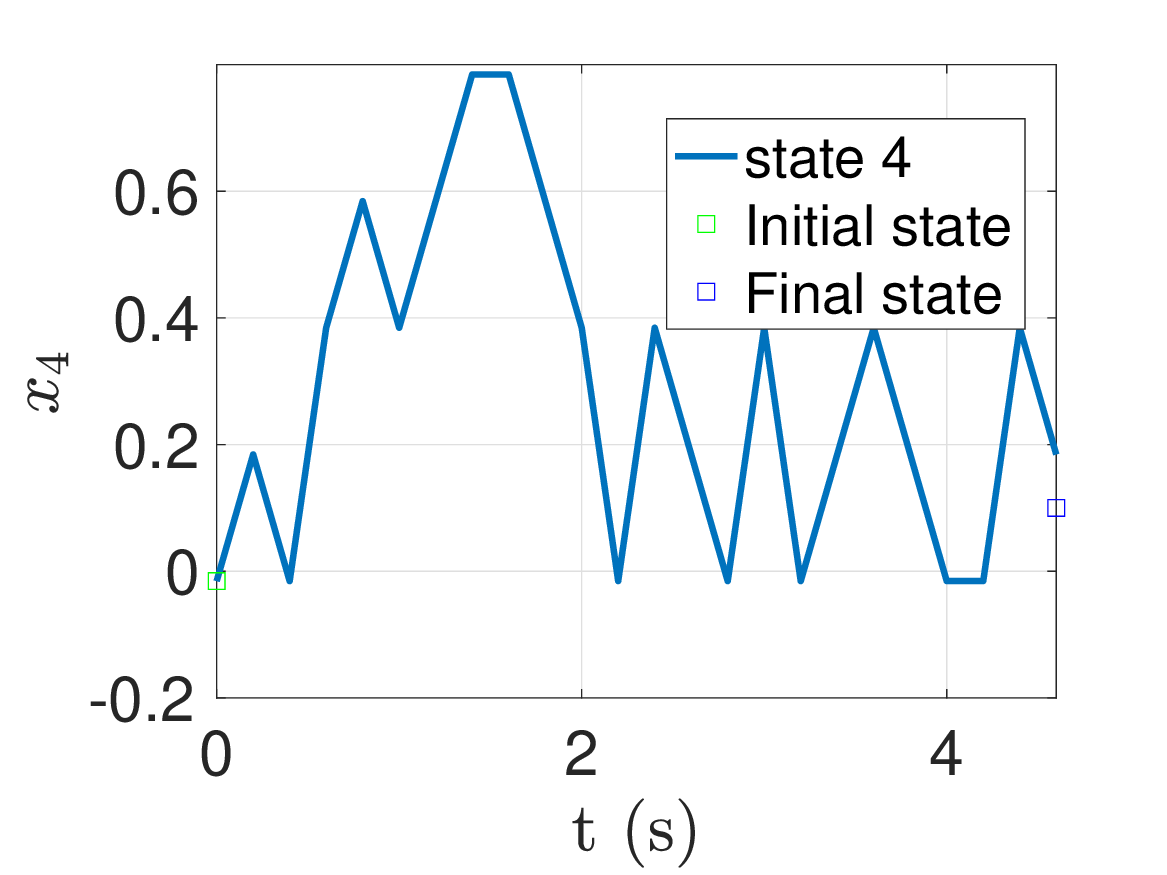}}
	\subfigure[The trajectory of $x_{5}$ component of the generated motion plan.]{\includegraphics[width=0.23\textwidth]{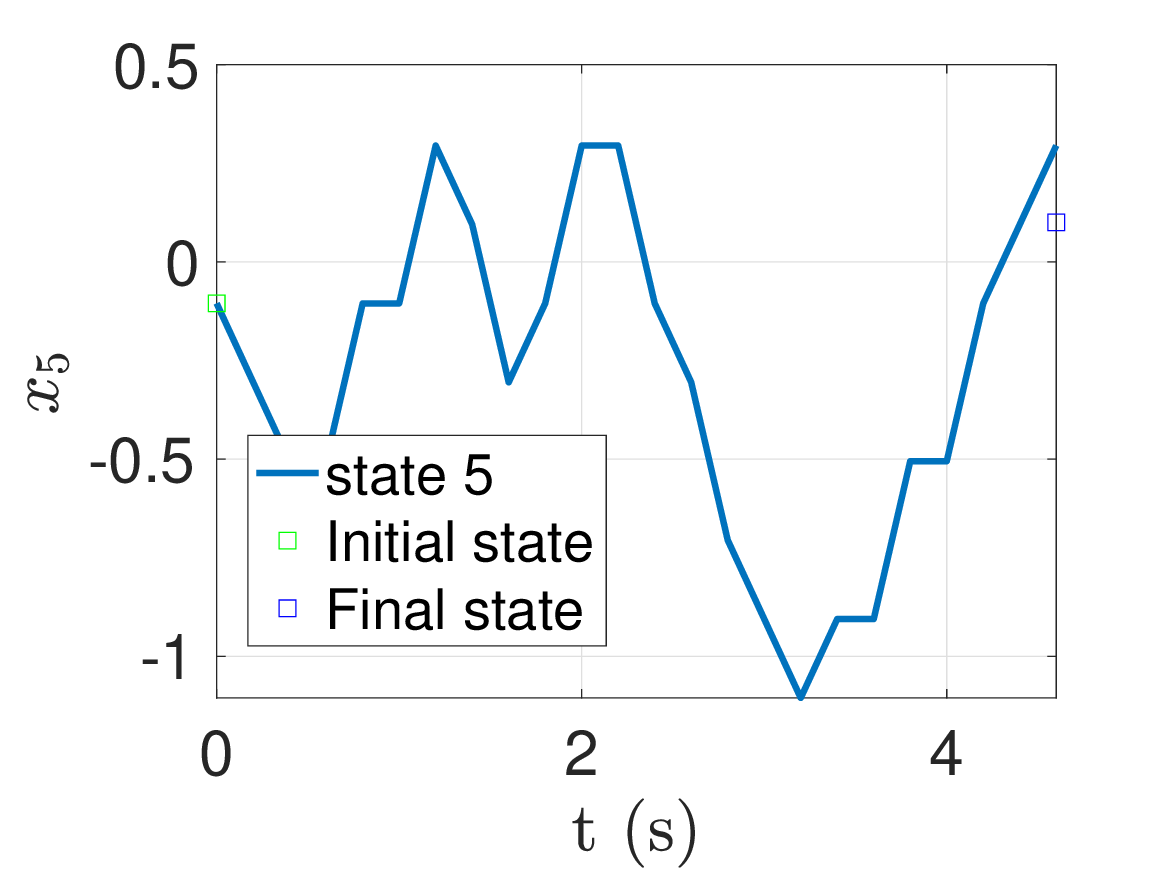}}
	\subfigure[The trajectory of $x_{6}$ component of the generated motion plan.]{\includegraphics[width=0.23\textwidth]{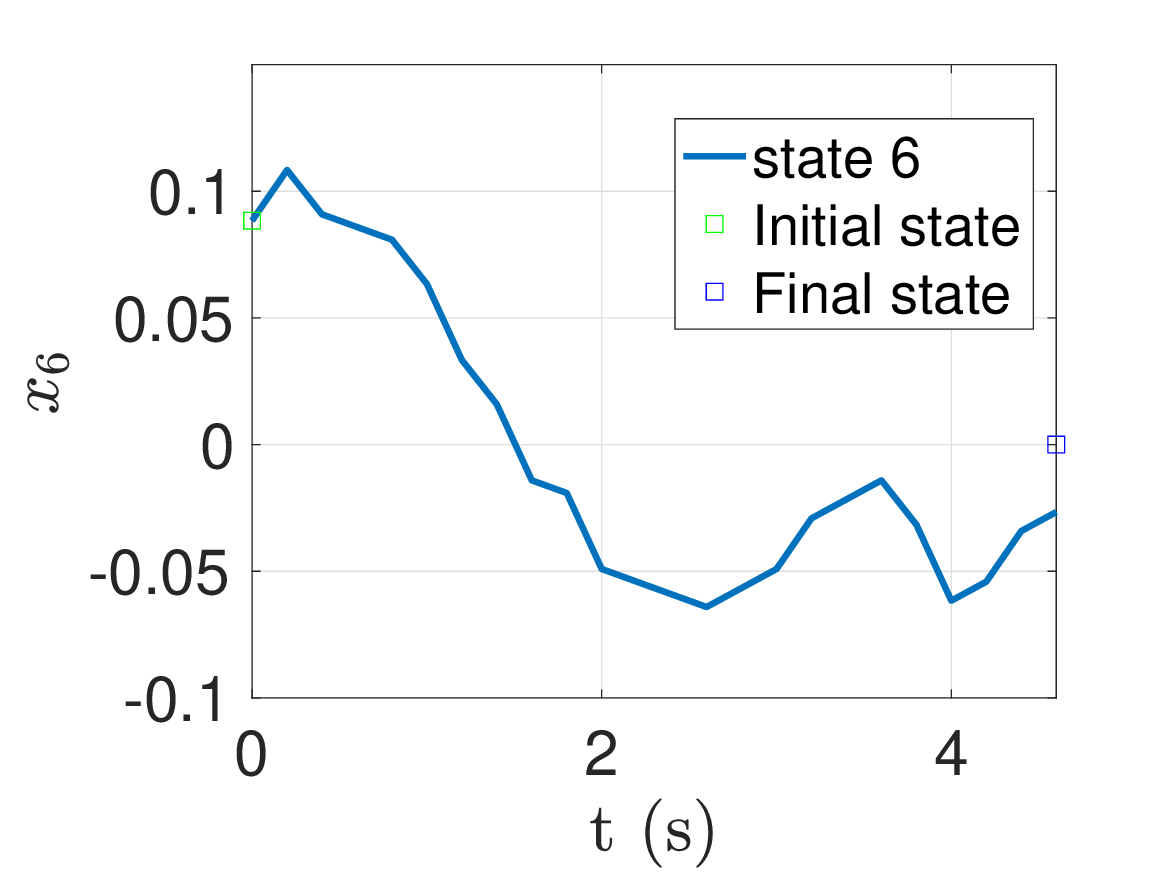}}
	\caption{The above shows the trajectories of each state components of the generated motion plan to the sample motion planning problem for biped system. In each figure in Figure \ref{fig:illustrationbiped}, green square denotes the corresponding component of $X_{0}$ and blue square denotes the corresponding component of $X_{f}$.}
	\label{fig:illustrationbiped}
\end{figure}

\ifbool{report}{}{
\begin{dci}
	The authors declared no potential conflicts of interest with respect to the research, authorship, and/or publication of this article.
\end{dci}

\begin{funding}
	The authors disclosed receipt of the following financial support for the research, authorship, and/or publication of this article: This research is supported by NSF [grants numbers CNS-2039054 and
	CNS-2111688], by AFOSR [grants numbers FA9550-19-1-0169, FA9550-20-
	1-0238, FA9550-23-1-0145, and FA9550-23-1-0313], by AFRL [grant
	numbers FA8651-22-1-0017 and FA8651-23-1-0004], by ARO [grant number.
	W911NF-20-1-0253], and by DoD [grant number W911NF-23-1-0158].
\end{funding}}

\theendnotes
\bibliographystyle{SageH}
\bibliography{references.bib}

\appendix
\section{Proof of Proposition \ref{theorem:concatenationsolution}}\label{section:concatenationsolution}
\subsection{Theoretical Tools to Prove Proposition \ref{theorem:concatenationsolution}}
\begin{lemma}\label{lemma:concatenationHTD}
	Given two functions, denoted $\psi_{1}$ and $\psi_{2}$ and defined on hybrid time domains, the domain of their concatenation $\psi = \psi_{1}|\psi_{2}$, is also a hybrid time domain.
\end{lemma}
\begin{proof}
	Given that $\dom \psi_{1}$ and $\dom \psi_{2}$ are hybrid time domains as per Definition \ref{definition:hybridtimedomain}, for each $(T_{1}, J_{1})\in \dom \psi_{1}$ and each $(T_{2}, J_{2})\in \dom \psi_{2}$, it follows that
	$$
	\dom \psi_{1} \cap ([0, T_{1}]\times \{0,1,..., J_{1}\}) = \cup_{j = 0}^{J_{1} }([t^{1}_{j}, t^{1}_{j + 1}], j)
	$$
	for some finite sequence of times $0=t^{1}_{0}\leq t^{1}_{1}\leq t^{1}_{2}\leq ...\leq t^{1}_{J_{1} + 1}  = T_{1}$ and 
	$$
	\dom \psi_{2} \cap ([0, T_{2}]\times \{0,1,..., J_{2}\}) = \cup_{j = 0}^{J_{2} }([t^{2}_{j}, t^{2}_{j + 1}], j)
	$$
	for some finite sequence of times $0=t^{2}_{0}\leq t^{2}_{1}\leq t^{2}_{2}\leq ...\leq t^{2}_{J_{2} + 1} = T_{2}$.
	
	Given that $\dom \psi = \dom \psi_{1} \cup (\dom \psi_{2} + \{(T, J)\})$ where $(T, J) = \max \dom \psi_{1}$, it follows that for each $(T', J')\in \dom \psi$,
	\begin{equation}\label{equation:concatenation_proof1}
		\begin{aligned}
		&\dom \psi \cap ([0, T']\times \{0,1,..., J'\}) \\
		&= \left(\dom \psi_{1} \cup (\dom \psi_{2} + \{(T, J)\})\right)\cap ([0, T']\times \{0,1,..., J'\}) \\
		&= ( \dom \psi_{1}  \cap ([0, T']\times \{0,1,..., J'\})) \\ &\cup
		(\dom \psi_{2} + \{(T, J)\})\cap ([0, T']\times \{0,1,..., J'\}) \\
		\end{aligned}
	\end{equation}
	
	Since $(T', J')\in \dom \psi$ and $\dom \psi = \dom \psi_{1} \cup (\dom \psi_{2} + \{(T, J)\})$, $(T', J')$ satisfies one of the following two conditions:
	\begin{enumerate}
		\item $(T', J') \in \dom \psi_{1}$ such that $T'\leq T$ and $J'\leq J$;
		\item $(T', J') \in  (\dom \psi_{2} + \{(T, J)\})$ such that $T'\geq T$ and $J'\geq J$.
	\end{enumerate}
	
	For the first item, (\ref{equation:concatenation_proof1}) can be written as:
	$$
	\begin{aligned}
	&\dom \psi \cap ([0, T']\times \{0,1,..., J'\})\\
	&= ( \dom \psi_{1}  \cap ([0, T']\times \{0,1,..., J'\}))\\& \cup (\dom \psi_{2} + \{(T, J)\})\cap ([0, T']\times \{0,1,..., J'\}) \\
	& = \cup_{j = 0}^{J'}([t^{1}_{j}, t^{1}_{j + 1}], j) \cup \emptyset = \cup_{j = 0}^{J'}([t^{1}_{j}, t^{1}_{j + 1}], j)
	\end{aligned}
	$$
	for some finite sequence of times $0=t^{1}_{0}\leq t^{1}_{1}\leq t^{1}_{2}\leq ...\leq t^{1}_{J'+ 1} = T'$. 
	
	For the second item, (\ref{equation:concatenation_proof1}) can be written as
	\begin{equation}
	\begin{aligned}
	&\dom \psi \cap ([0, T']\times \{0,1,..., J'\})\\
	&= ( \dom \psi_{1}  \cap ([0, T']\times \{0,1,..., J'\})) \\
	&\cup (\dom \psi_{2} + \{(T, J)\})\cap ([0, T']\times \{0,1,..., J'\}) \\
	& = ( \dom \psi_{1}  \cap ([0, T]\times \{0,1,..., J\})) \\
	&\cup (\dom \psi_{2} + \{(T, J)\})\cap ([T, T']\times \{J,J + 1,..., J'\}) \\
	& = ( \dom \psi_{1}  \cap ([0, T]\times \{0,1,..., J\})) \\
	&\cup (\dom \psi_{2}\cap ([0, T' - T]\times \{0,1,..., J' - J\}) + \{(T, J)\})\\
	& = \cup_{j = 0}^{J}([t^{1}_{j}, t^{1}_{j + 1}], j) \cup (\cup_{j = 0}^{J' - J}([t^{2}_{j}, t^{2}_{j + 1}], j) + \{(T, J)\})\\
	& = \cup_{j = 0}^{J}([t^{1}_{j}, t^{1}_{j + 1}], j) \cup (\cup_{j = 0}^{J' - J}([t^{2}_{j} + T, t^{2}_{j + 1} + T], j + J))\\
	\end{aligned}
	\end{equation}
	for some finite sequence of times $0=t^{1}_{0}\leq t^{1}_{1}\leq t^{1}_{2}\leq ...\leq t^{1}_{J + 1} = T$ and $T=t^{2}_{0} + T\leq t^{2}_{1} + T\leq t^{2}_{2}+ T\leq ...\leq t^{2}_{J' - J + 1} + T= T' $. Hence, a finite sequence $\{t_{i}\}_{i = 0}^{J'}$ can be constructed as follows such that $0=t_{0}\leq t_{1}\leq t_{2}\leq ...\leq t_{J'} = T'$.
	$$
	t_{i} = \left\{\begin{aligned}
	&t^{1}_{i} &i\leq J\\
	&t^{2}_{i - J} + T &J + 1\leq i \leq J'
	\end{aligned}\right.
	$$
	
	Therefore, for each $(T', J')\in \dom \psi$, there always exists a finite sequence of $\{t_{i}\}_{i = 0}^{J'}$ such that 
	$$
	\dom \psi \cap ([0, T']\times \{0,1,..., J'\}) =  \cup_{j = 0}^{J'}([t_{j}, t_{j + 1}], j). 
	$$ Consequently, this establishes  $\dom \psi$ as a hybrid time domain.  \ifbool{report}{}{\qed}
\end{proof}
\begin{definition}[Lebesgue measurable set and Lebesgue measurable function]\label{definition:lebesguemeasurable}
	A set $S\subset \myreals$ is said to be \emph{Lebesgue measurable}  if it has positive Lebesgue measure $\mu(S)$, where $\mu(S) = \sum_{k} (b_{k} - a_{k})$ and $(a_{k}, b_{k})$'s are all of the disjoint open intervals in $S$. Then $s: S\to \myreals[n_{s}]$ is said to be  a \emph{Lebesgue measurable function} if for every open set $\mathcal{U}\subset \myreals[n_{s}]$ then set $\{r\in S: s(r) \in \mathcal{U}\}$ is Lebesgue measurable.
\end{definition}
\begin{lemma}\label{lemma:concatenation_input_Lebesguemeasurable}
	Given two hybrid inputs, denoted $\myinputt_{1}$ and $\myinputt_{2}$, their concatenation $\myinputt = \myinputt_{1}|\myinputt_{2}$, is such that for each $j\in \nnumbers$, $t\mapsto \myinputt(t,j)$ is Lebesgue measurable on the interval $I^{j}_{\myinputt}:=\{t:(t, j)\in \dom \myinputt\}$.
\end{lemma}
\begin{proof}
	We show that for each $j\in \mathbb{N}$, the function $t\mapsto \myinputt(t, j)$ satisfies conditions outlined in Definition \ref{definition:lebesguemeasurable} over the interval $I^{j}_{\myinputt}:=\{t:(t, j)\in \dom \myinputt\}$. 
	
	In line with Definition \ref{definition:concatenation}, we have $\dom \psi = \dom \psi_{1} \cup (\dom \psi_{2} + \{(T, J)\})$ where $(T, J) = \max \dom \psi_{1}$. For each $j\in \nnumbers$, $j$ and $J$ satisfy one of the following three conditions:
	\begin{enumerate}[label=\Roman*]
		\item $j < J$;
		\item $j > J$;
		\item $j = J$.
	\end{enumerate}
	In each case, we proceed to demonstrate that the function $t\mapsto \myinputt(t, j)$ is Lebesgue measurable.
	\begin{enumerate}[label=\Roman*]
		\item In the case where $j < J$, given that $\myinputt_{1}$ is a hybrid input, the function $t\mapsto \myinputt_{1}(t, j)$ is Lebesgue measurable over the interval $I^{j}_{\myinputt_{1}}:=\{t:(t, j)\in \dom \myinputt_{1}\}$.  According to Definition \ref{definition:concatenation}, it follows that $\myinputt(t, j) = \myinputt_{1}(t, j)$ holds for all $(t, j)\in \dom \myinputt_{1}$ where $j < J$. Consequently, the function $t\mapsto \myinputt(t, j) = \myinputt_{1}(t, j) $ is Lebesgue measurable on the interval $I^{j}_{\psi} = I^{j}_{\psi_{1}}$ when $j < J$. 
		\item In the case where  $j > J$, given that $\myinputt_{2}$ is a hybrid input, the function $t\mapsto \myinputt_{2}(t, j)$ is Lebesgue measurable on the interval $I^{j}_{\psi_{2}}$. 
		According to Definition \ref{definition:concatenation}, it is established that $\myinputt(t, j) = \myinputt_{2}(t - T, j - J)$ for each $(t, j)\in \dom \myinputt_{2} + \{(T, J)\}$ where $j > J$. Therefore, the function $t\mapsto \myinputt(t, j) = \myinputt_{2}(t - T, j - J)$ is Lebesgue measurable on the interval $I^{j}_{\psi} = I^{j - J}_{\psi_{2}} + \{T\}$ when $j > J$.
		\item In the case where $j = J$, according to Definition \ref{definition:concatenation}, it follows that
		$$
		\myinputt(t, J) = \left\{
		\begin{aligned}
		&\myinputt_{1}(t, J) &t\in I^{J}_{\psi_{1}}\backslash \{T\}\\
		&\myinputt_{2}(t - T,0) &t \in I^{0}_{\psi_{2}} + \{T\}.\\
		\end{aligned}
		\right.
		$$
		Note that $t\mapsto \myinputt_{1}(t, J)$ is Lebesgue measurable over the interval $I_{\myinputt_{1}}^{J}$ and $t\mapsto \myinputt_{2}(t, 0)$ is Lebesgue measurable on the interval $I_{\myinputt_{2}}^{0}$. As per Definition \ref{definition:lebesguemeasurable}, it is established that
		\begin{enumerate}
			\item for each open set $\mathcal{U}_{1}\subset \{\myinputt_{1}(t, J)\in \myreals[m]: t \in I_{\myinputt_{1}}^{J}\}$, the set $\{t\in I_{\myinputt_{1}}^{J}: \myinputt_{1}(t, J) \in \mathcal{U}_{1}\}$ is Lebesgue measurable;
			\item for each open set $\mathcal{U}_{2}\subset \{\myinputt_{2}(t, 0)\in \myreals[m]: t \in I_{\myinputt_{2}}^{0}\}$, the set $\{t\in I_{\myinputt_{2}}^{0}: \myinputt_{2}(t, 0) \in \mathcal{U}_{2}\}$ is Lebesgue measurable. 
		\end{enumerate} 
		Consequently, for each open set $\mathcal{U}\subset \{\myinputt(t, J)\in \myreals[m]: t \in I_{\myinputt}^{J}\}$, given that
		$$
		\begin{aligned}
		&\mathcal{U}\subset\{\myinputt(t, J)\in \myreals[m]: t \in I_{\myinputt}^{J}\} =  \{\myinputt_{1}(t, J)\in \myreals[m]: \\
		&t \in I_{\myinputt_{1}}^{J}\}\cup  \{\myinputt_{2}(t, 0)\in \myreals[m]: t \in I_{\myinputt_{2}}^{0}\},
		\end{aligned}$$ therefore, at least one of $\widetilde{\mathcal{U}}_{1} := \mathcal{U}\cap \{\myinputt_{1}(t, J)\in \myreals[m]: t \in I_{\myinputt_{1}}^{J}\}$ and $\widetilde{\mathcal{U}}_{2} := \mathcal{U}\cap \{\myinputt_{2}(t, 0)\in \myreals[m]: t \in I_{\myinputt_{2}}^{0}\}$ is an open set. According to (a) and (b) above, it follows that at least one of the sets $\{t\in I_{\myinputt_{1}}^{J}: \myinputt_{1}(t, J) \in \widetilde{\mathcal{U}}_{1}\}$ and $\{t\in I_{\myinputt_{2}}^{0}: \myinputt_{2}(t, 0) \in \widetilde{\mathcal{U}}_{2}\}$ is Lebesgue measurable.
		Consequently, the set $\{t\in I_{\myinputt}^{J}: \myinputt(t, J) \in \mathcal{U}\} = \{t\in I_{\myinputt_{1}}^{J}: \myinputt_{1}(t, J) \in \widetilde{\mathcal{U}}_{1}\} \cup (\{t\in I_{\myinputt_{2}}^{0}: \myinputt_{2}(t, 0) \in \widetilde{\mathcal{U}}_{2}\} + \{T\})$ is Lebesgue measurable and, hence, the function $t\mapsto \myinputt(t, J)$ is Lebesgue measurable.
	\end{enumerate}
	Therefore, $t\mapsto \myinputt(t, j)$ is Lebesgue measurable on the interval $I_{\myinputt}^{j}$ for all $j\in \mathbb{N}$. \ifbool{report}{}{\qed}
\end{proof}
\begin{definition}[Locally essentially bounded function]\label{definition:locallybounded}
	A function $s: S\to \myreals[n_{s}]$ is said to be \emph{locally essentially bounded} if for each $r\in S$ there exists an open neighborhood $\mathcal{U}$ of $r$ such that $s$ is bounded almost everywhere\footnote{A property is said to hold almost everywhere or for almost all points in a set $S$ if the set of elements of $S$ at which the property does not hold has zero Lebesgue measure.} on $\mathcal{U}$; i.e., there exists $c\geq 0$ such that $|s(r)|\leq c$ for almost all $r\in \mathcal{U}\cap S$.
\end{definition}
\begin{lemma}\label{lemma:concatenation_input_locallyessentiallybounded}
	Given two hybrid inputs, denoted $\myinputt_{1}$ and $\myinputt_{2}$, their concatenation $\myinputt = \myinputt_{1}|\myinputt_{2}$, is such that for each $j\in \nnumbers$, $t\mapsto \myinputt(t,j)$ is locally essentially bounded on the interval $I^{j}_{\myinputt}:=\{t:(t, j)\in \dom \myinputt\}$.
\end{lemma}
\begin{proof}  
	We show that for all $j\in \mathbb{N}$, the function $t\mapsto \myinputt(t, j)$ satisfies conditions in Definition \ref{definition:locallybounded} on the interval $I^{j}_{\myinputt}:=\{t:(t, j)\in \dom \myinputt\}$. 
	
	According to Definition \ref{definition:concatenation}, $\dom \psi = \dom \myinputt_{1} \cup (\dom \myinputt_{2} + \{(T, J)\})$ where $(T, J) = \max \dom \myinputt_{1}$. For each $j\in \nnumbers$, $j$ and $J$ satisfy one of the following three conditions:
	\begin{enumerate}[label=\Roman*]
		\item $j < J$;
		\item $j > J$;
		\item $j = J$.
	\end{enumerate}
	In each case, we proceed to demonstrate that the function $t\mapsto \myinputt(t,j)$ is locally essentially bounded.
	\begin{enumerate}[label=\Roman*]
		\item In the case where $j < J$, given that $\myinputt_{1}$ is a hybrid input, the function $t\mapsto \myinputt_{1}(t, j)$ is locally essentially bounded on the interval $I^{j}_{\myinputt_{1}}:=\{t:(t, j)\in \dom \myinputt_{1}\}$.  As per Definition \ref{definition:concatenation}, it follows that $\myinputt(t, j) = \myinputt_{1}(t, j)$ for all $(t, j)\in \dom \myinputt_{1}$ where $j < J$. Consequently, the function $t\mapsto \myinputt(t, j) = \myinputt_{1}(t, j) $ is locally essentially bounded on the interval $I^{j}_{\psi} = I^{j}_{\psi_{1}}$ where $j < J$. 
		\item In the case where $j > J$, given that $\myinputt_{2}$ is a hybrid input, the function $t\mapsto \myinputt_{2}(t, j)$ is locally essentially bounded on the interval $I^{j}_{\myinputt_{2}}$. 
		As per Definition \ref{definition:concatenation}, it follows that $\myinputt(t, j) = \myinputt_{2}(t - T, j - J)$ for all $(t, j)\in (\dom \myinputt_{2} + \{(T, J)\})$ where $j > J$. Consequently, the function $t\mapsto \myinputt(t, j) = \myinputt_{2}(t - T, j - J)$ is locally essentially bounded on the interval $I^{j}_{\myinputt} = I^{j - J}_{\myinputt_{2}} + \{T\}$ where $j > J$.
		\item In the case where $j = J$, according to Definition \ref{definition:concatenation}, it follows that
		$$
		\myinputt(t, J) = \left\{
		\begin{aligned}
		&\myinputt_{1}(t, J) &t\in I^{J}_{\myinputt_{1}}\backslash \{T\}\\
		&\myinputt_{2}(t - T,0) &t \in I^{0}_{\myinputt_{2}} + \{T\}.\\
		\end{aligned}
		\right.
		$$
		
		Given that $t\mapsto \myinputt_{1}(t, j)$ is locally bounded on the interval $I^{J}_{\myinputt_{1}}$ and $t\mapsto \myinputt_{2}(t, 0)$ is locally bounded on the interval $I_{\myinputt_{2}}^{0}$, according to Definition \ref{definition:locallybounded}, it follows that
		\begin{enumerate}
			\item for each $r\in I_{\myinputt_{1}}^{J}$, there exists a neighborhood $\mathcal{U}_{1}$ of $r$ such that there exists $c_{1}\geq 0$ such that $|\myinputt_{1}(r)| \leq c_{1}$ for almost all $r\in \mathcal{U}_{1}\cap I_{\myinputt_{1}}^{J}$;
			\item for each $r\in I_{\myinputt_{2}}^{0}$, there exists a neighborhood $\mathcal{U}_{2}$ of $r$ such that there exists $c_{2}\geq 0$ such that $|\myinputt_{2}(r)| \leq c_{2}$ for almost all $r\in \mathcal{U}_{1}\cap I_{\myinputt_{2}}^{0}$.
		\end{enumerate} 
		Then, for each $r\in I_{\myinputt}^{J} = I_{\myinputt_{1}}^{J} \cup (I_{\myinputt_{2}}^{0} + \{T\})$,  the following two cases are considered:
		\begin{enumerate}
			\item If $r \in I_{\myinputt_{1}}^{J}$, then there exists a neighborhood $\mathcal{U}_{1}$ of $r$ such that there exists $c_{1}\geq 0$ such that $|\myinputt(r)| \leq c_{1}$ for almost all $r\in \mathcal{U}_{1}\cap I_{\myinputt}^{J}$.
			\item If $r\in I_{\myinputt_{2}}^{0} + \{T\}$, then there exists a neighborhood $\mathcal{U}_{2}$ of $r$ such that there exists $c_{2}\geq 0$ such that $|\myinputt_{2}(r)| \leq c_{2}$ for almost all $r\in \mathcal{U}_{1}\cap (I_{\myinputt_{2}}^{0} + \{T\})$.
		\end{enumerate}
		Therefore, the function $t\mapsto \myinputt(t, J)$ is locally essentially bounded on the interval $I_{\myinputt}^{J}$.
	\end{enumerate}
	Therefore, $t\mapsto \myinputt(t, j)$ is locally bounded on the interval $I_{\myinputt}^{j}$ for all $j\in \mathbb{N}$. \ifbool{report}{}{\qed}
\end{proof}
\begin{lemma}\label{lemma:concatenation_hybridinput}
	Given two hybrid inputs, denoted $\myinputt_{1}$ and $\myinputt_{2}$, then their concatenation $\myinputt = \myinputt_{1}|\myinputt_{2}$ is also a hybrid input.
\end{lemma}
\begin{proof}
	In this proof, we show that $\myinputt = \myinputt_{1}|\myinputt_{2}$ satisfies the conditions in Definition \ref{definition:hybridinput} as follows:
	\begin{enumerate}
		\item $\dom \myinputt$ is hybrid time domain;
		\item for each $j\in \nnumbers$, $t\to \myinputt(t, j)$ is Lebesgue measurable on the interval $I^{j}_{\myinputt}:=\{t:(t, j)\in \dom \myinputt\}$;
		\item for each $j\in \nnumbers$, $t\to \myinputt(t, j)$ is locally essentially bounded on the interval $I^{j}_{\myinputt}:=\{t:(t, j)\in \dom \myinputt\}$. 
	\end{enumerate}

	Given that $\myinputt_{1}$ and $\myinputt_{2}$ are hybrid inputs, therefore, Lemma \ref{lemma:concatenationHTD} ensures that $\dom \myinputt$ is hybrid time domain. Lemma \ref{lemma:concatenation_input_Lebesguemeasurable} guarantees that $\myinputt$ is Lebesgue measurable on the interval $I^{j}_{\myinputt}$. Lemma \ref{lemma:concatenation_input_locallyessentiallybounded} establishes that $\myinputt$ is locally essentially bounded on the interval $I^{j}_{\myinputt}$. Consequently, it follows that $\myinputt$ is a hybrid input.\ifbool{report}{}{\qed}
\end{proof}
\begin{definition}[Absolutely continuous function and locally absolutely continuous function]\label{definition:locallycontinuous}\cite[Definition A.20]{goebel2009hybrid}
	A function $s: [a, b]\to \myreals[n]$ is said to be \emph{absolutely continuous} if for each $\varepsilon > 0$ there exists $\delta > 0$ such that for each countable collection of disjoint subintervals $[a_{k}, b_{k}]$ of $[a, b]$ such that $\sum_{k}(b_{k} - a_{k}) \leq \delta$, it follows that $\sum_{k}|s(b_{k}) - s(a_{k})| \leq \varepsilon$. A function is said to be locally absolutely continuous if $r\mapsto s(r)$ is absolutely continuous on each compact subinterval of $\mypreals$.
\end{definition}
\begin{lemma}\label{lemma:concatenation_state_locallyabsolutelycontinuous}
	Given two hybrid arcs, denoted $\mystatet_{1}$ and $\mystatet_{2}$, such that $\mystatet_{1}(T, J) = \mystatet_{2}(0,0)$, where $(T, J) = \max \dom \mystatet_{1}$, their concatenation $\mystatet = \mystatet_{1}|\mystatet_{2}$, is such that for each $j\in \nnumbers$, $t\mapsto \mystatet(t,j)$ is locally absolutely continuous on the interval $I^{j}_{\mystatet}:=\{t:(t, j)\in \dom \mystatet\}$.
\end{lemma}
\begin{proof}
	We show that for all $j\in \mathbb{N}$, the function $t\mapsto \mystatet(t, j)$ satisfies the conditions in Definition \ref{definition:locallycontinuous}.
	
	According to Definition \ref{definition:concatenation}, we have $\dom \mystatet = \dom \mystatet_{1} \cup (\dom \mystatet_{2} +\{(T, J)\})$ where $(T, J) = \max \dom \mystatet_{1}$.
	It follows that
	$$
		I^{j}_{\mystatet} = \left\{ \begin{aligned}
		&I^{j}_{\mystatet_{1}} &\text{ if }j< J\\
		&I^{J}_{\mystatet_{1}} \cup (I^{0}_{\mystatet_{2}} +  \{T\}) &\text{ if }j = J\\
		&I^{j - J}_{\mystatet_{2}} + T&\text{ if }j > J.
		\end{aligned}\right.
	$$
	
	For each $j\in \nnumbers$, $j$ and $J$ satisfy one of the following three conditions:
	\begin{enumerate}[label=\Roman*]
		\item $j < J$;
		\item $j > J$;
		\item $j = J$.
	\end{enumerate}
	In each case, we proceed to demonstrate that the function $t\mapsto \mystatet(t,j)$ is absolutely continuous.
	\begin{enumerate}[label=\Roman*]
		\item In the case where $j < J$, given that $\mystatet_{1}$ is a hybrid arc, according to Definition \ref{definition:hybridarc}, the function $t\mapsto \mystatet_{1}(t, j)$ is locally absolutely continuous on the interval $I_{\mystatet_{1}}^{j} = I_{\mystatet}^{j}$ where $j < J$. According to Definition \ref{definition:concatenation}, it follows that $\mystatet(t, j) = \mystatet_{1}(t, j)$ for all $t\in I_{\mystatet_{1}}^{j} = I^{j}_{\mystatet}$ where $j < J$. Therefore, the function $t\mapsto \mystatet(t, j) =\mystatet_{1}(t, j)$ is locally absolutely continuous on the interval $I_{\mystatet}^{j}$ where $j < J$.
		\item In the case where $j> J$, given that $\mystatet_{2}$ is a hybrid arc, according to Definition \ref{definition:hybridarc}, the function $t\mapsto \mystatet_{2}(t - T, j - J)$ is locally absolutely continuous on the interval $I_{\mystatet_{2}}^{j - J} + \{T\}= I^{j}_{\mystatet}$ where $j > J$. According to Definition \ref{definition:concatenation}, it follows that $\mystatet(t, j) = \mystatet_{2}(t - T, j - J)$ for all $t\in I_{\mystatet_{2}}^{j - J} + \{T\} = I^{j}_{\mystatet}$ where $j > J$. Therefore, the function $t\mapsto  \mystatet_{2}(t - T, j - J) = \mystatet(t, j)$ is locally absolutely continuous on the interval $I_{\mystatet}^{j} = I_{\mystatet_{2}}^{j - J} + \{T\}$ where $j > J$.
		\item In the case where $j = J$, given that $I_{\mystatet}^{j} = I_{\mystatet_{1}}^{J}\cup (I_{\mystatet_{2}}^{0} + \{T\})$ and that the function $t\mapsto \mystatet(t, J)$ is locally absolutely continuous on the intervals $I_{\mystatet_{1}}^{J}$ and $(I_{\mystatet_{2}}^{0} + \{T\})$, respectively, therefore, it follows the following:
		\begin{enumerate}
			\item for each $\varepsilon_{1} > 0$ there exists $\delta_{1} > 0$ such that for each countable collection of disjoint subintervals $[a^{1}_{k_{1}}, b^{1}_{k_{1}}]$ of $I_{\mystatet_{1}}^{J}$ such that $\sum_{k_{1}}(b^{1}_{k_{1}} - a^{1}_{k_{1}}) \leq \delta_{1}$, it follows that $\sum_{k_{1}}|\mystatet(b^{1}_{k_{1}}) - \mystatet(a^{1}_{k_{1}})| \leq \varepsilon_{1}$;
			\item for each $\varepsilon_{2} > 0$ there exists $\delta_{2} > 0$ such that for each countable collection of disjoint subintervals $[a^{2}_{k_{2}}, b^{2}_{k_{2}}]$ of $I_{\mystatet_{2}}^{0} + \{T\}$ such that $\sum_{k_{2}}(b^{2}_{k_{2}} - a^{2}_{k_{2}}) \leq \delta_{2}$, it follows that $\sum_{k_{2}}|\mystatet(b^{2}_{k_{2}}) - \mystatet(a^{2}_{k_{2}})| \leq \varepsilon_{2}$.
		\end{enumerate}
		Then for each $\varepsilon > 0$, assume that $\delta > 0$ exists such that for each countable collection of disjoint subintervals $[a_{k}, b_{k}]$ of $I_{\mystatet_{1}}^{J}\cup (I_{\mystatet_{2}}^{0} + \{T\})$ such that $\sum_{k}(b_{k} - a_{k}) \leq \delta$, it follows that $\sum_{k}|\mystatet(b^{2}_{k}) - \mystatet(a^{2}_{k})| \leq \varepsilon$. 
		
		The collections of intervals $\{[a_{k}, b_{k}]\}_{k}$ intersecting with the intervals $I_{\mystatet_{1}}^{J}$ and $(I_{\mystatet_{2}}^{0} + {T})$ are defined as follows:
		\begin{enumerate}[label=\arabic*]
			\item $\{[a^{1}_{k_{1}}, b^{1}_{k_{1}}]\}_{k_{1}} := \{[a_{k}, b_{k}]\}_{k}\cap I_{\mystatet_{1}}^{J}$;
			\item $\{[a^{2}_{k_{2}}, b^{2}_{k_{2}}]\}_{k_{2}} := \{[a_{k}, b_{k}]\}_{k}\cap (I_{\mystatet_{2}}^{0} + \{T\})$.
		\end{enumerate}
	Subsequently, define $\varepsilon_{1} := \sum_{k_{1}}|\mystatet(b^{1}_{k_{1}}) - \mystatet(a^{1}_{k_{1}})|$ and $\varepsilon_{2} := \sum_{k_{2}}|\mystatet(b^{2}_{k_{2}}) - \mystatet(a^{2}_{k_{2}})|$. Given that 
	$$\begin{aligned}
	\sum_{k}|\mystatet(b^{2}_{k}) - \mystatet(a^{2}_{k})| 
	&=  \sum_{k_{1}}|\mystatet(b^{1}_{k_{1}}) - \mystatet(a^{1}_{k_{1}})| 
	\\
	&+ \sum_{k_{2}}|\mystatet(b^{2}_{k_{2}}) - \mystatet(a^{2}_{k_{2}})|\\
	 &+ |\phi_{2}(0, 0) - \phi_{1}(T, J)|, 
	\end{aligned}$$
	and that $\phi_{2}(0, 0) = \phi_{1}(T, J)$, it follows that
	$$
	\sum_{k}|\mystatet(b^{2}_{k}) - \mystatet(a^{2}_{k})| = \varepsilon_{1} + \varepsilon_{2} + 0.
	$$
	According to (a) and (b) above, for each $\varepsilon_{1} > 0$ and each $\varepsilon_{2} > 0$, there exists $\delta_{1} > 0$ and $\delta_{2} > 0$ such that $\sum_{k_{1}}(b^{1}_{k_{1}} - a^{1}_{k_{1}}) \leq \delta_{1}$ and $\sum_{k_{2}}(b^{2}_{k_{2}} - a^{2}_{k_{2}}) \leq \delta_{2}$ hold, respectively. Therefore, given that 
	$\sum_{k}(b_{k} - a_{k}) =  \sum_{k_{1}}(b^{1}_{k_{1}} - a^{1}_{k_{1}}) + \sum_{k_{2}}(b^{2}_{k_{2}} - a^{2}_{k_{2}}) \leq \delta_{1} + \delta_{2} = \delta$ and the existence of $\delta_{1}$ and $\delta_{2}$, the existence of $\delta = \delta_{1} + \delta_{2}$ is established.  
	\end{enumerate}
	Therefore, the function $t\mapsto \mystatet(t, j)$ is locally absolutely continuous on the interval $I_{\mystatet}^{j}$ for all $j\in \mathbb{N}$.\ifbool{report}{}{\qed}
\end{proof}
\begin{lemma}\label{lemma:concatenation_hybridarc}
	Given two hybrid arcs, denoted $\mystatet_{1}$ and $\mystatet_{2}$, such that $\mystatet_{1}(T, J) = \mystatet_{2}(0,0)$, where $(T, J) = \max \dom \mystatet_{1}$, then their concatenation $\mystatet = \mystatet_{1}|\mystatet_{2}$ is also a hybrid arc.
\end{lemma}
\begin{proof}
	We show that $\mystatet = \mystatet_{1}|\mystatet_{2}$ satisfies the conditions in Definition \ref{definition:hybridarc} as follows:
	\begin{enumerate}
		\item We show that $\dom \mystatet$ is hybrid time domain;
		\item We show that, for each $j\in \nnumbers$, $t\mapsto \mystatet(t,j)$ is locally absolutely continuous on the interval $I^{j}_{\mystatet}:=\{t:(t, j)\in \dom \mystatet\}$. 
	\end{enumerate}
	
	Note that $\mystatet_{1}$ and $\mystatet_{2}$ are functions defined on hybrid time domains. Therefore, Lemma \ref{lemma:concatenationHTD} guarantees that $\dom \mystatet$ is hybrid time domain. Lemma \ref{lemma:concatenation_state_locallyabsolutelycontinuous} ensures that for each $j\in \nnumbers$, $t\mapsto \mystatet(t,j)$ is locally absolutely continuous on the interval $I^{j}_{\mystatet}:=\{t:(t, j)\in \dom \mystatet\}$. Therefore, it is established that $\mystatet$ is a hybrid arc. \ifbool{report}{}{\qed}
\end{proof}

\begin{lemma}\label{lemma:concatenation_solution_c}
	Given two solution pairs, denoted $\psi_{1}$ and $\psi_{2}$ such that $\psi_{2}(0, 0)\in C$ if both $I_{\psi_{1}}^{J}$ and $I_{\psi_{2}}^{0}$ have nonempty interior, where, for each $j\in \{0, J\}$, $I_{\psi}^{j} = \{t: (t, j)\in \dom \psi\}$ and $(T, J) = \max \dom \psi_{1}$, then their concatenation $\psi = \psi_{1}|\psi_{2}$ satisfies that for each $j\in \mathbb{N}$ such that $I^{j}_{\psi}$ has nonempty interior $\interior(I^{j}_{\psi})$, $\psi = (\mystatet, \myinputt)$ satisfies
	$$
	(\mystatet(t, j),\myinputt(t, j))\in C
	$$ for all $t\in \interior I^{j}_{\mystatet}$.
\end{lemma}
\begin{proof}
	We show that for all $j\in \mathbb{N}$ such that $I^{j}_{\psi}$ has  nonempty interior, $\psi(t, j)\in C$ for all $t\in \interior I^{j}_{\psi}$.
	
	According to Definition \ref{definition:concatenation}, $\dom \psi = \dom \psi_{1} \cup (\dom \psi_{2} +\{(T, J)\})$ where $(T, J) = \max \dom \psi_{1}$.
	Therefore, 
	$$
	I^{j}_{\psi} = \left\{ \begin{aligned}
	&I^{j}_{\psi_{1}} &\text{ if }j< J\\
	&I^{J}_{\psi_{1}} \cup (I^{0}_{\psi_{2}} +  \{T\}) &\text{ if }j = J\\
	&I^{j - J}_{\psi_{2}} + T&\text{ if }j > J
	\end{aligned}\right.
	$$
	
	\begin{enumerate}[label=\Roman*]
		\item In the case of all $j< J$ such that $I^{j}_{\psi_{1}}$ has nonempty interior, given that $\psi_{1}$ is a solution pair to $\mathcal{H}$, according to Definition \ref{definition:solution}, we have
		$
		\psi_{1}(t, j)\in C
		$ for all 
		$t\in \interior I^{j}_{\psi_{1}}.
		$ As per Definition \ref{definition:concatenation}, it follows that $\psi(t, j) = \psi_{1}(t, j)$ for all $t \in \interior I^{j}_{\psi_{1}}$. Therefore, we have $\psi(t, j) = \psi_{1}(t, j)\in C$ for all $t\in \interior I^{j}_{\psi} = \interior I^{j}_{\psi_{1}}$.
		\item In the case of all $j>J$ such that $I^{j - J}_{\psi_{2}} + \{T\}$ has nonempty interior, given that $\psi_{2}$ is a solution pair  to $\mathcal{H}$, according to Definition \ref{definition:solution}, we have 
		$
		\psi_{2}(t, j)\in C
		$ for all 
		$
		t\in \interior I^{j}_{\psi_{2}}.
		$ 
		As per Definition \ref{definition:concatenation}, it follows that $\psi(t, j) = \psi_{2}(t - T, j - J)$ for all $t\in I_{\psi_{2}}^{j - J} + \{T\}$. Therefore, we have  
		$
		\psi(t, j) = \psi_{2}(t - T, j - J)\in C
		$
		for all $t\in \interior I_{\psi}^{j} = I^{j - J}_{\psi_{2}} + \{T\}$.
		\item In the case where $j = J$ and $I_{\psi}^{J} = I_{\psi_{1}}^{J}\cup (I_{\psi_{2}}^{0} + \{T\})$ has nonempty interior. If one of the intervals $I_{\psi_{1}}^{J}$ and $I_{\psi_{2}}^{0} + \{T\}$ has nonempty interior while the other does not, the validation process aligns with the previously discussed cases \uppercase\expandafter{\romannumeral 1} and \uppercase\expandafter{\romannumeral 2}.
		
		If both $I_{\psi_{1}}^{J}$ and $(I_{\psi_{2}}^{0} + \{T\})$ has nonempty interior, given that
		$$
		\psi(t, j) = \psi_{1}(t, j)\in C \quad \forall t\in \interior I_{\psi_{1}}^{J}
		$$ and that
		$$
		\begin{aligned}
		\psi(t, j) = \psi_{2}(t - T, j - J)\in C \forall t\in \interior I_{\psi_{2}}^{0} + \{T\},
		\end{aligned}
		$$ it can be concluded that $\psi(t, j)\in C$ for all $t\in \interior I^{j}_{\psi}\cup (\interior I_{\psi_{2}}^{0} + \{T\})$.
		
		In addition,  since the point $\psi_{2}(0, 0)$ is assumed to belong to $C$ in this case, then $\psi(t, J) \in C$ holds  for all $t\in \interior I_{\psi}^{J}  = \interior (I_{\psi_{1}}^{J} \cup (I_{\psi_{2}}^{0} + \{T\}))$.
		
	\end{enumerate}
	Hence, it is established that for all $j\in \mathbb{N}$ such that $I_{\psi}^{j}$ has nonempty interior, $\psi(t, j)\in C$ for each $t\in \interior I^{j}_{\psi}$.\ifbool{report}{}{\qed}
\end{proof}
\begin{lemma}\label{lemma:concatenation_solution_f}
	Given two solution pairs, denoted $\psi_{1}$ and $\psi_{2}$, then their concatenation $\psi = \psi_{1}|\psi_{2}$ satisfies that for each $j\in \mathbb{N}$ such that $I^{j}_{\psi}$ has nonempty interior $\interior(I^{j}_{\psi})$, $\psi = (\mystatet, \myinputt)$ satisfies
	$$
	\frac{\text{d}}{\text{d} t} {\mystatet}(t,j) = f(\mystatet(t,j), \myinputt(t,j))
	$$ for almost all $t\in I^{j}_{\mystatet}$.
\end{lemma}
\begin{proof}
	We show that for all $j\in \mathbb{N}$ such that $I^{j}_{\psi}$ has  nonempty interior and almost all $t\in I_{\psi}^{j}$,
	$
	\frac{\text{d}}{\text{d} t} {\mystatet} = f(\mystatet(t, j), \myinputt(t, j)).
	$
	
	According to Definition \ref{definition:concatenation}, it follows that $\dom \psi = \dom \psi_{1} \cup (\dom \psi_{2} +\{(T, J)\})$ where $(T, J) = \max \dom \psi_{1}$.
	Therefore, we have
	$$
	I^{j}_{\psi} = \left\{ \begin{aligned}
	&I^{j}_{\psi_{1}} &\text{ if }j< J\\
	&I^{J}_{\psi_{1}} \cup (I^{0}_{\psi_{2}} +  \{T\}) &\text{ if }j = J\\
	&I^{j - J}_{\psi_{2}} + T&\text{ if }j > J.
	\end{aligned}\right.
	$$
	\begin{enumerate}[label=\Roman*]
		\item In the case of all $j < J$ such that $I^{j}_{\psi_{1}}$ has  nonempty interior, given that $\psi_{1}$ is a solution pair, it follows that $\frac{\text{d}}{\text{d} t} {\mystatet}_{1} (t, j)= f(\mystatet_{1}(t, j), \myinputt_{1}(t, j))$ for almost all $t\in \interior I^{j}_{\psi_{1}}$. According to Definition \ref{definition:concatenation}, we have
		$$
		(\mystatet(t, j), \myinputt(t, j)) = (\mystatet_{1}(t, j), \myinputt_{1}(t, j))
		$$ for all $(t, j)\in \dom \psi_{1}\backslash (T, J)$. Therefore, it is established that
		$$
		\begin{aligned}
		\frac{\text{d}}{\text{d} t} {\mystatet}(t, j) &= \frac{\text{d}}{\text{d} t} {\mystatet}_{1}(t, j) = f(\mystatet_{1}(t, j), \myinputt_{1}(t, j)) \\
		&= f(\mystatet(t, j), \myinputt(t, j))
		\end{aligned}
		$$ for almost all $t\in \interior I^{j}_{\psi}$.
		\item In the case of all $j > J$ such that $I^{j- J}_{\psi_{2}}$ has  nonempty interior, given that $\psi_{2}$ is a solution pair and $I^{j- J}_{\psi_{2}}$ has nonempty interior, it follows that $\frac{\text{d}}{\text{d} t} {\mystatet}_{2}(t, j) = f(\mystatet_{2}(t- T, j - J), \myinputt_{2}(t- T, j- J))$ for almost all $t\in \interior I^{j - J}_{\psi_{2}} +\{T\}$. According to Definition \ref{definition:concatenation}, we have 
		$$
		\begin{aligned}
		&(\mystatet(t, j), \myinputt(t, j)) = (\mystatet_{2}(t - T, j - J), \myinputt_{2}(t - T, j - J))
		\end{aligned}
		$$
		for all $t\in I_{\psi_{2}}^{j - J} + \{T\}$. Therefore, it is established that
		$$
		\begin{aligned}
		\frac{\text{d}}{\text{d} t} {\mystatet}(t, j) &=\frac{\text{d}}{\text{d} t} {\mystatet}_{2}(t - T, j - J) \\
		&= f(\mystatet_{2}(t - T, j - J), \myinputt_{2}(t - T, j - J)) \\
		&= f(\mystatet(t, j), \myinputt(t, j))
		\end{aligned}
		$$
		for almost all $t\in \interior I^{j}_{\mystatet}$.
		\item In the case where $j = J$ and $I_{\psi}^{J} = I_{\psi_{1}}^{J}\cup (I_{\psi_{2}}^{0} + \{T\})$ has nonempty interior, If one of the intervals $I_{\psi_{1}}^{J}$ and $I_{\psi_{2}}^{0} + \{T\}$ has nonempty interior while the other does not, the validation process aligns with the previously discussed cases \uppercase\expandafter{\romannumeral 1} and \uppercase\expandafter{\romannumeral 2}.
		If both $I_{\psi_{1}}^{J}$ and $(I_{\psi_{2}}^{0} + \{T\})$ have nonempty interior, it follows that $I_{\psi}^{J} =  I_{\psi_{1}}^{J} \cup (I_{\psi_{2}}^{0} + \{T\})$. Given that $\frac{\text{d}}{\text{d} t} {\mystatet}(t, J) = f(\mystatet(t, J), \myinputt(t, J))$ for almost all $t\in I_{\psi_{1}}^{J}$ and $t\in I_{\psi_{2}}^{0} + \{T\}$, it is established that  $\frac{\text{d}}{\text{d} t} {\mystatet}(t, j)= f(\mystatet(t, j), \myinputt(t, j))$ for $t\in I_{\psi}^{J} =  I_{\psi_{1}}^{J} \cup (I_{\psi_{2}}^{0} + \{T\})$.
	\end{enumerate}
	Therefore, $ \frac{\text{d}}{\text{d} t} {\mystatet}(t, j) = f(\mystatet(t, j), \myinputt(t, j))$ holds for almost all $t\in I_{\psi}^{j}$.\ifbool{report}{}{\qed}
\end{proof}
\begin{lemma}\label{lemma:concatenation_solution_dg}
	Given a hybrid system $\hs = (C, f, D, g)$ and two solution pairs to $\hs$, denoted $\psi_{1}$ and $\psi_{2}$, such that $\mystatet_{1}(T, J) = \mystatet_{2}(0,0)$, where $(T, J) = \max \dom \mystatet_{1}$, then their concatenation $\psi = \psi_{1}|\psi_{2} = (\mystatet, \myinputt)$ satisfies that for all $(t,j)\in \dom (\mystatet, \myinputt)$ such that $(t,j + 1)\in \dom (\mystatet, \myinputt)$, 
$$
	\begin{aligned}
	(\mystatet(t, j), \myinputt(t, j))&\in D\\
	\mystatet(t,j+ 1) &= g(\mystatet(t,j), \myinputt(t, j)).
	\end{aligned}
$$
\end{lemma}
\begin{proof}
	We show that for all $(t, j)\in \dom \psi$ such that $(t, j + 1)\in \dom \psi  = \dom \psi_{1} \cup (\dom \psi_{2} +\{(T, J)\})$ where $(T, J) = \max \dom \psi_{1}$,
	$$
	\begin{aligned}
	\psi(t, j) &\in D\\
	\mystatet(t, j + 1) &= g(\mystatet(t, j), \myinputt(t, j)).
	\end{aligned}
	$$
	\begin{enumerate}
		\item In the case where $(t, j)$ is such that $(t, j) \in \dom \psi_{1}$ and $(t, j + 1)\in \dom \psi_{1}$, since $\psi_{1}$ is a solution pair to $\hs$, it follows that 
		$$
		\begin{aligned}
		\psi_{1}(t, j) &\in D\\
		\mystatet_{1}(t, j + 1) &= g(\mystatet_{1}(t, j), \myinputt_{1}(t, j)).\\
		\end{aligned}
		$$ Given that $(t, j + 1)\in \dom \psi_{1}$, consequently, we have $(t, j)\neq (T, J)$. Since $\psi(t, j) = \psi_{1}(t, j)$ for all $(t, j)\in \dom \psi_{1}\backslash (T, J)$, it is established that			
		$$
		\begin{aligned}
		\psi(t, j)  &= 	\psi_{1}(t, j)\in D\\
		\mystatet(t, j + 1) &= \mystatet_{1}(t, j + 1)\\
		&= g(\mystatet_{1}(t, j), \myinputt_{1}(t, j))\\
		& = g(\mystatet(t, j), \myinputt(t, j)) .
		\end{aligned}
		$$
		\item In the case where $(t, j)$ is such that $(t, j) \in \dom \psi_{2} + \{(T, J)\}$ and $(t, j + 1)\in \dom \psi_{2} + \{(T, J)\}$, according to Definition \ref{definition:concatenation}, it follows that $(t - T, j - J)\in \dom \psi_{2}$ and $(t - T, j - J + 1)\in \dom \psi_{2}$. Therefore, since $\psi_{2}$ is a solution pair, it is established that 
		$$
		\begin{aligned}
		\psi_{2}&(t - T, j - J) \in D\\
		\mystatet_{2}&(t - T, j - J + 1) \\
		&= g(\mystatet_{2}(t - T, j - J), \myinputt_{2}(t - T, j - J)).
		\end{aligned}
		$$ Note that $\psi(t, j) = \psi_{2}(t - T, j - J)$ for all $(t - T, j - J)\in \dom \psi_{2}$. Therefore, the following is satisfied:
		$$
		\begin{aligned}
		\psi(t, j)  &= 	\psi_{2}(t - T, j - J) \in D\\
		\mystatet(t, j + 1) &= \mystatet_{2}(t - T, j - J + 1)\\
		&= g(\mystatet_{2}(t - T, j - J), \myinputt_{2}(t - T, j - J))\\
		& = g(\mystatet(t, j), \myinputt(t, j)). 
		\end{aligned}
		$$
					\item In the special case where $(t, j)$ is such that $(t, j) \in \dom \psi_{1}$ and $(t, j + 1) \in \dom \psi_{2} + \{(T, J)\}$, this scenario corresponds to the situation where $(t, j) = (T, J - 1)\in \dom \psi_{1}$ and $(t, j + 1) = (T, J)\in \dom \psi_{2} + {(T, J)}$. Given that $\mystatet_{1}(T, J) = \mystatet_{2}(0,0)$, the following relationships can be deduced:
					$$
					\begin{aligned}
						\mystatet_{1}(T, J ) = \mystatet_{2}(0,0) = \mystatet(T, J)\\
					\mystatet_{1}(T, J - 1) = \mystatet(T, J - 1)\\
					\myinputt_{1}(T, J - 1) = \myinputt(T, J - 1)
					\end{aligned}
					$$
					Given that $\psi_{1}$  is a solution pair to $\mathcal{H}$ and considering that
					$$
						(T,J - 1)\in \dom \psi_{1} \quad (T, J)\in \dom \psi_{1},
					$$
					according to Definition \ref{definition:solution}, it follows that
					$$
					\begin{aligned}
					\mystatet_{1}(T, J) = g(\mystatet_{1}(T, J - 1), \myinputt_{1}(T, J - 1))\\
					\psi_{1}(T, J - 1) = (\mystatet_{1}(T, J - 1), \myinputt_{1}(T, J - 1)) \in D.
					\end{aligned}
				$$
					Therefore, we establish that
					$$
					\begin{aligned}
					 \mystatet(T, J) &= \mystatet_{2}(0,0)\\
					 & = \mystatet_{1}(T, J) \\
					 &= g(\mystatet_{1}(T, J - 1), \myinputt_{1}(T, J - 1)) \\
					 & = g(\mystatet(T, J - 1), \myinputt(T, J - 1)) \\
					 \end{aligned}
					 $$
					 and
					 $$
					 \begin{aligned}
					 &(\mystatet(T, J - 1), \myinputt(T, J - 1)) \\
					 &= (\mystatet_{1}(T, J - 1), \myinputt_{1}(T, J - 1))\in D.
					\end{aligned}
					$$
	\end{enumerate}\ifbool{report}{}{\qed}
\end{proof}

\subsection{Proof of Proposition \ref{theorem:concatenationsolution}}
\begin{proof}
		We show that the concatenation result $\psi = (\mystatet,  \myinputt) = (\mystatet_{1}|\mystatet_{2}, \myinputt_{1}|\myinputt_{2})$ satisfies the conditions in Definition \ref{definition:solution}. Namely, we show in this proof that each of the following items holds:
		\begin{enumerate}
			\item $\myinputt$ is a hybrid input;
			\item $\mystatet$ is a hybrid arc;
			\item $(\mystatet(0,0), \myinputt(0,0))\in \overline{C} \cup D$;
			\item $\dom \mystatet = \dom \myinputt (= \dom (\mystatet, \myinputt))$;
			\item for each $j\in \mathbb{N}$ such that $I^{j}_{\psi}$ has nonempty interior $\interior(I^{j}_{\psi})$, $(\mystatet, \myinputt)$ satisfies
			$
			(\mystatet(t, j),\myinputt(t, j))\in C
			$ for all $t\in \interior I^{j}_{\psi}$;
			\item for each $j\in \mathbb{N}$ such that $I^{j}_{\psi}$ has nonempty interior $\interior(I^{j}_{\psi})$, $(\mystatet, \myinputt)$ satisfies 
			$
			\frac{\text{d}}{\text{d} t} {\mystatet}(t,j) = f(\mystatet(t,j), \myinputt(t,j))
			$ for almost all $t\in I^{j}_{\psi}$;
			\item for all $(t,j)\in \dom (\mystatet, \myinputt)$ such that $(t,j + 1)\in \dom (\mystatet, \myinputt)$, 
			$
			(\mystatet(t, j), \myinputt(t, j))\in D$ and
			$\mystatet(t,j+ 1) = g(\mystatet(t,j), \myinputt(t, j))$.
		\end{enumerate}
		
		Given that $\psi_{1}$ is compact, therefore, it is feasible to concatenate $\psi_{2}$ to $\psi_{1}$. Next, we show that each of the items listed above holds as follows:
		\begin{enumerate}
			\item Lemma \ref{lemma:concatenation_hybridinput} guarantees that $\myinputt$ is a hybrid input.
			\item Given that $\mystatet_{1}(T, J) = \mystatet_{2}(0,0)$, therefore, Lemma \ref{lemma:concatenation_hybridarc} guarantees that $\mystatet$ is a hybrid arc.
			\item Given that $\psi_{1}$ is a solution pair to $\mathcal{H}$, it follows that $(\mystatet_{1}(0,0), \myinputt_{1}(0, 0)) \in\overline{C} \cup D$. According to Definition \ref{definition:concatenation}, we have $$(\mystatet(t, j), \myinputt(t,  j)) = (\mystatet_{1}(t, j), \myinputt_{1}(t, j))$$ for all $(t, j)\in \dom \psi_{1}$. Since $(0,0)\in \dom \psi_{1}$, it is established that
			$$\begin{aligned}
			\psi(0,0) &= (\mystatet(0,0), \myinputt(0, 0)) \\
			&= (\mystatet_{1}(0,0), \myinputt_{1}(0, 0))\in\overline{C} \cup D.
			\end{aligned}$$
			\item Given that $\psi_{1}$ and $\psi_{2}$ are solution pairs to $\mathcal{H}$, according to Definition \ref{definition:solution}, it follows that
			$$
			\dom \psi_{1} =\dom \mystatet_{1} = \dom \myinputt_{1},
			$$
			$$
			\dom \psi_{2} = \dom \mystatet_{2} = \dom \myinputt_{2}.
			$$
			
			According to Definition \ref{definition:concatenation}, it follows that
			$$\dom \mystatet = \dom \mystatet_{1} \cup (\dom \mystatet_{2} +\{(T, J)\})$$ and $$\dom \myinputt = \dom \myinputt_{1} \cup (\dom \myinputt_{2} +\{(T, J)\}).$$
			Therefore, the following is established:
			$$
			\begin{aligned}
			\dom \mystatet &= \dom \mystatet_{1} \cup (\dom \mystatet_{2} +\{(T, J)\}) \\
			&= \dom \myinputt_{1} \cup (\dom \myinputt_{2} +\{(T, J)\}) = \dom \myinputt.
			\end{aligned}
			$$
			\item Given that $\psi_{2}(0, 0)\in C$ if both $I_{\psi_{1}}^{J}$ and $I_{\psi_{2}}^{0}$ have nonempty interior, therefore, Lemma \ref{lemma:concatenation_solution_c} ensures that for each $j\in \mathbb{N}$ such that $I^{j}_{\psi}$ has nonempty interior $\interior(I^{j}_{\psi})$, $(\mystatet, \myinputt)$ satisfies
			$$
			(\mystatet(t, j),\myinputt(t, j))\in C
			$$ for all $t\in \interior I^{j}_{\psi}$.
			\item Lemma \ref{lemma:concatenation_solution_f} guarantees that for each $j\in \mathbb{N}$ such that $I^{j}_{\psi}$ has nonempty interior $\interior(I^{j}_{\psi})$, $(\mystatet, \myinputt)$ satisfies
			$$
			\frac{\text{d}}{\text{d} t} {\mystatet}(t,j) = f(\mystatet(t,j), \myinputt(t,j))
			$$ for almost all $t\in I^{j}_{\psi}$.
			\item Given that $\mystatet_{1}(T, J) = \mystatet_{2}(0,0)$, therefore, Lemma \ref{lemma:concatenation_solution_dg} guarantees that for all $(t,j)\in \dom (\mystatet, \myinputt)$ such that $(t,j + 1)\in \dom (\mystatet, \myinputt)$, 
			$$
			\begin{aligned}
			(\mystatet(t, j), \myinputt(t, j))&\in D\\
			\mystatet(t,j+ 1) &= g(\mystatet(t,j), \myinputt(t, j)).
			\end{aligned}
			$$
		\end{enumerate}
	In conclusion, the concatenation result $\psi$ satisfies all the conditions in Definition \ref{definition:solution} and, hence, is a solution pair to $\hs$.
	\end{proof}
\section{A Computational Framework to Simulate Continuous Dynamics}\label{section:computationsimulation}
\subsection{Numerical integration scheme model}\label{section:numericalintegration}
A general numerical integration scheme can be modeled as
\begin{equation}
\label{equation:numericalintegrationimplicit}
	F_{s}(x_{k}, x_{k+ 1}, \tilde{\myinputt}, f, t) = 0
\end{equation} where $s$ denotes the step size,  $x_{k}$ denotes the state at the $k$th step, $x_{k+ 1}$ denotes the approximation of the state at the next step, $\tilde{\myinputt}$ denotes the applied input signal, $f$ denotes the flow map, and $t$ denotes the time at current step, namely, $t = ks$. Two examples on how (\ref{equation:numericalintegrationimplicit}) models explicit and implicit numerical integration scheme are given as follows.
\begin{example}
	(Forward Euler method) Given the differential equation
	$
	\label{equation:governequationexp}
		\dot{x} = f(x, \tilde{\myinputt})
$ with initial state $x_{0}$ and input signal $\tilde{\myinputt}\in \mathcal{U}_{C}$,
	the integration scheme using the forward Euler method is
	\begin{equation}
	\label{equation:forwardeulerintegrationscheme}
		x_{k+ 1}= 	x_{k} + sf(x, \tilde{\myinputt}(t))		
	\end{equation}
	 Following (\ref{equation:numericalintegrationimplicit}), (\ref{equation:forwardeulerintegrationscheme}) can be modeled as 
	\begin{equation}
		F_{s}(x_{k}, x_{k+ 1}, \tilde{\myinputt}, f, t) := x_{k} + sf(x_{k}, \tilde{\myinputt}(t)) - x_{k+ 1}
	\end{equation}
	and $x_{k+ 1}$ can be obtained by solving (\ref{equation:numericalintegrationimplicit}).
\end{example}

\begin{example}
	(Backward Euler method) Given the differential equation
	$
	\label{equation:governequationimp}
	\dot{x} = f(x, \tilde{\myinputt})
	$ with initial state $x_{0}$  and input signal $\tilde{\myinputt}\in \mathcal{U}_{C}$,
	the integration scheme using backward Euler method is
	\begin{equation}
	\label{equation:backwardeulerintegrationscheme}
	x_{k + 1} = x_{k} + sf(x_{k + 1}, \tilde{\myinputt}(t + s))		
	\end{equation}
	 Following (\ref{equation:numericalintegrationimplicit}), (\ref{equation:backwardeulerintegrationscheme}) can be modeled as 
	$$
	F_{s}(x_{k}, x_{k + 1}, \tilde{\myinputt}, f, t) := x_{k} + sf(x_{k + 1}, \tilde{\myinputt}(t + s)) 
	- x_{k + 1}
	$$
	and $x_{k + 1}$ can be obtained by solving (\ref{equation:numericalintegrationimplicit}).
\end{example}

\subsection{Zero-crossing detection model to approximate $\hat{t}$}\label{section:zerocrossing}
When the priority option $\texttt{rule} = 1$, following Step 2 above, $\hat{t}$ in (\ref{equation:integrationduration}) is the first time when $(\hat{\phi}, \tilde{\myinputt})$ leaves $C$. When the priority option $\texttt{rule} = 2$, $\hat{t}$ is the first time when $(\hat{\phi}, \tilde{\myinputt})$ leaves $C\backslash D$.

 Given a set $S\subset \myreals[n]\times\myreals[m]$, a zero-crossing function $h: \mathbb{R}^{n}\times \mathbb{R}^{m}\to \mathbb{R}$ to detect whether the pair $(\mystatet, \myinputt)$ is in set $S$ is given by
	\begin{enumerate}
		\item $h(x, u) > 0$ for all $(x, u)\in \interior S$,
		\item $h(x, u)< 0$ for all $(x, u)\in \interior ((\mathbb{R}^{n}\times \mathbb{R}^{m})\backslash S)$,
		\item $h(x, u) = 0$ for all $(x, u)\in \partial  S$,
		\item $h$ is continuous over $\mathbb{R}^{n}\times \mathbb{R}^{m}$.
\end{enumerate}
When $\texttt{rule} = 1$, this function can be used with $S = C\backslash D$ to determine $\hat{t}$ as the first time when $(\hat{\phi}, \tilde{\myinputt})$ leaves $C\backslash D$ by checking zero crossings of $h_{f}(\hat{\phi}, \tilde{\myinputt})$. When $\texttt{rule} = 2$, this function can be used with $S = C$ to determine $\hat{t}$ as the first time when $(\hat{\phi}, \tilde{\myinputt})$ leaves $C$ by checking zero crossings of $h(\hat{\phi}, \tilde{\myinputt})$. Next, we illustrate this approach by constructing a zero-crossing function for the bouncing ball system in Example \ref{example:bouncingball}.
\begin{example}[(Example \ref{example:bouncingball}, revisited)] In the bouncing ball system in Example \ref{example:bouncingball}, the flow set is $C = \{(x, u)\in \mathbb{R}^{2}\times \mathbb{R}: x_{1}\geq 0\}$ and the jump set is $D = \{(x, u) \in \mathbb{R}^{2}\times \mathbb{R}: x_{1} = 0, x_{2} \leq 0, u\geq 0\}$. Then,
	$$
		h(x, u): = x_{1}\quad \forall (x, u)\in \mathbb{R}^{n}\times \mathbb{R}^{m}.
	$$ is a zero-crossing function for the detection of solutions leaving $C$. The same function can be used to detect solutions leaving $C\backslash D$ since $D\subset \partial C$ which implies that $\interior (C\backslash D) = \interior C$ and $\interior ((\mathbb{R}^{n}\times \mathbb{R}^{m})\backslash (C\backslash D)) = \interior ((\mathbb{R}^{n}\times \mathbb{R}^{m})\backslash C)$. 
\end{example}

The zero-crossing detection algorithm that approximates $\hat{t}$ can be modeled as 
\begin{equation}
	\label{equation:zerocrossing}
	\hat{t} = t_{zcd} (\hat{\mystatet}, \tilde{\myinputt}, \texttt{rule}, C, D)
\end{equation}
where the function $t_{zcd}$ employs a zero-crossing detection function for $C\backslash D$ or $C$, depending on the value of $\texttt{rule}$. We set $\hat{t}$ as $-1$ when no zero-crossing is detected.

\subsection{A computational framework to simulate continuous dynamics}
A computational approach to simulate the continuous dynamics of $\hs$ is given in Algorithm \ref{algo:flowsystemsimulator}. 
\begin{algorithm}[htbp]
	\caption{A computational framework to simulate the continuous dynamics $\texttt{continuous\_simulator}$}
	\begin{algorithmic}[1]
		\Function{$\texttt{continuous\_simulator}$}{$C, f, D, \texttt{rule}, s, x_{0}, \tilde{\myinputt}$}
		\raggedright
		\State Set $\mystatet_{0} \leftarrow x_{0}$, $\hat{t}\leftarrow 0$. 
		\While{$t\leq\overline{t}(\tilde{\myinputt}) \&   \hat{t} = -1$}
			\State Compute $\hat{\mystatet}$ by solving $F_{s}(\mystatet_{0}, \hat{\mystatet}, \tilde{\myinputt}, f, t) = 0$.
			\State $\hat{t} \leftarrow t_{zcd} (\hat{\mystatet}, \tilde{\myinputt}, \texttt{rule}, C, D)$.
			\If{$\hat{t} = -1$}
				\State $\mystatet\leftarrow \mystatet| \hat{\mystatet}$, $\myinputt \leftarrow \myinputt| \hat{\myinputt}$.
				\State $\mystatet_{0} \leftarrow \hat{\mystatet}(\hat{T}, 0)$ where $(\hat{T}, 0) = \max \dom \hat{\mystatet}$.
				\Else
				\State $\mystatet\leftarrow \mystatet| \hat{\mystatet}([0, \hat{t}], 0)$, $\myinputt \leftarrow \myinputt| \hat{\myinputt}([0, \hat{t}], 0)$.
				\State $\mystatet_{0} \leftarrow \hat{\mystatet}(\hat{t}, 0)$.
			\EndIf
		\EndWhile
		\State \Return $(\mystatet, \myinputt)$.
		\EndFunction
	\end{algorithmic}
\label{algo:flowsystemsimulator}
\end{algorithm}


\section{Proof of Proposition \ref{proposition:motioninflated}}\label{section:proofinflated}
\subsection{Theoretical Tools to Prove Proposition \ref{proposition:motioninflated}}
We show that $\psi = (\mystatet,  \myinputt)$ in Proposition \ref{proposition:motioninflated} satisfies the conditions in Definition \ref{definition:solution} for $(C_{\delta}, f_{\delta}, D_{\delta}, g_{\delta})$ where $(C_{\delta}, f_{\delta}, D_{\delta}, g_{\delta})$ is the $\delta$-inflation of hybrid system of $(C, f, D, g)$ for each $\delta > 0$, as defined in Definition \ref{definition:inflation}.

\begin{lemma}\label{lemma:inflationsame}
	Given a hybrid system $\hs = (C, f, D, g)$ and its $\delta$-inflation $\hs_{\delta} = (C_{\delta}, f_{\delta}, D_{\delta}, g_{\delta})$ for each $\delta > 0$, if $\psi = (\mystatet, \myinputt)$ is a solution pair to $\mathcal{H}$, then $\psi$ is also a solution pair to $\hs_{\delta}$.
\end{lemma}
\begin{proof}
	We show that $\psi$ satisfies all the conditions in Definition \ref{definition:solution} for $(C_{\delta}, f_{\delta}, D_{\delta}, g_{\delta})$, namely,
	\begin{enumerate}
		\item $\mystatet$ is a hybrid arc and $\myinputt$ is a hybrid input;
		\item $(\phi(0,0), \myinputt(0,0))\in \overline{C_{\delta}}\cup D_{\delta}$;
		\item For each $j\in \mathbb{N}$ such that $I^{j}_{\mystatet}$ has nonempty interior $\interior(I^{j}_{\mystatet})$, $(\mystatet, \myinputt)$ satisfies
		$$
		(\mystatet(t, j),\myinputt(t, j))\in C_{\delta}
		$$ for all $t\in \interior I^{j}_{\mystatet}$, and 
		$$
		\frac{\text{d}}{\text{d} t} {\mystatet}(t,j) = f_{\delta}(\mystatet(t,j), \myinputt(t,j))
		$$ for almost all $t\in I^{j}_{\mystatet}$.
		\item For all $(t,j)\in \dom (\mystatet, \myinputt)$ such that $(t,j + 1)\in \dom (\mystatet, \myinputt)$, 
		\begin{equation}
		\begin{aligned}
		(\mystatet(t, j), \myinputt(t, j))&\in D_{\delta}\\
		\mystatet(t,j+ 1) &= g_{\delta}(\mystatet(t,j), \myinputt(t, j)).
		\end{aligned}
		\end{equation}
	\end{enumerate}
	Next, we show that each item above is satisfied.
	\begin{enumerate}
		\item Given that $\psi = (\mystatet, \myinputt)$ is a solution pair to $\hs$, therefore, $\mystatet$ is a hybrid arc and $\myinputt$ is a hybrid input. 
		\item Since $\psi = (\mystatet, \myinputt)$ is a solution pair to $\hs$, then $(\phi(0,0), \myinputt(0,0))\in \overline{C}\cup D$. Note that $C\subset C_{\delta}$, $D\subset D_{\delta}$. Therefore, 
		$$
		(\phi(0,0), \myinputt(0,0))\in \overline{C}\cup D \subset \overline{C_{\delta}}\cup D_{\delta}.
		$$
		\item For all $j\in \mathbb{N}$ such that $I^{j}= \{t: (t, j)\in \dom (\phi, \myinputt)\}$ has nonempty interior, 
		\begin{enumerate}
			\item $(\phi(t, j),\myinputt(t, j))\in C \subset C_{\delta}$ for all $t\in \interior I^j$.
			\item for almost all $t\in I^j$,
			$$
			\dot{\phi}(t,j) = f(\phi(t,j), \myinputt(t,j)) = f_{\delta}(\phi(t,j), \myinputt(t,j)).
			$$
		\end{enumerate}
		\item For all $(t,j)\in \dom (\phi, \myinputt)$ such that $(t,j + 1)\in \dom (\phi, \myinputt)$,
		$$
		\begin{aligned}
		(\phi(t, j), \myinputt(t, j))&\in D\subset D_{\delta}\\
		\phi(t,j+ 1) &= g(\phi(t,j), \myinputt(t, j)) = g_{\delta}(\phi(t,j), \myinputt(t, j))
		\end{aligned}
		$$
	\end{enumerate}
	Since $\psi$ satisfies all the items in Definition \ref{definition:solution} for data $(C_{\delta}, f_{\delta}, D_{\delta}, g_{\delta})$, it is established that $\psi$ is a solution pair to $(C_{\delta}, f_{\delta}, D_{\delta}, g_{\delta})$.\ifbool{report}{}{\qed}
\end{proof}
\subsection{Proof of Proposition \ref{proposition:motioninflated}}
\begin{proof}
	We show that the motion plan $\psi = (\mystatet, \myinputt)$ to $\mathcal{P}$ satisfies each condition in Problem \ref{problem:motionplanning} for the data $\mathcal{P}_{\delta} = (X_{0}, X_{f}, X_{u},  (C_{\delta}, f_{\delta}, D_{\delta}, g_{\delta}))$, namely,
	\begin{enumerate}
		\item $\phi(0, 0) \in X_{0}$;
		\item $(\mystatet, \myinputt)$ is a solution pair to $\mathcal{H}_{\delta}$;
		\item $(T,J)$ is such that $\mystatet(T,J)\in X_{f}$, namely, the solution belongs to the final state set at hybrid time $(T, J)$;
		\item $(\mystatet(t,j), \myinputt(t, j))\notin X_{u}$ for each $(t,j)\in \dom (\mystatet, \myinputt)$ such that $t + j \leq T+ J$.
\end{enumerate}
Given that the initial state set $X_{0}$, the final state set $X_{f}$, and the unsafe set $X_{u}$ are the same in both $\mathcal{P}$ and $\mathcal{P}_{\delta}$ and that $\psi$ is a motion plan to $\mathcal{P}$, items 1, 3, and 4 are satisfied for free. Note that, by construction, every solution pair to $\hs$ is a solution pair to $\hs_{\delta}$. In fact, by Lemma \ref{lemma:inflationsame}, $(\mystatet, \myinputt)$ is a solution pair to $\mathcal{H}_{\delta}$, namely, item 2 is satisfied. Therefore, all the items are satisfied and $(\mystatet, \myinputt)$ is a motion plan to $\mathcal{P}_{\delta}$. \ifbool{report}{}{\qed}
\end{proof}
\section{Proof of Lemma \ref{lemma:motionplanintlated}}\label{section:proofinflatedlemma}
To prove Lemma \ref{lemma:motionplanintlated}, we first establish the following result.
\begin{lemma}\label{lemma:inflationsubset}
	Given a hybrid system defined as $\hs = (C, f, D, g)$ and its $\delta_{f}$-inflation, denoted $\hs_{\delta_{f}} = (C_{\delta_{f}}, f_{\delta_{f}}, D_{\delta_{f}}, g_{\delta_{f}})$ and defined in (\ref{model:inflatedhybridsystem}), if a state input pair $(y, v)\in \myreals[n]\times\myreals[m]$ is such that $(y, v)\in X$, where $X$ can be either $C$ or $D$, then for each $\delta'\in [0, \delta_{f}]$, we have $(y + \delta'\myball, v + \delta'\myball) \subset X_{\delta_{f}}$.
\end{lemma}
\begin{proof}
	Note that, from (\ref{model:inflatedhybridsystem}), $C_{\delta_{f}}$ and $D_{\delta_{f}}$ are captured by the following set by choosing $X = C$ or $X = D$, respectively:
	\begin{equation}\label{equation:generalizedinflationset}
	\begin{aligned}
	X_{\delta_{f}} := &\{(x, u)\in \mathbb{R}^{n}\times \mathbb{R}^{m}: \exists (y,  v)\in X:  x\in y + \delta_{f}\myball, \\
	&u\in v+ \delta_{f} \myball\}.
	\end{aligned}
	\end{equation}
	Then, to prove $(y + \delta'\myball, v + \delta'\myball) \subset X_{\delta_{f}}$ where $(y, v)\in X$, we show that for each $(y', v')\in (y + \delta'\myball, v + \delta'\myball)$, we have $(y', v') \in X_{\delta_{f}}$.
	
	Given that $(y', v')\in (y + \delta'\myball, v + \delta'\myball)$ where $(y, v)\in X$, by (\ref{equation:generalizedinflationset}), it follows that 
	$$
	(y', v')\in X_{\delta'}.
	$$
	Since $\delta' \in [0, \delta_{f}]$, it follows that $y + \delta'\myball \subset y + \delta_{f}\myball$ and $v + \delta'\myball \subset  v + \delta_{f}\myball$. Then, by (\ref{equation:generalizedinflationset}), we have that
	$$
		X_{\delta'} \subset X_{\delta_{f}}.
	$$
	Therefore, for each $(y', v')\in (y + \delta'\myball, v + \delta'\myball)$ and each $\delta' \in [0, \delta_{f}]$ where $(y, v)\in X$, we have 
	$$
	(y', v')\in X_{\delta'}\subset X_{\delta_{f}}.
	$$
	  Hence, we have 
	  $$
	  	(y + \delta'\myball, v + \delta'\myball) \subset X_{\delta_{f}}.
	  $$
\end{proof}
To establish Lemma \ref{lemma:motionplanintlated}, we proceed as follows.
Proposition \ref{proposition:motioninflated} establishes that $\psi$, which is a motion to problem $\mathcal{P}$, is also a motion plan to problem $\mathcal{P}_{\delta_{f}}$. We need to show that $\psi$ has clearance $\min \{\delta_{s}, \delta_{f}\}$.
	
	Since $\psi = (\mystatet, \myinputt)$ is a solution pair to $\mathcal{H} = (C, f, D, g)$, then for all $j\in \mathbb{N}$ such that $I^{j}= \{t: (t, j)\in \dom (\mystatet, \myinputt)\}$ has nonempty interior, we have
	$$
	(\phi(t, j),\myinputt(t, j))\in C
	$$ for all $t\in \interior I^j.$
	Therefore, by Lemma \ref{lemma:inflationsubset}, for each $\delta'\in [0, \delta_{f}]$ and each $j\in \mathbb{N}$ such that $I^{j}$ has nonempty interior, it follows that
	$$
	(\phi(t, j) +\delta' \mathbb{B}, \myinputt(t, j) + \delta' \mathbb{B} )\subset C_{\delta_{f}}.
	$$
	Hence, item 1 in Definition \ref{definition:dynamicsclearance} is satisfied. 
	
	Similarly, since $\psi$ is a solution pair to $\mathcal{H} = (C, f, D, g)$, for each $(t, j)\in \dom \psi $ such that $(t, j + 1)\in \dom \psi$, we have
	$$(\phi(t, j), \myinputt(t, j))\in D.$$ 
	Therefore, by Lemma \ref{lemma:inflationsubset}, for each $\delta'\in [0, \delta_{f}]$ and each $(t, j)\in \dom \psi $ such that $(t, j + 1)\in \dom \psi$, it follows that
	$$
	(\phi(t, j) +\delta_{f} \mathbb{B}, \myinputt(t, j) + \delta_{f} \mathbb{B} )\subset D_{\delta_{f}}.
	$$
	Hence, item 2 in Definition \ref{definition:dynamicsclearance} is also satisfied. 
	
	Therefore, $\delta_{f}$ satisfies all the conditions in Definition \ref{definition:dynamicsclearance} and, hence, is the dynamics clearance. Then, by Definition \ref{definition:clearance}, the clearance of $\psi$ is $\min \{\delta_{s}, \delta_{f}\}$. \ifbool{report}{}{\qed}
\section{Proof of Lemma \ref{lemma:nearestvertex}}\label{section:proof_near}
The proof closely parallels that of Lemma 4 as presented in \cite{kleinbort2018probabilistic}. Suppose there exists a vertex, denoted $z$, in the search graph $\mathcal{T} = (V, E)$ such that $\overline{x}_{z} \in S$ and $\overline{x}_{z}\notin x_{c} +\delta\mathbb{B}$, as otherwise it is immediate that $\overline{x}_{v_{cur}}\in x_{c} + \delta\mathbb{B}$ because all the vertices in $\mathcal{T}$ belong to $x_{c} +\delta\mathbb{B}$. We show that, under the conditions in Lemma  \ref{lemma:nearestvertex}, if the sampling point $x_{rand}$ returned by the function call $\texttt{random\_state}$ is such that $x_{rand}\in x_{c} + \pn{(}\delta/5\pn{)}\mathbb{B}$, it follows that $\overline{x}_{v_{cur}}\in x_{c} + \delta \mathbb{B}$.
	
Given that $x_{rand}\in x_{c} + \pn{(}\delta/5\pn{)}\mathbb{B}$, it follows that
\begin{equation}\label{equation:nearestrandom1}
	|x_{rand} - x_{c}| \leq \frac{\delta}{5}.
\end{equation}
	Since $\overline{x}_{v}\in x_{c}+ \pn{(}2\delta/5\pn{)}\mathbb{B}$, then we have
$$
	|\overline{x}_{v} - x_{c}| \leq \frac{2\delta}{5}.
$$
	Therefore, by the triangle inequality, it follows that
	
$$
\begin{aligned}
	&|x_{rand} - \overline{x}_{v}| \leq |x_{rand} - x_{c}| +  |x_{c} - \overline{x}_{v} | \leq \frac{3\delta}{5}.\\
\end{aligned}
$$
	
Since $\overline{x}_{z}\notin x_{c} +\delta\mathbb{B}$, then 
$
	|\overline{x}_{z} - x_{c}| > \delta.
$
Using (\ref{equation:nearestrandom1}), by the triangle inequality, it follows that
$$
\begin{aligned}
	\delta &< |\overline{x}_{z} - x_{c}| \leq |\overline{x}_{z} - x_{rand}| + | x_{rand} - x_{c}|  \\
	& \leq|\overline{x}_{z} - x_{rand}| + \frac{\delta}{5}.
\end{aligned}
$$
namely,
$
	|\overline{x}_{z} - x_{rand}| > 4\delta/5.
$
	
	Since $|\overline{x}_{v} - x_{rand}| \leq 3\delta/5$ and $|\overline{x}_{z} - x_{rand}| > 4\delta/5$, we have that $x_{rand}$ is closer to $v$ than to $z$. This implies that $z$, which can be any vertex that is not in $x_{c} +\delta\mathbb{B}$, will not be reported as $v_{cur}$ by the function call $\texttt{nearest\_neighbor}$. 
	
	If $v$ is reported as $v_{cur}$, then, since $\overline{x}_{v} \in x_{c}+ 2\delta/5$, we have that $\overline{x}_{v} \in x_{c}+ 2\delta/5 \subset x_{c} + \delta\mathbb{B}$.
	If $v$ is not reported as $v_{cur}$, then there must exists another vertex $y\in V$ such that 
	$y$ is either closer to or equidistant from $x_{rand}$ as compared to $v$, i.e.,
	$$
	|\overline{x}_{y} - x_{rand}| \leq |\overline{x}_{v} - x_{rand}|\leq \frac{3\delta}{5}.
	$$ Then, $|\overline{x}_{y} - x_{c}|\leq |\overline{x}_{y} - x_{rand}| + |x_{rand} - x_{c}|\leq  4\delta/5$, which implies that $\overline{x}_{y} \in x_{c} + \delta \mathbb{B}$. This implies that if the sampling point $x_{rand}$ is such that $x_{rand}\in x_{c} + \pn{(}\delta/5\pn{)}\mathbb{B}$, no matter $v$ or $y$ is reported as $x_{v_{cur}}$, we have $x_{v_{cur}}\in x_{c} + \delta\mathbb{B}$.
	
	Note that $x_{rand}\in x_{c} + \pn{(}\delta/5\pn{)}\mathbb{B}$ implies $x_{rand}$ is sampled from the ball centered at $x_{c}$ with radius $\delta/5$. Therefore, by Definition \ref{assumption:uniformsample} and (\ref{equation:probabilitylebesgue}), the probability of $x_{rand}\in x_{c} + \pn{(}\delta/5\pn{)}\mathbb{B}$ is $\zeta_{n}(\delta/5)^{n}/\mu(S)$, where $\zeta_{n}$ denotes the Lebesgue measure of the unit ball in $\mathbb{R}^{n}$.
\section{Proof of Lemma \ref{lemma:pccontinuouslowerbound}}\label{section:proof_continuous}
\subsection{Supporting Lemmas to Prove Lemma \ref{lemma:pccontinuouslowerbound}}
\begin{lemma}(Lemma 2 in \cite{kleinbort2018probabilistic})\label{lemma:kleinbort}
	Given a hybrid system $\mathcal{H}$ that satisfies Assumption \ref{assumption:flowlipschitz},  let $\psi = (\mystatet, \myinputt)$ and $\psi' = (\mystatet', \myinputt')$ be two purely continuous solution pairs to $\mathcal{H}$ with $(T, 0) = \max \dom \psi$, $(T', 0) = \max \dom \psi'$ and $T' \leq T$.  The input functions $\myinputt$ and $\myinputt'$ are assumed to be constant, denoted $\myinputt: [0, T]\times\{0\} \to u_{C}\in U_{C}$ and $\myinputt': [0, T']\times\{0\} \to u'_{C}\in U_{C}$, respectively. Suppose initial state $\mystatet(0, 0) \in \mystatet'(0, 0) + \delta\mathbb{B}$ for some $\delta > 0$. Then $|\phi(T') - \phi'(T')| \leq e^{K^{f}_{x}T'} \delta  + K^{f}_{u}T'e^{K^{f}_{x}T'} \Delta u$, where $\Delta u = |u_{C} - u'_{C}|$, $K^{f}_{x}$ and $K^{f}_{u}$ are from Assumption \ref{assumption:flowlipschitz}.
\end{lemma}
\begin{lemma}\label{lemma:flowevent1}
	Given a hybrid system $\mathcal{H}$ that satisfies Assumption \ref{assumption:flowlipschitz}, let $\psi = (\mystatet, \myinputt)$ and $\psi' = (\mystatet', \myinputt')$ be two purely continuous solution pairs to $\mathcal{H}$ with $(T, 0) = \max \dom \psi$, $(T', 0) = \max \dom \psi'$ and $T' \leq T$. The input functions $\myinputt$ and $\myinputt'$ are assumed to be constant, denoted $\myinputt: [0, T]\times\{0\} \to u_{C}\in U_{C}$ and $\myinputt': [0, T']\times\{0\} \to u'_{C}\in U_{C}$, respectively.
	Suppose initial state $\mystatet(0, 0) \in \mystatet'(0, 0) + \kappa_{1}\delta\mathbb{B}$ for some $\delta > 0$ and $\kappa_{1}\in (0, 1/2)$.
	Then, for each $\kappa_{2}\in (2\kappa_{1}, 1)$ and each $\epsilon\in (0, \frac{\kappa_{2}\delta}{2})$, $\phi$ and $\phi'$ are $(\overline{T}, \kappa_{2}\delta)$-close where $\overline{T} = \max \{T, T'\}$ if the following hold:
	\begin{enumerate}
		\item $T$ and $T'$ are such that\begin{equation}
		\label{equation:tm1}
		\begin{aligned}
		&T' \in T_{k} = \{t_{l}\in [\max\{T - \kappa_{2}\delta, 0\}, T]: \\
		&\forall t'\in [t_{l}, T], \phi(t', 0) + (\frac{\kappa_{2}\delta}{2} - \epsilon)\mathbb{B} \\
		&\subset \phi(T, 0) + \frac{\kappa_{2}\delta}{2}\mathbb{B}\},
		\end{aligned}
		\end{equation}
		\item $u_{C}$ and $u'_{C}$ are such that \begin{equation}
		\label{equation:inequalityflowproof31}
		|u_{C} - u'_{C}| < \frac{\frac{\kappa_{2} \delta}{2} - \epsilon -\exp(K_{x}^{f}T)\kappa_{1}\delta}{K^{f}_{u}T \exp(K_{x}^{f}T)}.
		\end{equation}
	\end{enumerate}
\end{lemma}
\begin{proof}
	We show that $\psi$ and $\psi'$ satisfy each item in Definition \ref{definition:closeness}.
	\begin{enumerate}
		\item This item shows that $\psi$ and $\psi'$ satisfy the first condition in Definition \ref{definition:closeness}, namely, for all $(t, j)\in \dom \mystatet$ with $t + j \leq \overline{T}$, there exists $s$ such that $(s, j)\in \dom \mystatet'$, $|t - s|< \kappa_{2}\delta$, and $|\mystatet(t, j) - \mystatet'(s, j)| < \kappa_{2}\delta$.
		
		Because of Assumption \ref{assumption:flowlipschitz}, according to Lemma \ref{lemma:kleinbort}, it follows
		\begin{equation}\label{equation:flowlemmainequality}
			\begin{aligned}
			&|\mystatet(t, 0) - \mystatet'(t, 0)| \\
			&\leq \exp(K_{x}^{f}t)\kappa_{1}\delta + K^{f}_{u}t\exp(K_{x}^{f}t) |u_{C} - u'_{C}| 
			\end{aligned}
		\end{equation}
		
		Note that $T' \in T_{k}\subset [\max\{T - \kappa_{2}\delta, 0\}, T]$. Therefore, we have $T'\leq T$. Because $T' \leq T$, for all $(t, 0)\in \dom \mystatet'$ with $t + 0 \leq T' + 0 \leq T + 0 =T = \overline{T}$, there exists $s = t$ such that $(s, 0)\in \dom \phi$, $|t - s|= 0 < \kappa_{2}\delta$. Then, by applying (\ref{equation:inequalityflowproof31}) to (\ref{equation:flowlemmainequality}), we have
		$$
		\begin{aligned}
			&|\mystatet(t, 0) - \mystatet'(s, 0)|  = |\mystatet(t, 0) - \mystatet'(t, 0)| \\
			&\leq \exp(K_{x}^{f}t)\kappa_{1}\delta + K^{f}_{u}t\exp(K_{x}^{f}t)\Delta u \\
			&\leq \exp(K_{x}^{f}T)\kappa_{1}\delta + K^{f}_{u}T\exp(K_{x}^{f}T)\Delta u\\
			&\leq \frac{\kappa_{2}\delta}{2} - \epsilon < \kappa_{2} \delta.
		\end{aligned}
		$$
		Hence, item 1 in Definition \ref{definition:closeness} is established. 
		\item This item shows that $\psi$ and $\psi'$ satisfy the second condition in Definition \ref{definition:closeness}, namely, for all $(t, j)\in \dom \mystatet'$ with $t + j \leq \overline{T} = T$, there exists $s$ such that $(s, j)\in \dom \mystatet$, $|t - s|< \kappa_{2}\delta$, and $|\mystatet'(t, j) - \mystatet(s, j)| < \kappa_{2}\delta$.
		
		We consider the following two cases.
		\begin{enumerate}
			\item For all $(t, 0)\in \dom \phi$ with $0 \leq t + 0 \leq T' + 0 = T'$, there exists $s = t$ such that $(s, 0)\in \dom \phi'$, $|t - s| = 0 < \kappa_{2} \delta$ and 
			$$
			|\phi(t, 0) - \phi'(s, 0)| \leq \frac{\kappa_{2} \delta}{2} - \epsilon < \kappa_{2} \delta - \epsilon < \kappa_{2} \delta
			$$ because of (\ref{equation:flowlemmainequality}).
			\item This item considers the case of $(t, 0)\in \dom \phi$ with $T'\leq t + 0 \leq T + 0 = T$. 
%
			
			For all $(t, 0)\in \dom \phi$ with $T'\leq t + 0 \leq T + 0 = T$, let $s = T'$. Because of (\ref{equation:tm1}), then $s = T'\in [\max\{T - \kappa_{2}\delta, 0\}, T]$. Since $s \in [\max\{T - \kappa_{2}\delta, 0\}, T]$ and $t\in [T', T]$, therefore, we have 
			$$
			|t - s| \leq \kappa_{2}\delta.
			$$ Also, because of the definition of $T_{k}$ in (\ref{equation:tm1}), for all $(t, 0)\in \dom \phi$ with $T' \leq t \leq T$, we have
			$$
			\begin{aligned}
			\phi(t, 0) &\in \phi(t, 0) + (\frac{\kappa_{2}\delta}{2} - \epsilon)\mathbb{B} \\
			&\subset \phi(T, 0) + \frac{\kappa_{2}\delta}{2}\mathbb{B}.
			\end{aligned}
			$$
			Because of (\ref{equation:flowlemmainequality}) and (\ref{equation:tm1}), we also have
			$$
			\begin{aligned}
			&\phi'(s, 0) \in \phi(s, 0) + (\frac{\kappa_{2}\delta}{2} - \epsilon)\mathbb{B} \\
			&\subset \phi(T, 0) + \frac{\kappa_{2}\delta}{2}\mathbb{B}.
			\end{aligned}
			$$
			
			Therefore, $\phi(t, 0)$ and $\phi'(s, 0)$ are two points in a ball centered at $\phi(T, 0)$ with radius $\frac{\kappa_{2}\delta}{2}$. Note that the maximum distance between two points within a ball is its diameter. 
			Hence, we have
			$$|\phi(t, 0) - \phi'(s, 0)| < \kappa_{2} \delta.$$ 
		\end{enumerate}
	\end{enumerate}
	Therefore, when both (\ref{equation:tm1}) and (\ref{equation:inequalityflowproof31}) hold, $\phi$ and $\phi'$ are $(\overline{T}, \kappa_{2}\delta)$-close.\ifbool{report}{}{\qed}
\end{proof}
\begin{lemma}\label{lemma:flowevent2}
	Given a hybrid system $\mathcal{H}$ that satisfies Assumption \ref{assumption:flowlipschitz}, let $\psi = (\mystatet, \myinputt)$ and $\psi' = (\mystatet', \myinputt')$ be two purely continuous solution pairs to $\mathcal{H}$ with $(T, 0) = \max \dom \psi$, $(T', 0) = \max \dom \psi'$ and $T' \leq T$. The input functions $\myinputt$ and $\myinputt'$ are assumed to be constant, denoted $\myinputt: [0, T]\times\{0\} \to u_{C}\in U_{C}$ and $\myinputt': [0, T']\times\{0\} \to u'_{C}\in U_{C}$, respectively.
	Suppose initial state $\mystatet(0, 0) \in \mystatet'(0, 0) + \kappa_{1}\delta\mathbb{B}$ for some $\delta > 0$ and $\kappa_{1}\in (0, 1/2)$.
	Then, for each $\kappa_{2}\in (2\kappa_{1}, 1)$ and each $\epsilon\in (0, \frac{\kappa_{2}\delta}{2})$, $\mystatet'(T', 0)\in \mystatet(T, 0) + \kappa_{2}\delta \mathbb{B}$ if the following hold:
	\begin{enumerate}
		\item $T$ and $T'$ are such that\begin{equation}
		\label{equation:tm2}
		\begin{aligned}
		&T' \in T_{k} = \{t_{l}\in [\max\{T - \kappa_{2}\delta, 0\}, T]: \forall t'\in [t_{l}, T],\\
		& \phi(t', 0) + (\frac{\kappa_{2}\delta}{2} - \epsilon)\mathbb{B} \subset \phi(T, 0) + \frac{\kappa_{2}\delta}{2}\mathbb{B}\},
		\end{aligned}
		\end{equation}
		\item $u_{C}$ and $u'_{C}$ are such that \begin{equation}
		\label{equation:inequalityflowproof32}
		|u_{C} - u'_{C}| < \frac{\frac{\kappa_{2} \delta}{2} - \epsilon -\exp(K_{x}^{f}T)\kappa_{1}\delta}{K^{f}_{u}T \exp(K_{x}^{f}T)}.
		\end{equation}
	\end{enumerate}
\end{lemma}
\begin{proof}
	Because of Assumption \ref{assumption:flowlipschitz}, according to Lemma \ref{lemma:kleinbort}, we have
	\begin{equation}\label{equation:flowlemmainequality2}
	\begin{aligned}
	&|\mystatet(t, 0) - \mystatet'(t, 0)| \leq\\
	& \exp(K_{x}^{f}t)\kappa_{1}\delta + K^{f}_{u}t\exp(K_{x}^{f}t) |u_{C} - u'_{C}| 
	\end{aligned}
	\end{equation}
	
	 Note that $T' \in T_{k}\subset [\max\{T - \kappa_{2}\delta, 0\}, T]$, therefore, we have $T'\leq T$.  Then, by applying $T'\leq T$ and (\ref{equation:inequalityflowproof32}) to (\ref{equation:flowlemmainequality2}), we have
	$$
	\begin{aligned}
	&|\phi(T', 0) - \phi'(T', 0)| \\
	&\leq \exp(K_{x}^{f}T')\kappa_{1}\delta + K^{f}_{u}T'\exp(K_{x}^{f}T')|u_{C} - u'_{C}|  \\
	&\leq \exp(K_{x}^{f}T)\kappa_{1}\delta + K^{f}_{u}T\exp(K_{x}^{f}T)|u_{C} - u'_{C}| \\
	&< (\frac{\kappa_{2}\delta}{2} - \epsilon).\\
	\end{aligned} 
	$$
	Therefore, we have $\phi'(T', 0) \in \phi(T', 0) + (\frac{\kappa_{2}\delta}{2} - \epsilon)\mathbb{B}$.
	
	Since (\ref{equation:tm2}) holds, then we have
	$$
	\begin{aligned}
	\phi(T', 0) + (\frac{\kappa_{2}\delta}{2} - \epsilon)\mathbb{B}&\subset \phi(T, 0) + \frac{\kappa_{2}\delta}{2}\mathbb{B}\\
	&\subset \phi(T, 0) + \kappa_{2}\delta\mathbb{B}.
	\end{aligned}
	$$
	Therefore, we have $\phi'(T', 0)\in	\phi(T', 0) + (\frac{\kappa_{2}\delta}{2} - \epsilon)\mathbb{B}\subset \phi(T, 0) + \kappa_{2}\delta\mathbb{B}$.\ifbool{report}{}{\qed}
\end{proof}
\subsection{Proof of Lemma \ref{lemma:pccontinuouslowerbound}}
This proof proceeds as follows. From item 1 in Definition \ref{assumption:inputlibrary}, we denote the constant input signal  that is randomly selected from $\mathcal{U}_{C}$ as 
$
\tilde{\myinputt}': [0, t'_{m}]\to u'_{C}\in U_{C},
$ where $t'_{m}$ denotes the time duration of $\tilde{\myinputt}'$. Since in the statement of Lemma \ref{lemma:pccontinuouslowerbound}, the input function $\myinputt$ is also constant, denote $\myinputt: [0, \tau]\times\{0\}\to u_{C}\in U_{C}$. Under Assumption \ref{assumption:flowlipschitz}, Lemma~\ref{lemma:flowevent1} and Lemma~\ref{lemma:flowevent2} guarantee that both $E_{1}$ and $E_{2}$ occur if $t'_{m}$ and $u'_{C}$ satisfy the following conditions:
	\begin{enumerate}
		\item $t'_{m}$ is such that \begin{equation}
		\label{equation:sss1}
		\begin{aligned}
		&t'_{m} \in T_{k} :=  \{t_{l}\in [\max\{\tau - \kappa_{2}\delta, 0\}, \tau]: \\
		&\phi(t', 0) + r'\mathbb{B} \subset \phi(\tau, 0) + \frac{\kappa_{2}\delta}{2}\mathbb{B}\quad\forall t'\in [t_{l}, \tau]\}
		\end{aligned}
		\end{equation} where $r' = \frac{\kappa_{2} \delta}{2} - \epsilon$ and the elements in (\ref{equation:sss1}) come from (\ref{equation:tm1}) and (\ref{equation:tm2}).
		\item $u'_{C}$ is such that \begin{equation}
		\label{equation:inequalityflowproof11}
		0\leq\Delta u \leq \frac{\frac{\kappa_{2} \delta}{2} - \epsilon -\exp(K_{x}^{f}\tau)\kappa_{1}\delta}{K^{f}_{u}\tau \exp(K_{x}^{f}\tau)}.
		\end{equation} where $\Delta u := |u_{C} - u'_{C}|$ and the elements in (\ref{equation:inequalityflowproof11}) come from (\ref{equation:inequalityflowproof31}) and~(\ref{equation:inequalityflowproof32}).
\end{enumerate}
Then, by the uniform execution of HyRRT as defined in Definition \ref{assumption:uniformsample}, we proceed to characterize the probability of selecting $\tilde{\myinputt}'$, namely, selecting $t'_{m}$ and $u'_{C}$ satisfying (\ref{equation:sss1}) and (\ref{equation:inequalityflowproof11}), respectively, and provide a lower bound as in (\ref{equation:lemmaflow}).

	
	We first show that set $T_{k}$ has positive Lebesgue measure, which will be used to characterize $p_{t}$ in (\ref{equation:lemmaflow}). Since $\phi$ is purely continuous, for arbitrary small $\epsilon > 0$, there exists a lower bound $t_{l}' \in (0, \tau)$ such that 
	\begin{equation}\label{eq:ssss1}
		|\phi(t', 0) - \phi(\tau, 0)|< \epsilon
	\end{equation}
	for each $t'\in [t'_{l}, \tau]$. For each $t'\in [t'_{l}, \tau]$ and each point $x_{p}\in \phi(t', 0) + r'\mathbb{B}$, by the triangle inequality, it follows that:
	$$
	\begin{aligned}
	|x_{p}- \phi(\tau, 0)| &= |x_{p} - \phi(t', 0) +  \phi(t', 0) - \phi(\tau, 0)|\\
	&\leq  |x_{p} - \phi(t', 0)| +  |\phi(t', 0) - \phi(\tau, 0)|.\\
	\end{aligned}
	$$
	From $|x_{p} - \phi(t', 0)| \leq r' = \frac{\kappa_{2} \delta}{2} - \epsilon$ and (\ref{eq:ssss1}), it follows that
	\begin{equation}\label{eq:ssss2}
		\begin{aligned}
		|x_{p}- \phi(\tau, 0)| &\leq |x_{p} - \phi(t', 0)| +  |\phi(t', 0) - \phi(\tau, 0)|\\
		&\leq r' + \epsilon = \frac{\kappa_{2}\delta}{2}.
		\end{aligned}
	\end{equation}
	From (\ref{eq:ssss2}), for each point $x_{p}\in \phi(t', 0) + r'\mathbb{B}$, it follows $x_{p}\in \phi(\tau, 0) + \frac{\kappa_{2}\delta}{2}\mathbb{B}$. Therefore, we have
	$$
	\phi(t', 0) + r'\mathbb{B}\subset \phi(\tau, 0) + \frac{\kappa_{2}\delta}{2}\mathbb{B}
	$$ for each $t'\in [t'_{l}, \tau]$. This leads to the existence of $t'_{l} < \tau$ such that

	$$
	\begin{aligned}
	t'_{l} \in \overline{T}_{k} := &\{t_{l}\in [0, \tau]: \forall t'\in [t_{l}, \tau], \phi(t', 0) + r'\mathbb{B}\\
	&\subset \phi(\tau, 0) + \frac{\kappa_{2}\delta}{2}\mathbb{B}\},
	\end{aligned}
	$$ implying that $\mu(\overline{T}_{k}) \geq \tau - t'_{l} > 0$, where $\mu(\overline{T}_{k})$ denotes the Lebesgue measure of $\overline{T}_{k}$. Since the interval $[\max\{\tau - \kappa_{2}\delta, 0\}, \tau]$ has positive Lebesgue measure and the intervals $\overline{T}_{k}$ and $[\max\{\tau - \kappa_{2}\delta, 0\}, \tau]$ are both upper bounded by $\tau$, then $T_{k} = \overline{T}_{k}\cap [\max\{\tau - \kappa_{2}\delta, 0\}, \tau]$ has positive Lebesgue measure. By (\ref{equation:probabilitylebesgue}), the probability of selecting the $t'_{m}$ such that (\ref{equation:sss1}) is satisfied, denoted $p_{t}$, is computed as follows:
	$$
	p_{t} = \frac{\mu(T_{k})}{T_{m}} \in (0, 1]
	$$  where $T_{m}$ is from Definition \ref{assumption:inputlibrary}.

	Next, we discuss the conditions on $u'_{C}$ such that both $E_{1}$ and $E_{2}$ occur. In addition to (\ref{equation:inequalityflowproof11}), to ensure that no intersection between $\psi_{new}$ and the unsafe set $X_{u}$ prevents the return of $\psi_{new}$ in the function call $\texttt{new\_state}$, the following condition is also required:
	\begin{equation}
	\label{equation:inequalityflowproof2}
	\Delta u \leq \delta.
	\end{equation} Therefore, to ensure that both $E_{1}$ and $E_{2}$ occur, $u'_{C}$ need to satisfy both (\ref{equation:inequalityflowproof11}) and (\ref{equation:inequalityflowproof2}), namely,
	\begin{equation}
	\label{equation:inequalityflowproof3}
	\Delta u \leq \min \left\{\frac{\frac{\kappa_{2} \delta}{2} - \epsilon -\exp(K_{x}^{f}\tau)\kappa_{1}\delta}{K^{f}_{u}\tau \exp(K_{x}^{f}\tau)}, \delta\right\}.
	\end{equation}

	Note that the choice of $u'_{C}$ that satisfies (\ref{equation:inequalityflowproof3}) is a ball in $\mathbb{R}^{m}$ centered at $u_{C}$ with radius $\max \left\{\min \left\{\frac{\frac{\kappa_{2}\delta}{2} - \epsilon -\exp(K_{x}^{f}\tau)\kappa_{1}\delta}{K^{f}_{u}\tau \exp(K_{x}^{f}\tau)}, \delta\right\}, 0\right\}$, where the operator $\max\{\cdot, 0\}$ prevents the negative values for the radius. Therefore, according to (\ref{equation:probabilitylebesgue}), the probability of selecting $u'_{C}$ satisfying  (\ref{equation:inequalityflowproof11}), denoted $p_{u}$, is
	$$
	p_{u} = \frac{\zeta_{n} \left(\max \left\{\min \left\{\frac{\frac{\kappa_{2}\delta}{2} - \epsilon -\exp(K_{x}^{f}\tau)\kappa_{1}\delta}{K^{f}_{u}\tau \exp(K_{x}^{f}\tau)}, \delta\right\}, 0\right\}\right)^{m}}{\mu(U_{C})}.
	$$
	
	Hence, we have
	$$
	\begin{aligned}
	&\mbox{\rm Prob}[E_{1}\&E_{2}]\geq p_{t} p_{u} \\
	&= p_{t}\frac{\zeta_{n} \left(\max \left\{\min \left\{\frac{\frac{\kappa_{2}\delta}{2} - \epsilon -\exp(K_{x}^{f}\tau)\kappa_{1}\delta}{K^{f}_{u}\tau \exp(K_{x}^{f}\tau)}, \delta\right\}, 0\right\}\right)^{m}}{\mu(U_{C})}.
	\end{aligned}
	$$\ifbool{report}{}{\qed}
%
%
%

\section{Proof of Lemma \ref{lemma:pcdiscretelowerbound}}\label{section:proof_discrete}
	Under Assumption \ref{assumption:pcjumpmap}, we have
	$$
	\begin{aligned}
	|\phi(0, 1) - \phi_{new}(0, 1)| &\leq K^{g}_{x}|\phi(0, 0) - \phi_{new}(0, 0)| + K^{g}_{u}\Delta u\\
	\end{aligned}
	$$
	where $\Delta u = |\myinputt(0, 0) - \myinputt_{new}(0, 0)|$. From $\phi_{new}(0, 0)\in \phi(0, 0) + \kappa_{1}\delta\mathbb{B}$, we have
	\begin{equation}\label{eq:ssss3}
		\begin{aligned}
		|\phi(0, 1) - \phi_{new}(0, 1)| \leq K^{g}_{x}\kappa_{1}\delta + K^{g}_{u}\Delta u
		\end{aligned}
	\end{equation}
	with $K^{g}_{x}, K^{g}_{u} > 0$. We denote the input value that is randomly selected from $\mathcal{U}_{D}$ in the function call $\texttt{new\_state}$ as $u_{D}\in \myreals[m]$. Given that $u_{D}$ is input to the discrete dynamics simulator (\ref{simulator:jump}) to simulate $(\mystatet_{new}, \myinputt_{new})$, it follows that $\myinputt_{new}(0, 0) = u_{D}$ and, hence, $\Delta u = |\myinputt(0, 0) - u_{D}|$.
	
	Then, we show that if
	\begin{equation}
	\label{equation:inputclose}
	\Delta u\leq \frac{(\kappa_{2} - K^{g}_{x}\kappa_{1})\delta}{K^{g}_{u}},
	\end{equation}
	then, the probabilistic event $E$ occurs. From (\ref{equation:inputclose}), it follows from (\ref{eq:ssss3}) that
	$$
	|\phi(0, 1) - \phi_{new}(0, 1)| \leq K^{g}_{x}\kappa_{1}\delta + K^{g}_{u}\Delta u \leq \kappa_{2}\delta,
	$$
	namely, $ \phi_{new}(0, 1) \in \phi(0, 1) + \kappa_{2}\delta\mathbb{B}$, implying that $E$ occurs.
	
	To ensure that no intersection between $\psi_{new}$ and the unsafe set $X_{u}$ prevents the return of $\psi_{new}$ in the function call $\texttt{new\_state}$, the following condition on the sampling result of $u_{D}$ is also required:
	\begin{equation}
	\label{equation:jumpsafety}
	0\leq \Delta u = |\myinputt(0, 0) - u_{D}| \leq \delta.
	\end{equation}
	Therefore, to satisfy both (\ref{equation:inputclose}) and (\ref{equation:jumpsafety}) such that $E$ occurs, we have
	\begin{equation}
	\label{equation:lemmajump}
	0\leq\Delta u \leq \min \left\{\frac{(\kappa_{2} - K^{g}_{x}\kappa_{1})\delta}{K^{g}_{u}}, \delta\right\}.
	\end{equation}
	By (\ref{equation:probabilitylebesgue}), the probability to randomly select an input value from $\mathcal{U}_{D}$ that satisfies (\ref{equation:lemmajump}) is
	\begin{equation}
	\label{equation:jumpsamplelowerbound}
	p_{u} = \frac{\zeta_{n} \left(\max \left\{\min \left\{\frac{(\kappa_{2} - K^{g}_{x}\kappa_{1})\delta}{K^{g}_{u}}, \delta\right\}, 0\right\}\right)^{m}}{\mu(U_{D})}
	\end{equation}
	where $\zeta_{n}$ is the Lebesgue measure of the unit ball in $\mathbb{R}^{m}$ and $\mu(U_{D})$ denotes the Lebesgue measure of $U_{D}$.
	Therefore, we have
	\begin{equation}
	\mbox{\rm Prob}[E] \geq p_{u} = \frac{\zeta_{n} \left(\max \left\{\min \left\{\frac{(\kappa_{2} - K^{g}_{x}\kappa_{1})\delta}{K^{g}_{u}}, \delta\right\}, 0\right\}\right)^{m}}{\mu(U_{D})}.
	\end{equation}\ifbool{report}{}{\qed}

\section{Closeness Guarantee between the Concatenation Results of Hybrid Arcs}
\subsection{Supporting Lemma}
\begin{lemma}\label{lemma:closeequaltj}
	Given two compact hybrid arcs $\phi$ and $\phi'$ that are $(\tau, \epsilon)$-close, then we have $|T - T'| < \epsilon$ and $J  = J'$\pn{,} where $(T, J) = \max \dom \phi$, $(T', J') = \max \dom \phi'$, and $\tau = \max\{T + J, T' + J'\}$.
\end{lemma}
\begin{proof}
	We prove
	$$
		|T - T'| < \epsilon
	$$by contradiction. Suppose 
	\begin{equation}\label{equation:flowclose1}
		T - T' \geq \epsilon.
	\end{equation} Since $\phi$ and $\phi'$ are $(\tau, \epsilon)$-close, from Definition \ref{definition:closeness},	for $(T, J)\in \dom \phi$ satisfying $T + J \leq \tau = \max\{T + J, T' + J'\}$, there exists $s$ such that $(s, J)\in \dom \phi'$ and 
	\begin{equation}\label{equation:flowclose2}
	|T - s|< \epsilon.
	\end{equation} By $(s, J)\in \dom \phi'$ and (\ref{equation:flowclose1}), it follows that $s\leq T' < T$. Then, by (\ref{equation:flowclose1}), we have 
	$$
	|T - s| \geq T - T' \geq \epsilon,
	$$ which contradicts (\ref{equation:flowclose2}). Therefore, (\ref{equation:flowclose1}) cannot hold. A similar contradiction can be derived if we suppose $T' - T \geq \epsilon$. Therefore, we have $|T - T'| < \epsilon$.
	
	Similarly, we prove $J  = J'$ by contradiction. Suppose 
	\begin{equation}\label{equation:jumpclose1}
	 J > J'.
	\end{equation}
	Then\pn{,} from Definition \ref{definition:closeness}, for $(T, J)\in \dom \phi$ satisfying $T + J \leq \tau = \max\{T + J, T' + J'\}$, there exists $s$ such that $(s, J)\in \dom \phi'$. However, since $(T', J') = \max \dom \phi'$ and (\ref{equation:jumpclose1}), such $s$ does not exist. Therefore, (\ref{equation:jumpclose1}) cannot hold. A similary contradiction can also be derived if we suppose $J < J'$. Therefore, we have $J = J'$.
\end{proof}
\subsection{Closeness Guarantee}
Next, we demonstrate that the operation of concatenation preserves the closeness between the hybrid arcs.
\begin{proposition}
	\label{proposition:concatenateclose}
	Given compact hybrid arcs $\phi_{1}$,  $\phi_{2}$,  $\phi'_{1}$, and $\phi'_{2}$ such that $\phi_{1}$ and $\phi'_{1}$ are $(\tau_{1}, \epsilon_{1})$-close and $\phi_{2}$ and $\phi'_{2}$ are $(\tau_{2}, \epsilon_{2})$-close, where $(T_{1}, J_{1}) = \max \dom \phi_{1}$, $(T'_{1}, J'_{1}) = \max \dom \phi'_{1}$ and $\tau_{1} = \max \{T_{1} + J_{1}, T'_{1} + J'_{1}\}$, then $\phi_{1}|\phi_{2}$ and $\phi'_{1}|\phi'_{2}$ are $(\tau_{1}+ \tau_{2}, \epsilon_{1} + \epsilon_{2})$-close.
\end{proposition}
\begin{proof}
%
	We show that $\phi = \phi_{1}|\phi_{2}$ and $\phi' = \phi'_{1}|\phi'_{2}$ satisfy the following, as introduced in Definition \ref{definition:closeness}:
	\begin{enumerate}[label=C\arabic*.]
		\item for all $(t, j)\in \dom \phi$ with $t + j \leq \tau_{1}+ \tau_{2}$, there exists $s$ such that $(s, j)\in \dom \phi'$, $|t - s|< \epsilon_{1} + \epsilon_{2}$, and $|\phi(t, j) - \phi'(s, j)| < \epsilon_{1} + \epsilon_{2}$;
		\item for all $(t, j)\in \dom \phi'$ with $t + j \leq \tau_{1}+ \tau_{2}$, there exists $s$ such that $(s, j)\in \dom \phi$, $|t - s|< \epsilon_{1} + \epsilon_{2}$, and $|\phi'(t, j) - \phi(s, j)| < \epsilon_{1} + \epsilon_{2}$.
	\end{enumerate}
	First, we show that C1 holds. From Definition \ref{definition:concatenation}, it follows that $\dom \phi = \dom \phi_{1} \cup (\dom \phi_{2} + \{(T_{1}, J_{1})\})$, where $(T_{1}, J_{1}) = \max \dom \phi_{1}$. Then, for all $(t, j)\in \dom \phi = \dom \phi_{1} \cup (\dom \phi_{2} + \{(T_{1}, J_{1})\} )$ with $t + j \leq \tau_{1}+ \tau_{2}$, we show that C1 is satisfied when 1) $(t, j)\in \dom \phi_{1}$ and when 2) $(t, j)\in  \dom \phi_{2} + \{(T_{1}, J_{1})\}$, respectively.
	\begin{enumerate}
		\item For all $(t, j)\in \dom \phi$ such that $(t, j)\in \dom \phi_{1}$, namely, $t + j \leq T_{1} + J_{1}$, given that $\tau_{1} = \max \{T_{1} + J_{1}, T'_{1} + J'_{1}\}$, it follows that 
		$$
		t + j \leq \tau_{1} = \max \{T_{1} + J_{1}, T'_{1} + J'_{1}\}.
		$$ Since $\phi_{1}$ and $\phi'_{1}$ are $(\tau_{1}, \epsilon_{1})$-close, there exists $s'$ such that $(s', j)\in \dom \phi'_{1}$, $|t - s'|< \epsilon_{1}$, and $|\phi_{1}(t, j) - \phi'_{1}(s', j)| < \epsilon_{1}$.
		Therefore, there exists $s = s'$ such that $(s, j)\in \dom \phi'_{1}\subset \dom \phi'$, $|t - s| < \epsilon_{1} < \epsilon_{1} + \epsilon_{2}$, and $|\phi(t, j) - \phi'(s, j)| = |\phi_{1}(t, j) - \phi'_{1}(s, j)| < \epsilon_{1} < \epsilon_{1} + \epsilon_{2}$. Hence, C1 is established.
		\item For all $(t, j)\in \dom \phi$ such that $(t, j)\in (\dom \phi_{2} + \{(T_{1}, J_{1})\} )$, namely, $t + j \geq T_{1} + J_{1}$, since $\phi_{2}$ and $\phi'_{2}$ are $(\tau_{2}, \epsilon_{2})$-close, there exists $s'$ such that 
		\begin{equation}\label{equation:closeinequality0}
			(s', j - J_{1})\in \dom \phi'_{2}
		\end{equation}
		\begin{equation}\label{equation:closeinequality1}
			|t - T_{1} - s'|< \epsilon_{2},
		\end{equation}
		\begin{equation}\label{equation:closeinequality2}
			|\phi_{2}(t - T_{1}, j - J_{1}) - \phi'_{2}(s', j - J_{1})| < \epsilon_{2}.
		\end{equation}
		
		Since $\phi_{1}$ and $\phi'_{1}$ are $(\tau_{1}, \epsilon_{1})$-close where $\tau_{1} = \max \{T_{1} + J_{1}, T'_{1} + J'_{1}\}$, by Lemma \ref{lemma:closeequaltj}, we have
		\begin{equation}\label{equation:closeinequality3}
			|T_{1} - T'_{1}| \leq \epsilon_{1}
		\end{equation}
		and
		\begin{equation}\label{equation:closeinequality4}
		J_{1} = J'_{1}.
		\end{equation}
		By (\ref{equation:closeinequality0}) and (\ref{equation:closeinequality4}), it follows that $s'$ is such that 
		$$\begin{aligned}
		&(s' + T'_{1}, j)\in \dom \phi'_{2} + \{(T'_{1}, J_{1})\}  \\
		&=  \dom \phi'_{2} + \{(T'_{1}, J'_{1})\}\subset \dom \phi'.
		\end{aligned}
		$$ Furthermore,  because of (\ref{equation:closeinequality1}) and (\ref{equation:closeinequality3}), by the triangle inequality, it follows that 
		$$
		|t - (s' + T'_{1})|\leq |t - T_{1} - s'| + |T_{1} - T'_{1}| < \epsilon_{2} + \epsilon_{1}.
		$$ 
		By (\ref{equation:closeinequality2}), it follows that 
		$$
		\begin{aligned}
		|\phi(t, j) - &\phi'(s' + T'_{1}, j)|  \\
		&= |\phi_{2}(t - T_{1}, j - J_{1}) - \phi'_{2}(s', j - J_{1})| \\
		& <  \epsilon_{2} < \epsilon_{1} + \epsilon_{2}.
		\end{aligned}
		$$ 
		Therefore, we can find $s = s' + T'_{1}$ such that $(s, j)\in \dom \phi'$, $|t - s'|< \epsilon_{1} + \epsilon_{2}$, and $|\phi(t, j) - \phi'(s', j)| < \epsilon_{1} + \epsilon_{2}$.
	\end{enumerate}
	The proof for C2 follows a similar logic to the aforementioned arguments, achieved by swapping $\mystatet$ and $\mystatet'$.  Therefore, $\phi_{1}|\phi_{2}$ and $\phi'_{1}|\phi'_{2}$ are $(\tau_{1}+ \tau_{2}, \epsilon_{1} + \epsilon_{2})$-close.\ifbool{report}{}{\qed}
\end{proof}

\section{Definition of Truncation and Translation Operation}\label{section:pc_positiveclearance}
\begin{definition}[(Truncation and translation operation)]
	\label{definition: truncation}
	Given a function $\phi: \dom \phi \to \mathbb{R}^{n}$, where $\dom \phi$ is hybrid time domain, and pairs of hybrid time $(T_{1}, J_{1})\in \dom \phi$ and $(T_{2}, J_{2})\in \dom \phi$ such that $T_{1}\leq T_{2}$ and $J_{1} \leq J_{2}$, the function $\widetilde\phi: \dom \widetilde\phi \to \mathbb{R}^{n}$ is the truncation of $\phi$ between $(T_{1}, J_{1})$ and $(T_{2}, J_{2})$ and translation by $(T_{1}, J_{1})$  if
	\begin{enumerate}
		\item $\dom \widetilde\phi =  \dom \phi \cap ([T_{1}, T_{2}]\times \{J_{1}, J_{1} + 1, ..., J_{2}\}) - \{(T_{1}, J_{1})\}$, where the minus sign denotes Minkowski difference;
		\item $\widetilde\phi(t, j) = \phi(t + T_{1}, j + J_{1})$ for all $(t, j)\in \dom \widetilde\phi$.
	\end{enumerate}
\end{definition}
\end{document}